\documentclass[letterpaper, 10 pt, journal, twoside]{ieeetran}  

\usepackage{subfigure}
\usepackage{graphicx}
\usepackage{amsmath}
\usepackage{physics}
\usepackage{amssymb}
\usepackage{mathrsfs}
\usepackage{url}
\usepackage{textcomp}
\usepackage[font=small, skip=2pt]{caption}
\usepackage{gensymb}
\usepackage{float}
\usepackage{xfrac}
\usepackage{comment}
\usepackage{colortbl}
\usepackage{xcolor}
\usepackage{etoolbox}
\usepackage{multirow}
\usepackage{array}
\usepackage{enumitem}
\usepackage{kantlipsum}
\usepackage{sidecap}
\usepackage{bm}
\usepackage{xy}
\xyoption{import}
\usepackage{amsthm}
\usepackage{soul}
\usepackage{etoolbox}
\usepackage[ruled, vlined, linesnumbered]{algorithm2e}
\usepackage{cite}
\usepackage{wrapfig}

\definecolor{black-default}{rgb}{0, 0, 0}
\definecolor{medium-blue}{rgb}{0,0.25,1}
\definecolor{medium-green}{rgb}{0,0.75,0.25}
\definecolor{coral}{RGB}{150, 64, 0}

\newtheorem{theorem}{Theorem}
\newtheorem{lemma}{Lemma}
\newtheorem{proposition}{Proposition}
\newtheorem{corollary}{Corollary}
\newtheorem{assumption}{Assumption}
\newtheorem{definition}{Definition}
\newtheorem{remark}{Remark}

\DeclareMathOperator*{\argmin}{arg\,min}

\newcommand{\figref}[1]{{Fig.~\ref{#1}}}

\newcommand{\tableref}[1]{Table~\ref{#1}}

\newcommand{\secref}[1]{Section~\ref{#1}}
\newcommand{\subSecref}[1]{Section~\ref{#1}}
\newcommand{\apdxref}[1]{~\ref{#1}}
\newcommand{\lemmaref}[1]{Lemma~\ref{#1}}
\newcommand{\theoremref}[1]{Theorem~\ref{#1}}
\newcommand{\propositionref}[1]{Proposition~\ref{#1}}
\newcommand{\algoref}[1]{Algorithm~\ref{#1}}
\newcommand{\algrefsTwo}[2]{Algs.~\ref{#1}, \ref{#2}}
\newcommand{\alglineref}[1]{line~\ref{#1}}
\newcommand{\singleQuote}[1]{\lq{#1}\rq}
\newcommand{\alglinerefs}[2]{{lines~\ref{#1}-\ref{#2}}}
\newcommand{\paren}[1]{\left({#1}\right)}
\newcommand{\parenShort}[1]{({#1})}

\newcommand{\revision}[1]{\textcolor{black}{#1}}
\newcommand{\revisionTwo}[1]{\textcolor{black}{#1}}

\newcommand{\reverseAlg}[1]{\textcolor{coral}{#1}}
\newcommand{\algorithmScale}{0.72}

\newcommand{\barDimScale}{0.27}

\newcommand{\compareResultCostScale}{0.25}
\newcommand{\compareResultScale}{0.32}
\newcommand{\simSolWidth}{2.925cm}
\newcommand{\simSolHeight}{3.3cm}
\newcommand{\parameterScale}{0.25} 
\newcommand{\revisionAlgorithm}{\color{black}}

\newcommand*{\figuretitle}[1]{%
    {\centering
    \textbf{#1}
    \par\medskip}
}

\makeatletter
\renewenvironment{proof}[1][\proofname]{\par
  \vspace{-\topsep}
  \pushQED{\qed}%
  \normalfont
  \topsep0pt \partopsep0pt 
  \trivlist
  \item[\hskip\labelsep
        \itshape
    #1\@addpunct{.}]\ignorespaces
}{%
  \popQED\endtrivlist\@endpefalse
  \addvspace{6pt plus 6pt} 
}
\makeatother

\newcommand{\blockComment}[1]{\iffalse [1]\fi}

\renewcommand{\vec}[1]{\bm{#1}}
\newcommand{\vecm}[1]{\bm{#1}}

\makeatletter
\newcommand{\removelatexerror}{\let\@latex@error\@gobble}
\makeatother

\renewcommand{\vec}[1]{\boldsymbol{#1}}

\IEEEoverridecommandlockouts                              

\title{Bidirectional Sampling-Based Motion Planning Without Two-Point Boundary Value Solution}

\author{Sharan~Nayak,~\IEEEmembership{Student Member,~IEEE,} and Michael W. Otte,~\IEEEmembership{Member,~IEEE}
\thanks{The authors Sharan Nayak (Graduate Student Researcher) and Michael Otte (Assistant Professor) are with the Department of Aerospace Engineering, University of Maryland, College Park (UMD), MD 20742 USA. Email: snayak18@umd.edu, otte@umd.edu.}}

\begin{document}
\maketitle

\begin{abstract}
Bidirectional path and motion planning approaches decrease planning time, on average, compared to their unidirectional counterparts. In single-query feasible motion planning, using bidirectional search to find a \textit{continuous} motion plan requires an edge connection between the forward search tree and the reverse search tree. Such a tree-tree connection requires solving a two-point Boundary Value Problem (BVP). 
\revision{However, obtaining a closed-form two-point BVP solution can be difficult or impossible for many systems. While numerical methods can provide a reasonable solution in many cases, they are often computationally expensive, numerically unstable, or sensitive (to an initial guess) for the purposes of single-query sampling-based motion planning.}
To overcome this challenge, we present a novel bidirectional search \revision{strategy} that does not require solving the two-point BVP. Instead of connecting the \revision{forward and reverse} trees directly, the reverse tree's cost information is used as a guiding heuristic for forward search. This enables the forward search to  quickly \revision{grow down the reverse tree---}converging to a fully feasible solution \textit{without} a direct tree-tree connection and \textit{without} the solution to a two-point BVP. 
\revision{We propose two algorithms that use this strategy for single-query feasible motion planning for various dynamical systems, performing experiments in both simulation and hardware test-beds}. \revision{We find that these algorithms perform}  better than or \revisionTwo{comparable to} existing state-of-the-art methods \revision{with respect to} quickly finding an initial feasible solution.
\end{abstract}

\begin{IEEEkeywords}
Motion and Path Planning, Autonomous agents and Dynamics
\end{IEEEkeywords}

\section{INTRODUCTION}
\label{sec:introduction}
\IEEEPARstart{S}{ampling} based motion planning approaches have become popular in robotics due to their simplicity and rapid exploration of high-dimensional state spaces. They involve randomly sampling collision-free feasible states of a robot in the configuration space and moving between them in a way that respects a robot's kinematics and dynamics.
\par Feasible planning \cite{lavalle1998rapidly,liu2019goal} and any-time asymptotically optimal (AO) motion planning \cite{karaman2011sampling,hauser2016asymptotically} are two problems in single-query \cite{lavalle2006planning} motion planning that have garnered interest in the motion planning community. Traditionally, feasible planning is used to find solutions to difficult problems (e.g. Piano mover's problem \cite{kuffner2000rrt}, Alpha puzzle \cite{zhang2008d}) where finding any solution is viewed as an achievement. 
\revision{For this reason, feasible planning is concerned with finding any valid solution, and all solutions are considered equivalent---there is no requirement of minimizing cost or maximizing reward}. More recently, feasible planning has also been used as the first step in any-time AO motion planning, where it is preferable to have \revision{both} an actionable solution as soon as possible 
\revision{but then also get better} solutions 
\revision{(with respect to a metric)} whenever there is \revision{additional planning time} to refine the solution toward optimal. In this article, we focus on solving the feasible motion planning problem as quickly as possible.

\begin{figure}[t!]
\begin{subfigure}
  \centering
  \figuretitle{Different types of searches}
  \includegraphics[scale=0.36]{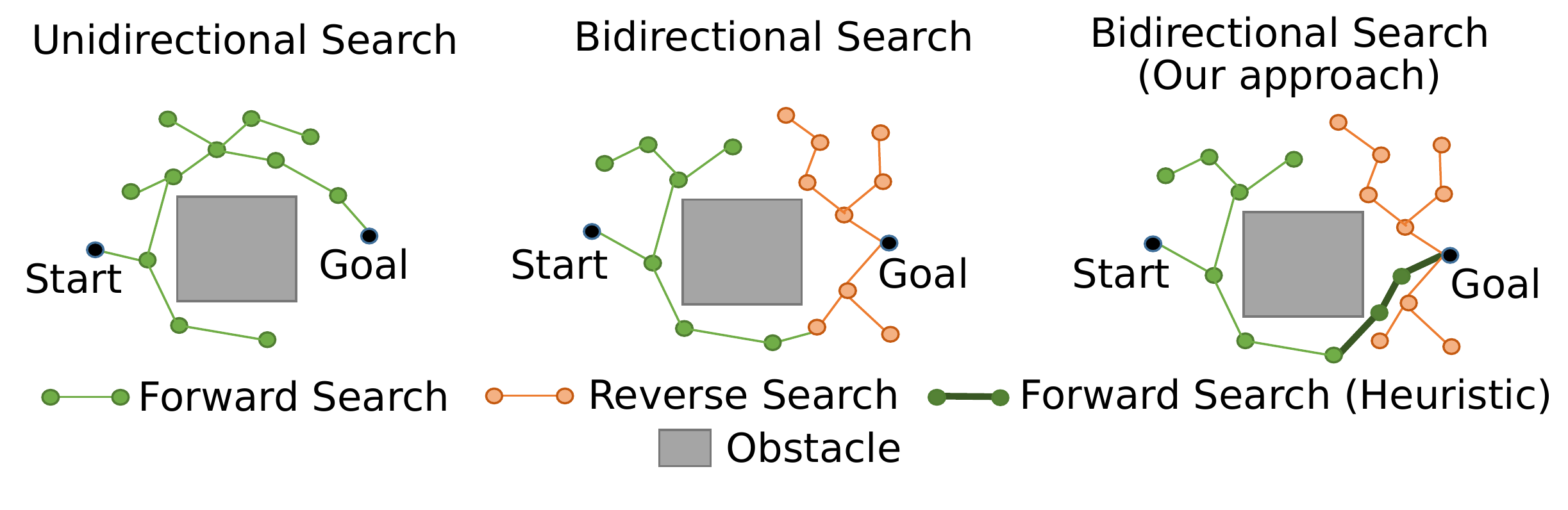}
\end{subfigure}%
\caption{Comparison of unidirectional (left) and bidirectional (middle) searches commonly used in single-query feasible planning, and our bidirectional search approach (right) where the forward search uses heuristic from reverse search to reach the goal. Our approach is applicable when two-point BVP cannot be solved.}
\label{fig:UniBiSearchFig}
\vspace{5mm}
\begin{subfigure}
  \centering
  \figuretitle{Availability and absence of two-point BVP}
  \hspace{7mm}
  \includegraphics[scale=0.35]{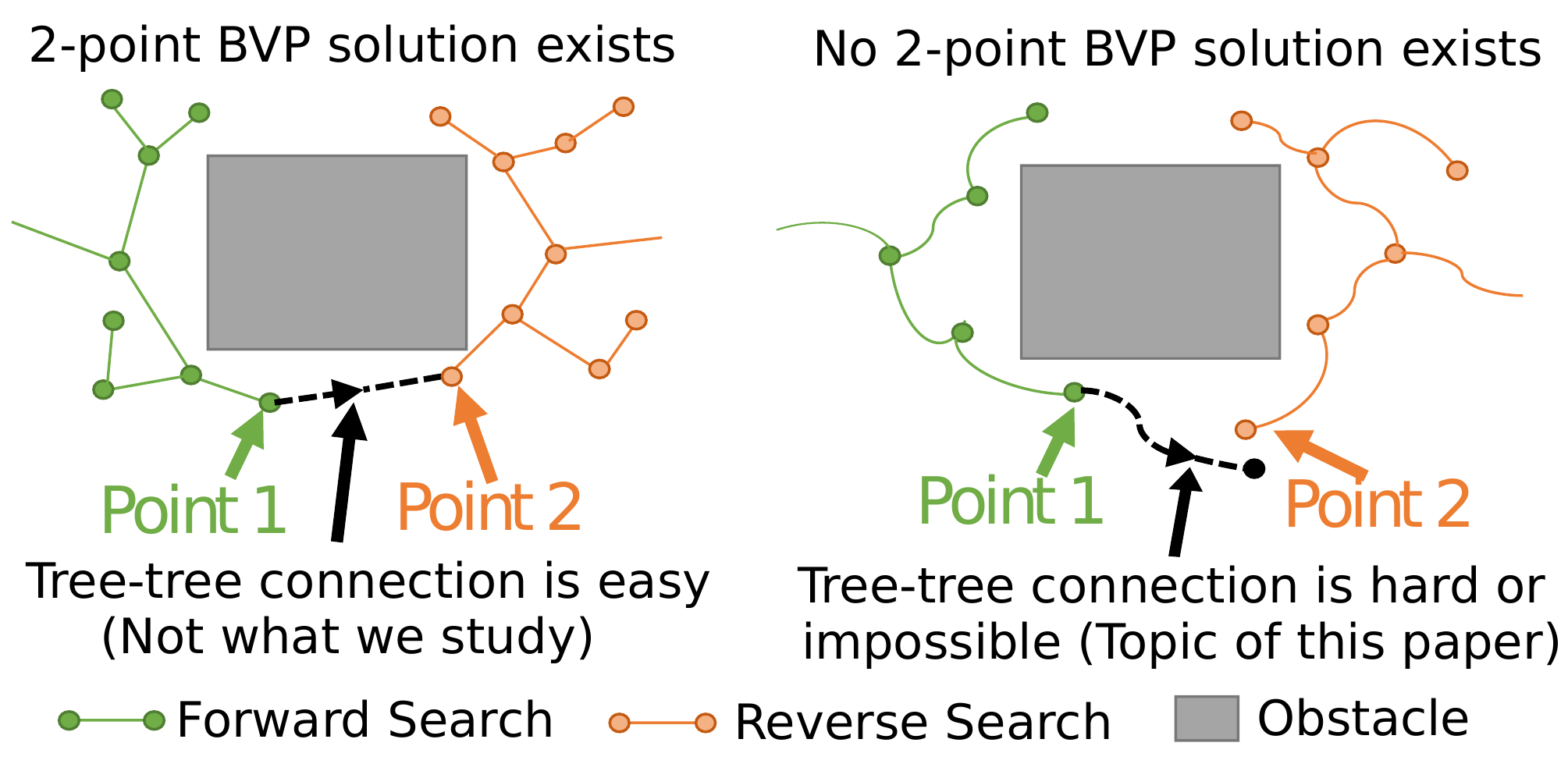}
\end{subfigure}%
\caption{Comparison of instances where the two-point BVP can (left) and cannot (right) be solved for connecting the forward and reverse trees in a bidirectional search.}
\label{fig:2BVPOnly}
\end{figure}
\par Most single-query sampling-based feasible motion planning algorithms use unidirectional or bidirectional search (\figref{fig:UniBiSearchFig}) for their planning. Unidirectional search has a single search-tree built from the start state forwards to the goal state \revision{or vice-versa,} whereas bidirectional search has two search trees: one expanding forward from the start or state and the other expanding backward from the goal state. Typically, the forward and reverse trees are connected if/when they grow ``close" to each other. However, as we shall describe shortly, this is not the approach taken in our work. 
\par Bidirectional searches are popular due to their improved speed \cite{liu2019goal} in finding an initial feasible solution in higher-dimensional search spaces and spaces with narrow passages \cite{jordan2013optimal}. Existing bidirectional search approaches \cite{jordan2013optimal,devaurs2013enhancing,starek2014bidirectional} have focused on solving the planning problem when a steering function is available, i.e., when the two-point Boundary Value Problem (BVP) \cite{lavalle2006planning,keller2018numerical} is solvable (\figref{fig:2BVPOnly}). However, it is difficult to derive closed-form solutions for many kinodynamic robots, especially with non-holonomic constraints. Practical solutions of two-point BVPs for many systems do not exist in a closed form, and may require the use of numerical methods like shooting approaches \cite{osborne1969shooting}, which are often computationally expensive for real-world planning time constraints. On the other hand, existing bidirectional approaches that do not require solving the two-point-BVP \cite{lavalle2001randomized} suffer from discontinuities when connecting the two trees, while \revision{post-processing methods} like curve-smoothing  \cite{liu2019goal} have been employed to remove the discontinuities, which may not be trivial.

Maneuver/Trajectory libraries \cite{branicky2008path,go2006autonomous} have been used to speed up execution of motion planners. The speed-up is achieved by storing the pre-computed trajectories in memory rather than calculating them online during the generation of a motion plan. \revision{However, while there are (typically) an unaccountably infinite number of trajectories that a system could potentially include in a maneuver library, only a finite subset of such trajectories can be included in practice due to memory constraints}. Therefore, maneuver libraries do not remove the requirement of solving a two-point-BVP when linking two sampling-based motion planning trees rooted at arbitrary start and goal coordinates\footnote{Special cases of symmetry can be used to create a lattice of trajectories compatible with graph search algorithms, which can be used as an alternative to sampling-based motion planning; however, such lattices present a different set trade-offs: they are not probabilistically compete and do not scale to high dimensional spaces.}.

\textbf{The main contribution of this article is a novel \mbox{\textit{bidirectional}} single-query feasible sampling-based motion planning algorithm called Generalized Bidirectional Rapidly-exploring Random Tree (GBRRT) that \textit{does not require solving the two-point BVP}}. The term `Generalized' implies that this algorithm can be applied to many systems including systems that are: holonomic, non-holonomic, and/or kinodynamic. Our method (\figref{fig:VisualAbstract}) differs from the bidirectional RRT-Connect \cite{lavalle2001randomized} approach in that GBRRT does not try to connect the two trees when they are close; instead, it uses the heuristic information provided by the reverse search tree to help the forward tree quickly grow towards the goal configuration. 

Another contribution of this article is an asymmetric variant of GBRRT called Generalized Asymmetric Bidirectional Rapidly-exploring Random Tree (GABRRT) that uses \textit{a naive reverse search} with no dynamical trajectories in the reverse tree (\figref{fig:BiDirRRTVariant}) to provide heuristic information to the forward tree; GABRRT performs better than GBRRT in certain conditions (\subSecref{subSection:GABRRT}). We prove that GBRRT and GABRRT are probabilistically complete \cite{lavalle2006planning}.
We run multiple experiments in simulation to evaluate the performance of GBRRT and GABRRT versus seven other algorithms across six different dynamical systems. We also run hardware experiments in a quadrotor test-bed to evaluate GABRRT in a more realistic physical setting.

\begin{figure}[ht]
\begin{subfigure}
  \centering
  \figuretitle{GBRRT applied to an example motion planning problem}
  \includegraphics[width=2.87cm, height=2.87cm]{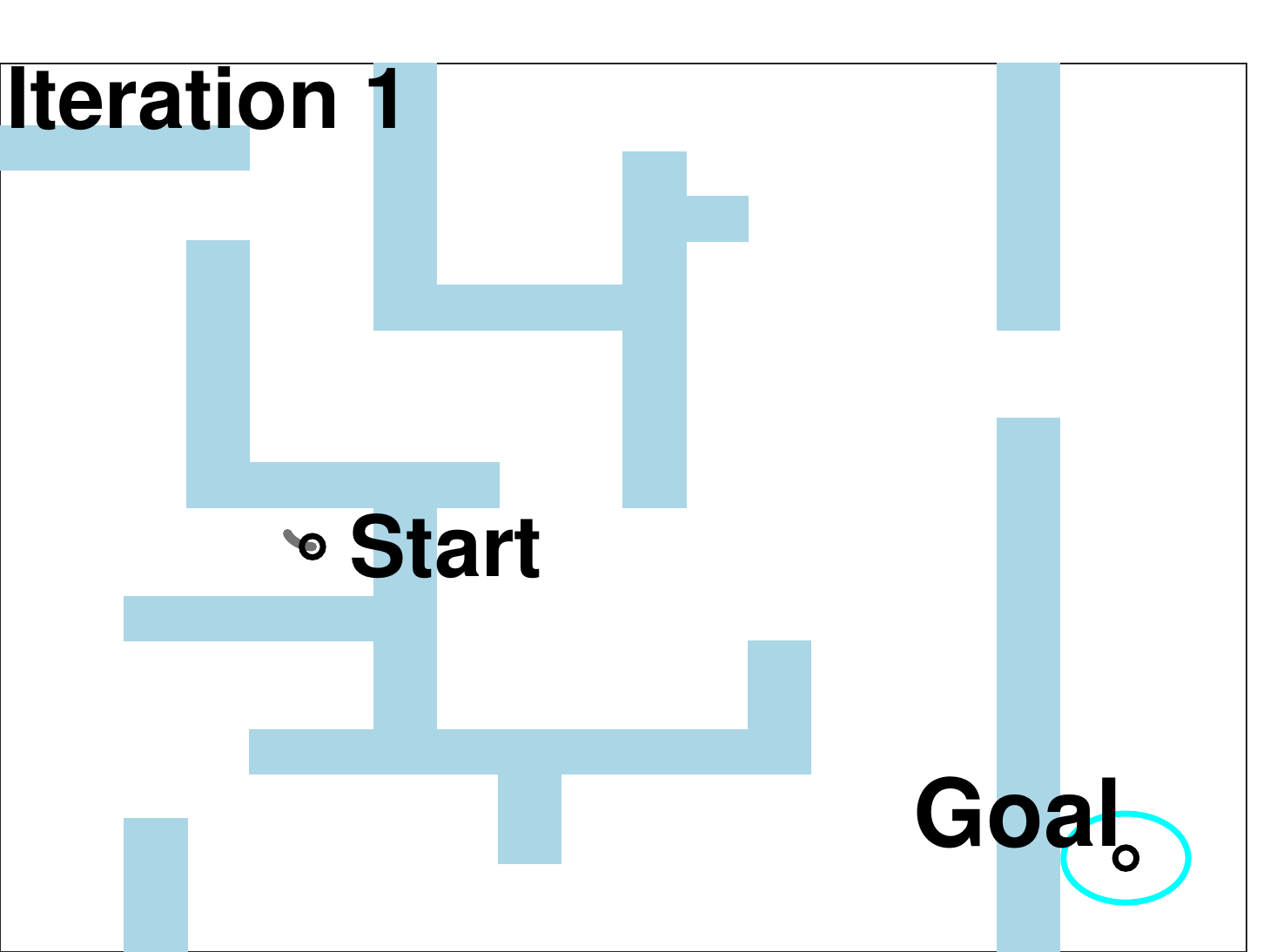}
\end{subfigure}%
\begin{subfigure}
\centering
  \includegraphics[width=2.87cm, height=2.87cm]{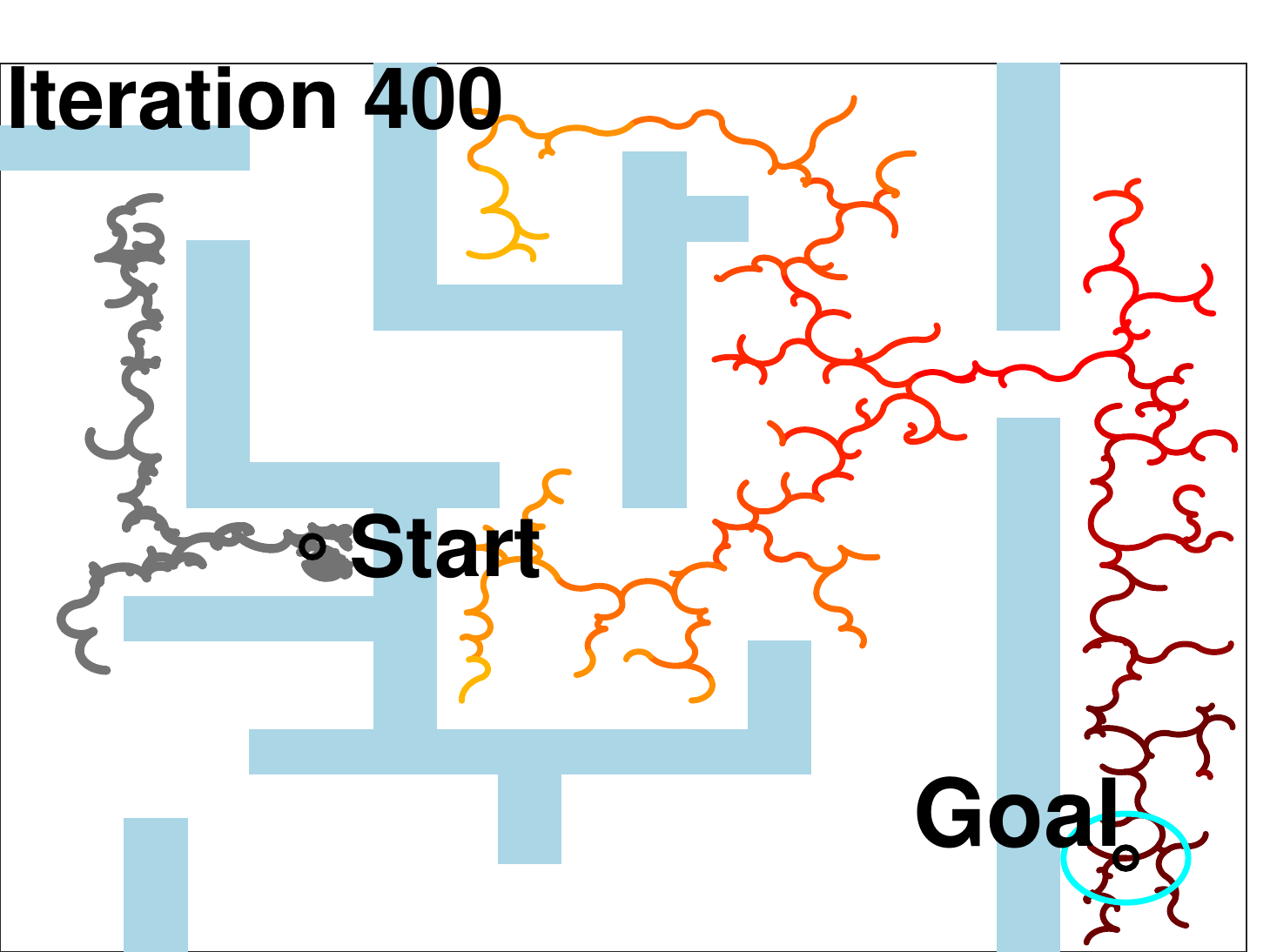}
\end{subfigure}
\begin{subfigure}
 \centering
  \includegraphics[width=2.87cm, height=2.87cm]{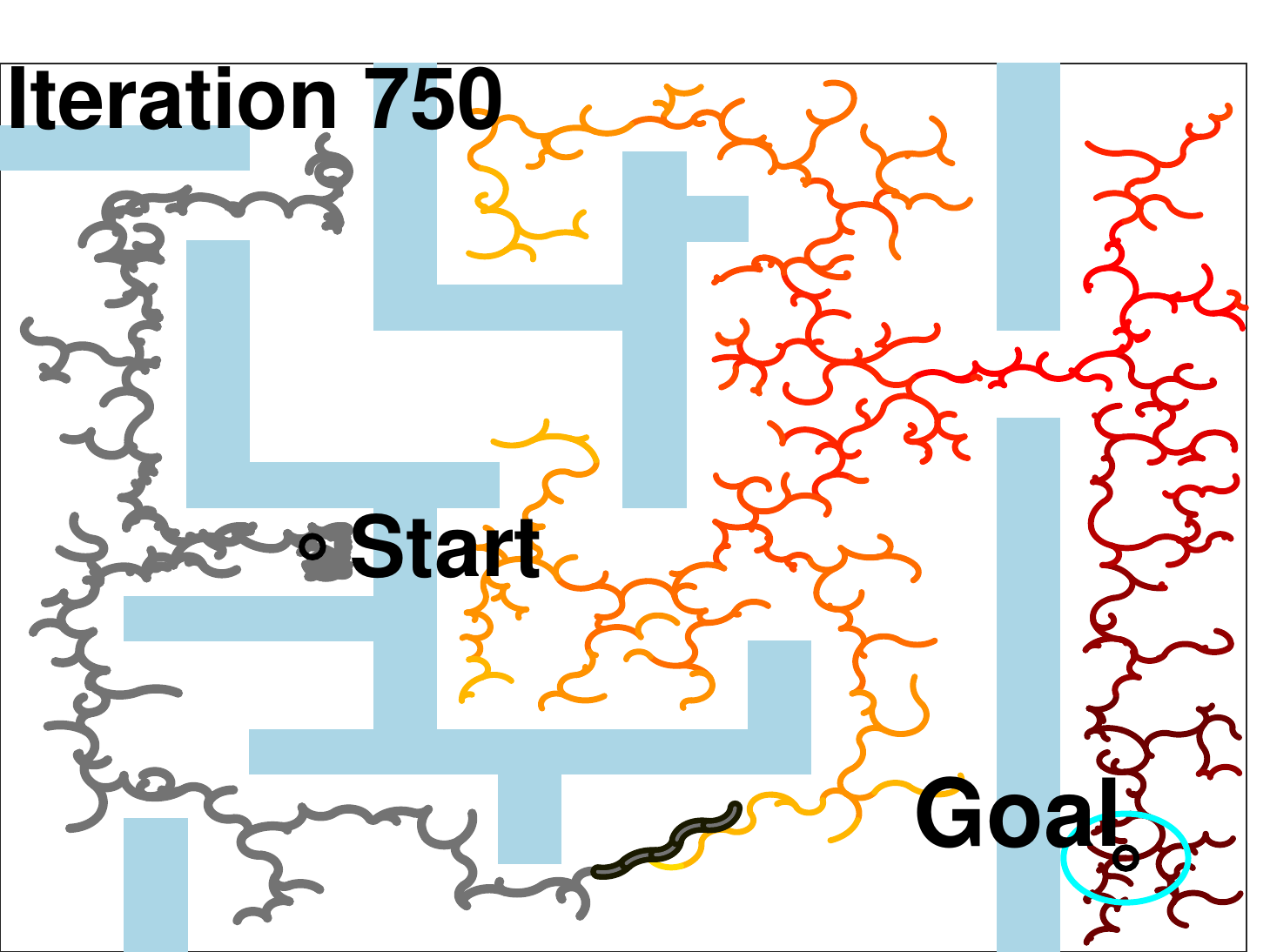}
\end{subfigure}%
\begin{subfigure}
\centering
  \includegraphics[width=2.87cm, height=2.87cm]{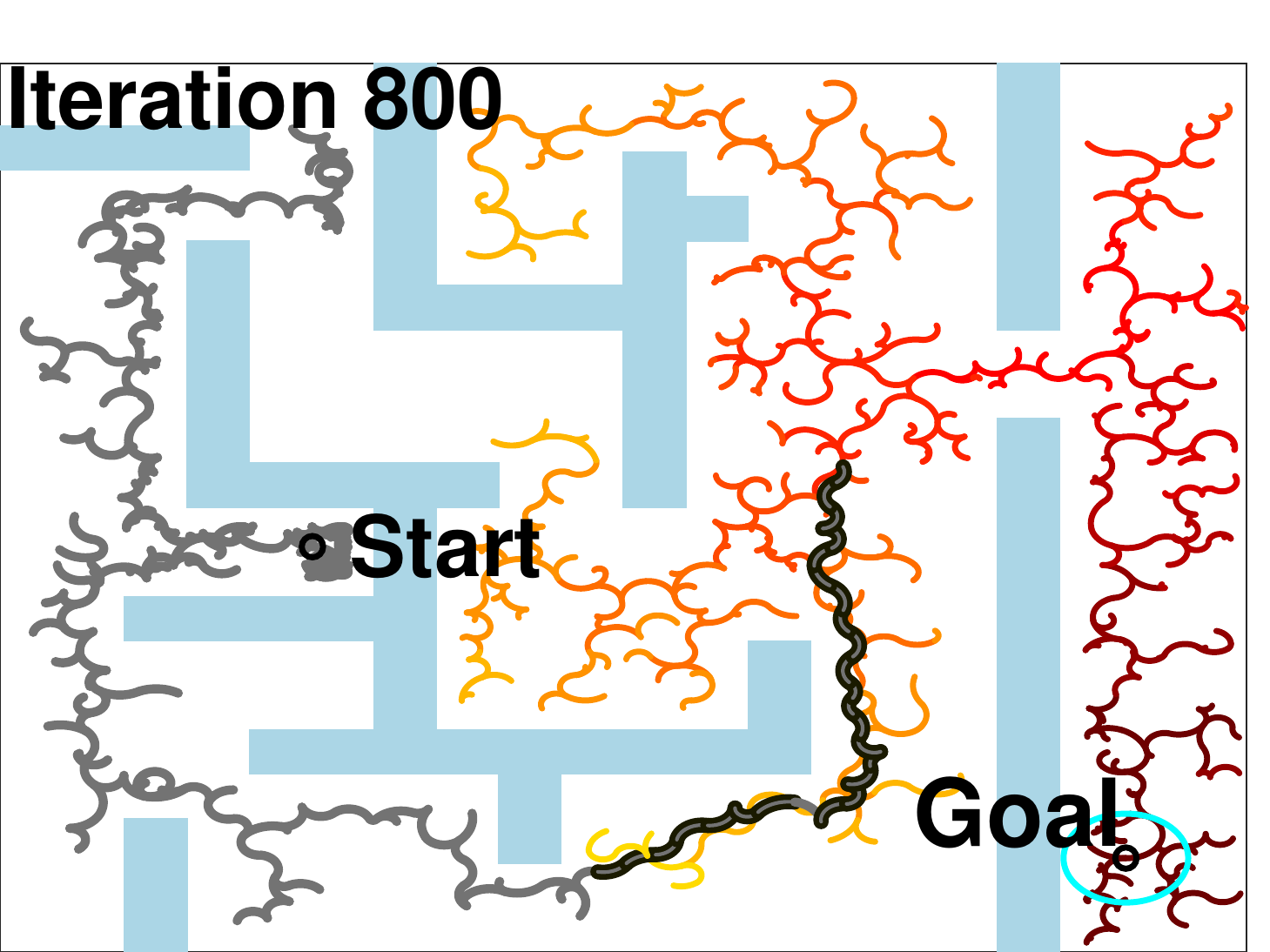}
\end{subfigure}
\begin{subfigure}
  \centering
  \includegraphics[width=2.87cm, height=2.87cm]{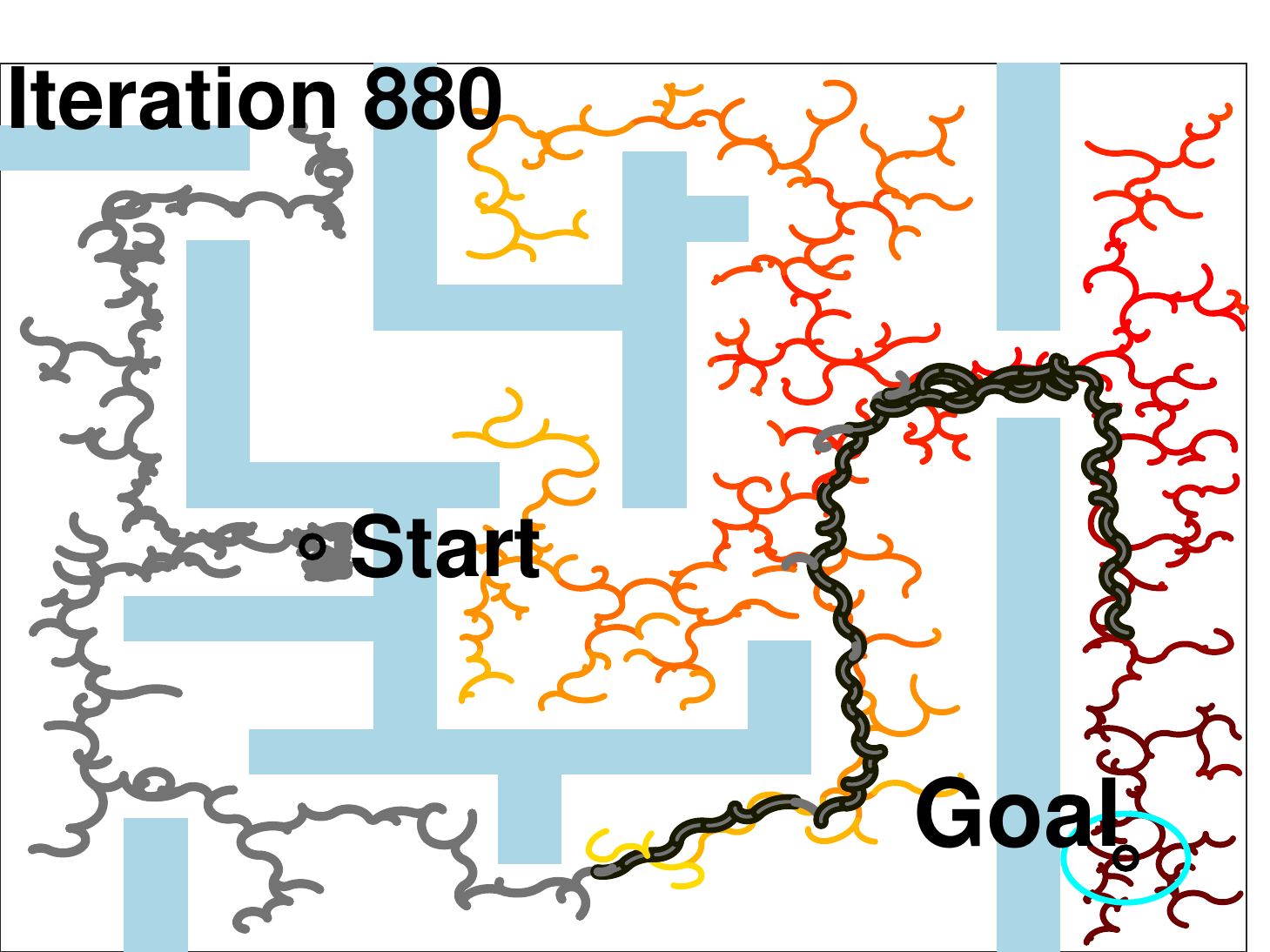}
\end{subfigure}%
\begin{subfigure}
\centering
  \includegraphics[width=2.87cm, height=2.87cm]{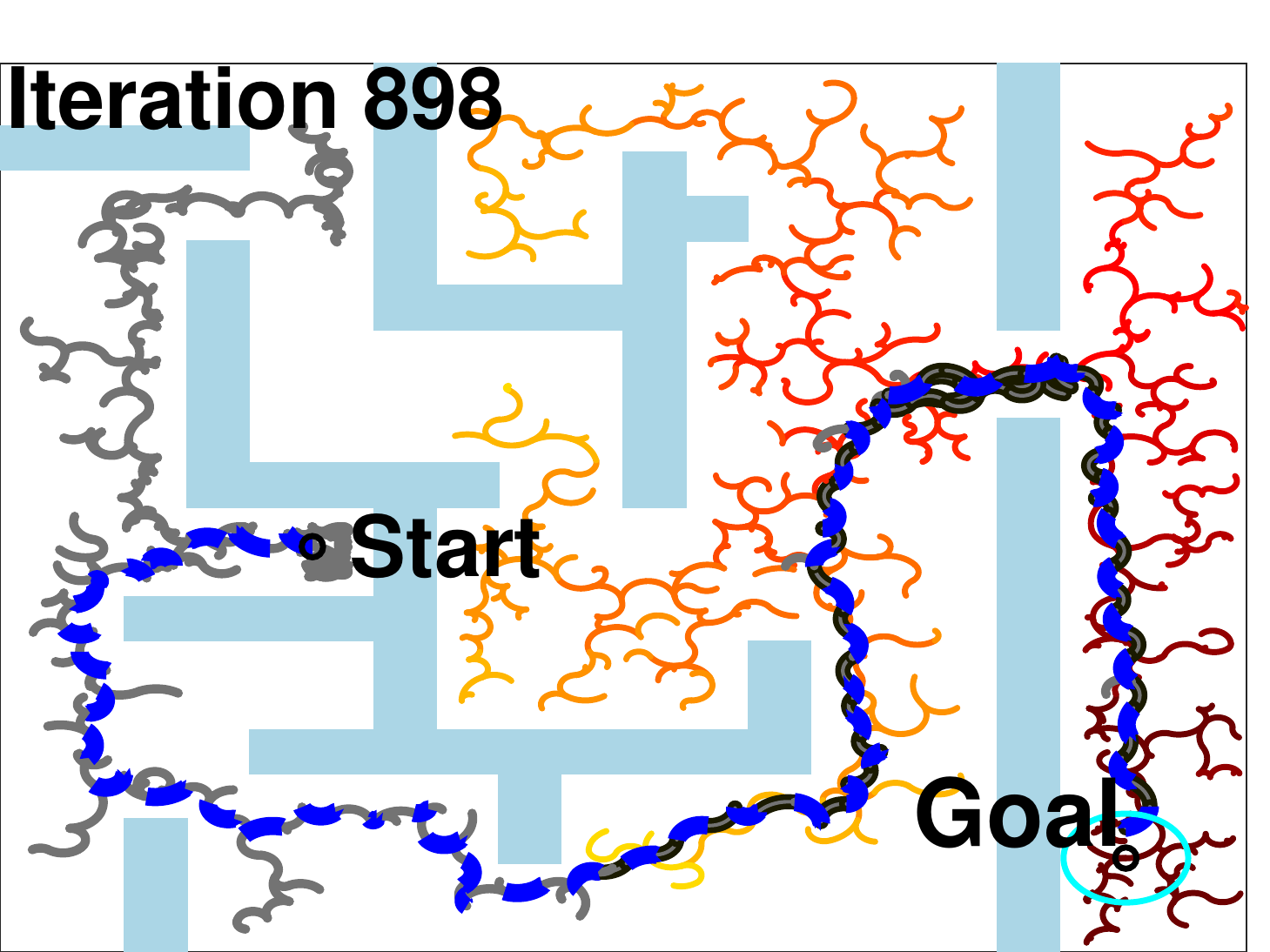}
\end{subfigure}
\begin{subfigure}
\centering
  \includegraphics[scale=0.6]{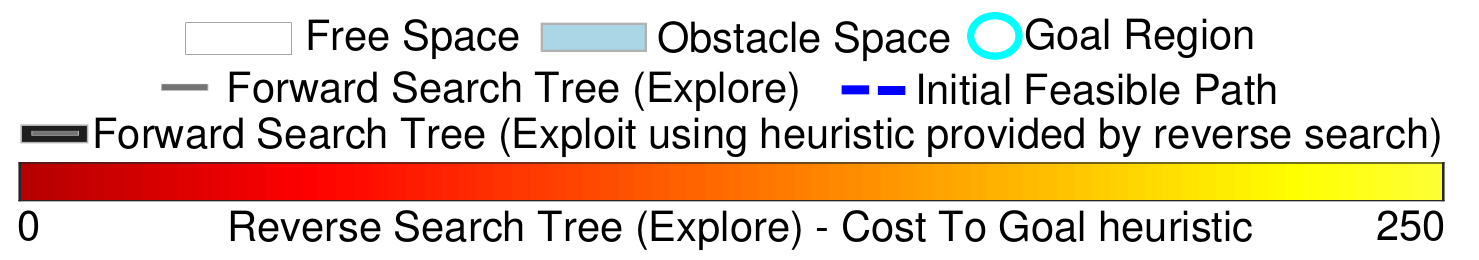}
\end{subfigure}
\vspace{-3mm}
\setlength{\belowcaptionskip}{-11pt}
\caption{Proposed Generalized Bidirectional RRT (GBRRT) algorithm, employing the forward search tree to utilize a combination of exploration and exploitation (heuristic provided by the reverse search tree) search to quickly find a feasible path in a maze.}
\label{fig:VisualAbstract}
\end{figure}

\par The rest of the article is organized as follows. \secref{sec:related_work} discusses related  work. \secref{sec:preliminaries} \revision{contains} the preliminaries, including the definition of terms used throughout the article. \secref{sec:problem_definition} provides a formal definition of our work. 
\secref{sec:algorithm_description} discusses our proposed algorithm\revision{s}. \secref{sec:specificImplementation} presents \revision{additional details related to the specific implementations used in our experiments}. \secref{sec:analysis} provides a proof of probabilistic completeness.
\secref{sec:experimental_setup} describes the setup used for running experiments. \secref{sec:results} presents the result of the experiments and limitations of our algorithm\revision{s}. \secref{sec:conclusion} concludes by summarizing the contributions and main results. An appendix contains additional experiments and runtime analysis.

\section{RELATED WORK}
\label{sec:related_work}
In the following subsections, we discuss related work on existing bidirectional sampling-based motion planning algorithms, algorithms addressing the lack of two-point BVP, and algorithms that use heuristics to improve convergence.

\subsection{Bidirectional sampling-based algorithms}
\label{subSection:BidirectionalAlgs}
There have been a variety of bidirectional sampling-based motion planning algorithms proposed in the literature. A seminal work is bidirectional RRT-Connect \cite{lavalle2001randomized} which grows two trees toward each other by having each tree grow toward the nearest vertex of the opposite tree.

Jordan et al. \cite{jordan2013optimal} develop a bidirectional variant of RRT* \cite{karaman2011sampling} that provides asymptotic optimality utilizing various heuristics and procedures to improve the convergence rate. Tahir et al. \cite{tahir2018potentially} and Xinyu et al. \cite{xinyu2019bidirectional} use artificial potential fields \cite{khatib1986real} to generate intelligent samples using gradient descent methods to improve speed of convergence of bidirectional RRT*. Devaurs et al. \cite{devaurs2013enhancing} update the unidirectional Transition-RRT (T-RRT) algorithm to create a bidirectional variant that achieves fast convergence and better quality paths over configuration spaces where cost functions are defined. Starek et al. \cite{starek2014bidirectional} propose a bidirectional approach to Fast Marching Trees (FMT*), which they show has better convergence rates to high-quality paths than many existing unidirectional planners. Strub et al. \cite{aitstar} use an asymmetric bidirectional search in Adaptively Informed Trees (AIT*) with the reverse tree created using Lifelong Planning A* (LPA*). 
Although the algorithms mentioned above provide fast convergence rates, they cannot be used (in their current form) when the 2-point BVP cannot be solved, or used in practice when doing so is computationally expensive. 

Bidirectional algorithms \cite{lavalle2001randomized,liu2019goal} that do not require solving 2-point BVP, instead require extra computation to connect the discontinuity that exists between the two trees, e.g., using methods like point perturbation \cite{lavalle2001randomized} or Bézier curves \cite{liu2019goal}. In contrast, our approach avoids the tree-tree connection problem entirely by using the heuristic provided by the reverse search tree to grow \revisionTwo{the forward tree} towards the goal.

\subsection{\revision{Solving two-point BVP}}
\label{subSection:twoPointBVPAbsence}
\par Solving a 2-point BVP usually involves solving a differential equation constrained to the given start and end boundary conditions. This is non-trivial for robots with complex dynamics. Hence, researchers have \revision{explored} solving the motion planning problem \revision{either} by generating approximations to the two-point BVP \cite{hwan2011anytime}, using shooting approaches \cite{osborne1969shooting}, or by simplifying the dynamics by linearization \cite{webb2013kinodynamic} and then solving the \revision{resulting linear} 2-point BVP. More recently, \revision{a number of} random forward-propagation-based algorithms have been proposed that do not require solving the 2-point BVP. These include the Stable Sparse-RRT (SST) \cite{li2016asymptotically}, Informed SST (iSST) \cite{littlefield2018informed}, Dominance-Informed Region Trees (DIRT) \cite{littlefield2018efficient}, Asymptotically Optimal- RRT (AO-RRT) \cite{kleinbort2020refined} \cite{hauser2016asymptotically} and \revisionTwo{the algorithm using bundles of edges (BOE) \cite{shome2021asymptotically}}. These are anytime AO (near) algorithms that provide fast initial solutions (specifically iSST, DIRT, \revisionTwo{BOE} due to the use of heuristics) and then continue to improve their solutions over time with respect to a user-defined metric. A common theme in these random-propagation-based algorithms is that they are unidirectional. However, in expectation, bidirectional searches provide improved efficiency and faster convergence rates \cite{jordan2013optimal,lavalle2006planning} than unidirectional searches. Therefore, we present a new type of bidirectional search that does not require solving 2-point BVP. Some of the comparison algorithms included in our experiments \revisionTwo{(\secref{sec:experimental_setup})} are AO algorithms; \revisionTwo{however, we consider only their} feasible planning phase of operation---the discovery of the first feasible solution (the phase of operation for which GBRRT could potentially be used as a drop-in replacement).

\begin{figure}[ht]
\captionsetup{font={color=black, small}}
\centering
\begin{subfigure}
  \centering
  \includegraphics[scale=0.3]{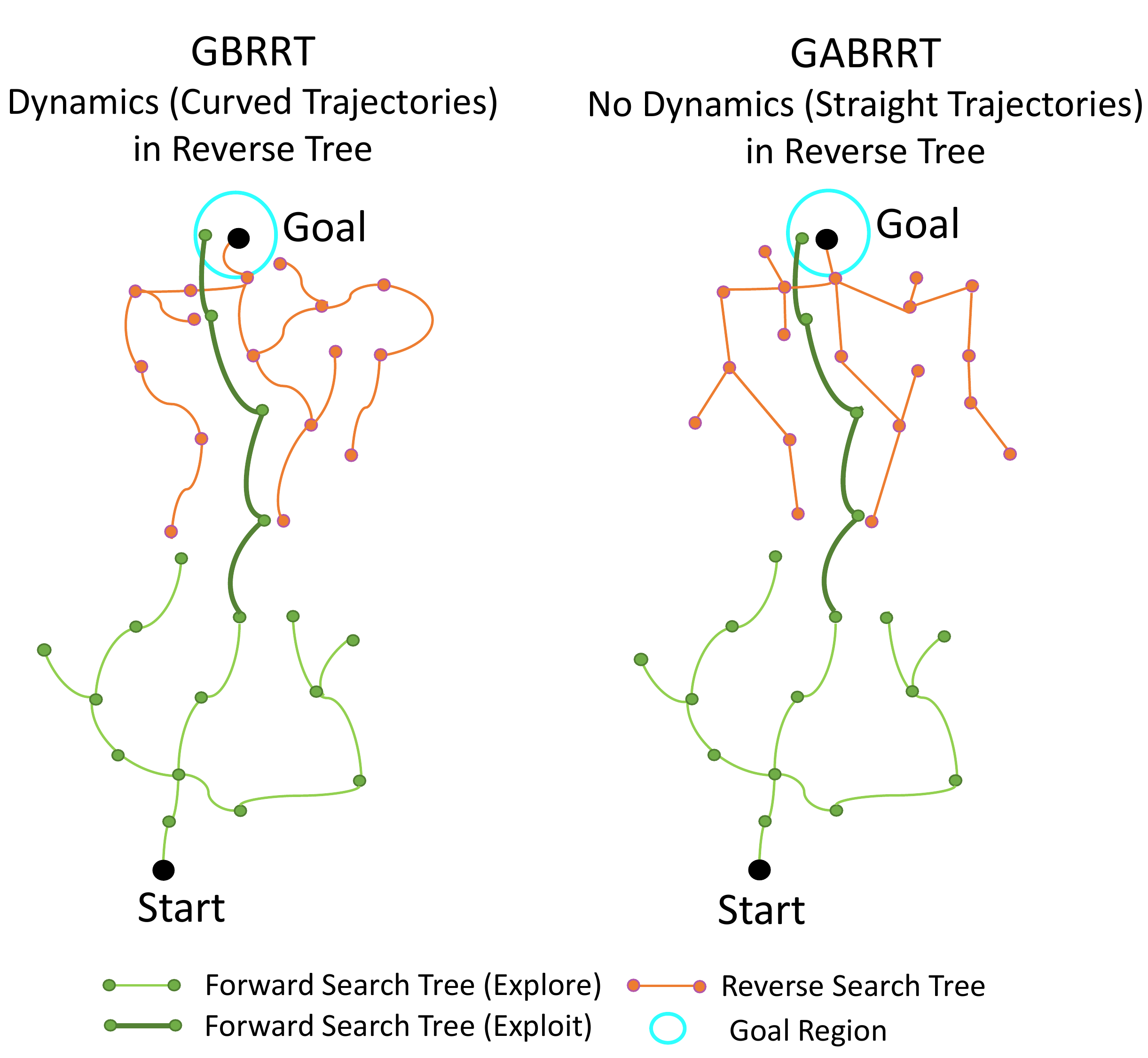}
\end{subfigure}%
\caption{Left: GBRRT (with dynamics in the reverse tree). Right: GABRRT (with no dynamics in the reverse tree).}
\label{fig:BiDirRRTVariant}
\vspace{-0.3cm}
\end{figure}

\subsection{Use of heuristics to improve solution convergence}
\label{subSection:HeuristicUsage}
\par A variety of heuristics have previously been used in the context of sampling-based motion planning because a well-chosen heuristic can often improve search efficiency. One common approach is sampling the goal node with a certain probability, also known as goal biasing \cite{lavalle2001randomized}. Akgun et al. \cite{akgun2011sampling} use sampling heuristics to sample more nodes near a found path to \revisionTwo{increase the rate of finding} a more optimal one. Urmson et al. \cite{urmson2003approaches} propose the heuristically guided RRT (hRRT) where a quality measure is used to balance exploration vs. exploitation when selecting the next node to expand, leading to lower-cost solutions. \revision{Anytime-RRT \cite{ferguson2006anytime} presents a general heuristic for rejecting samples in the context of anytime planning whenever the heuristic cost from the start to the new sample plus the heuristic cost from the sample to the goal is greater than the length of the best known path. C-FOREST \cite{otte2013c} extends this idea to AO sampling-based motion planning on both single-process and multi-process implementations.
In $D$-dimensional Euclidean spaces ($\mathbb{R}^D$), the sample region from \cite{ferguson2006anytime, otte2013c} takes the special form of a hyperellipse;  Informed RRT* extends the idea from \cite{otte2013c} by showing how to generate new samples from the hyperellipse embedded in $\mathbb{R}^D$  directly ({\it without} using rejection sampling).}

\par DIRT \cite{littlefield2018efficient}  and iSST \cite{littlefield2018informed} use user-defined heuristics to provide high rates of convergence towards optimal solution. However, sometimes it is difficult to come up with a heuristic that provides a benefit in a particular problem instance, even though heuristic for the entire problem domain may be easy to formulate \cite{aitstar}. To alleviate this issue, Strub et al. \revision{propose a geometric planner} \cite{aitstar} that uses the problem-specific cost-to-go heuristic provided by the ever-improving reverse search tree obtained from LPA* to converge towards the optimum. \revision{Our methods use a similar approach to \cite{aitstar} in using the reverse search tree to get the cost-to-goal heuristic but differ in that our algorithms do not require solving the two-point BVP.} \footnote{We note that \cite{lavalle2001randomized} and  \cite{urmson2003approaches} use heuristics for feasible (initial solution) planners. \cite{aitstar}, \cite{littlefield2018informed}, \cite{littlefield2018efficient}, \cite{akgun2011sampling}, \cite{ferguson2006anytime}, and \cite{otte2013c} use heuristics for AO (initial and subsequent solutions) planners.}
\par A technical report describing our preliminary research on bidirectional search without solving the two-point boundary value problem was released in October 2020 \cite{arxivArticle}.

\section{PRELIMINARIES}
\label{sec:preliminaries}
Let $\mathcal{X}$   be a smooth $n$-dimensional (manifold) state-space of a robot. \revision{The state-space may contain geometry states like position and angles, and higher-order dynamical states such as velocities, accelerations, etc.} Let the obstacle space $\mathcal{X}_{obs}$ be an open subset of $\mathcal{X}$. $\mathcal{X}_{obs}$ is the set of all states where the system is in collision with obstacles, e.g, in the workspace. The free space $\mathcal{X}_{free}$ is defined as $\mathcal{X}_{free} = \mathcal{X}_{obs}\setminus\mathcal{X}$. $\mathcal{X}_{free}$ is the closed subset of  $\mathcal{X}$ and \revision{set of all states that are not defined by obstacle collisions.} Let the start state be denoted by $x_{start}$, goal state $x_{goal}$ and the goal region be $\mathcal{X}_{goal}$ with ${x_{start}\in \mathcal{X}_{free}}$, and ${x_{goal}\in \mathcal{X}_{goal}}$ and ${\mathcal{X}_{goal} \subset \mathcal{X}_{free}}$. Let $\mathcal{U}$ be the \revisionTwo{$m$-dimensional} input control set. The system dynamics satisfies the differential equation of the form:
\begin{equation}\label{systemDiffEq}
\begin{aligned}
\dot{x}(t) &= f\paren{x(t), u(t)}, \hspace{2mm} x(t) \in \mathcal{X}, \hspace{2mm} u(t) \in \mathcal{U}.  \\
\end{aligned}
\end{equation}
System \eqref{systemDiffEq} is \revision{assumed to be} Lipschitz continuous in both state and input arguments if there exists $K_{x}$, $K_{u} > 0$ such that for all $x_{1}$, $x_{2} \in \mathcal{X}$, $u_{1}$, $u_{2} \in \mathcal{U}$:

\begin{equation}\label{lispchitzContinuity}
\begin{aligned}
\norm{f(x_{1}, u_{1}) - f(x_{2}, u_{1})} \leq K_{x}\norm{x_{1} - x_{2}}, \\
\norm{f(x_{1}, u_{1}) - f(x_{1}, u_{2})} \leq K_{u}\norm{u_{1} - u_{2}},
\end{aligned}
\end{equation}
where $\norm{\cdot}$ is the \revision{Euclidean} norm. \revision{This assumption is required for proving the probabilistic completeness of our proposed algorithms (\secref{sec:analysis}).}

\begin{definition}
A valid trajectory (also called edge) $\mathcal{E}$ of duration $t_{\mathcal{E}}$ is a continuous function that satisfies ${\mathcal{E} : [0, t_{\mathcal{E}}] \rightarrow \mathcal{X}_{free}}$, where $\mathcal{E}$ is generated by forward integrating system \eqref{systemDiffEq} using a control function $\Upsilon : [0, t_{\mathcal{E}}] \rightarrow \mathcal{U}$.
\end{definition}
\noindent This definition is valid for trajectories generated during forward search. \revision{We generate reverse tree trajectories by performing backwards integration of system \eqref{systemDiffEq}}. 
\par \revision{In the specific implementations used in our experiments, $t_{\mathcal{E}} \in \left[0, T_{max} \right]$ where $T_{max}$ is a user-defined maximum propagation time.} We randomly vary $t_{\mathcal{E}}$ but apply a \revision{constant control input} for the entire duration of  $t_{\mathcal{E}}$ as done in \cite{li2016asymptotically,kleinbort2018probabilistic}. \revision{Although constant control is useful in practice due to the discretization inherent in digital controllers \cite{li2016asymptotically}, the application of constant control is not a general requirement of our proposed algorithms.}

\begin{definition}
A distance metric $d_{M}$ is a function that satisfies \mbox{$d_{M}$ : $\mathcal{X}\times\mathcal{X} \rightarrow \mathbb{R}^{+}$} where $\times$ denotes the cartesian product and  $\mathbb{R}^{+}$ is the set of non-negative real numbers.
Assuming $x_{1}, x_{2}, x_{3} \in \mathcal{X}$, $d_{M}$ satisfies the following properties.
\begin{itemize}
    \item Non-negativity:  $d_{M}(x_{1}, x_{2}) \geq 0$
    \item Identity of indiscernibles: \scalebox{0.92}{$d_{M}(x_{1}, x_{2}) = 0 \iff x_{1} = x_{2}$}
    \item Symmetry: $d_{M}(x_{1}, x_{2}) =  d_{M}(x_{2}, x_{1})$
    \item Triangle inequality: \scalebox{0.85}{\mbox{$d_{M}(x_{1}, x_{2}) + d_{M}(x_{2}, x_{3}) \geq d_{M}(x_{1}, x_{3})$}}
\end{itemize}

\end{definition}

\noindent \revision{GBRRT and GABRRT only require a distance function that satisfies the triangle inequality property, which we define next.}

\begin{definition}
\revisionAlgorithm
A distance function $d_{X}$ is a function that satisfies \mbox{$d_{X}$ : $\mathcal{X}\times\mathcal{X} \rightarrow \mathbb{R}$} where $\mathbb{R}$ is the set of real numbers.
\end{definition}
\par \noindent \revision{
Unlike $d_{M}$, $d_{X}$ must only satisfy the triangle inequality property. Thus $d_{X}$ has less restrictive requirements than $d_{M}$ and this makes $d_{X}$ more general than $d_{M}$.}

\revision{We also define $d^{\mathtt{ND}}_{X}$ as the distance function on a lower-dimensional subspace of $\mathcal{X}$ where the higher-order dynamics (velocity or accelerations) are not included. The superscript ND refers to ``\textit{No Dynamics}''. GABRRT (\algoref{alg:GABRRT}) uses this distance function to generate the cost-to-go-heuristic from a non-dynamical reverse tree} \revision{(GBRRT, \algoref{alg:GBRRT}, does not).}

\begin{definition}
The cost $\mathcal{C}$ of a trajectory is a function that satisfies $\mathcal{C}:\mathscr{E} \rightarrow \mathbb{R}$ where $\mathscr{E}$ is the space of all trajectories.
\end{definition}
The cost \revision{can be} defined as the distance traveled along a trajectory or the energy or time required to execute the trajectory. In our experiments, we define $\mathcal{C}$ as the distance along a shortest trajectory, 
$
C(\mathcal{E}) = \limsup\limits_{N\rightarrow \infty
} \sum_{i=1}^{N} d_{X}\paren{\mathcal{E}(t_{i-1}), \mathcal{E}(t_{i})}
$,
where $t_{0} \leq t_{1} \leq t_{2} \ldots \leq t_{N}$. By construction, $d_{X}$ satisfies the triangle inequality, and can thus be used to define the ``distance'' between two points for the purposes of our algorithm. The trajectory cost has previously been used in the context of AO planners, but in our work, we use it to define the cost-from-start (Definition 6) and cost-to-goal (Definition 7) heuristics.

\begin{definition}
Let $\mathcal{G} \triangleq \paren{\mathbf{V}, \mathbf{E}}$ be a graph data structure embedded in $\mathcal{X}$ representing a motion search tree. Let $\mathbf{V}$ and $\mathbf{E}$ \revision{represent} the node set and edge set respectively. Let $x_{a}$ and $x_{b}$ represent any two nodes $\in \mathbf{V}$. Then a feasible path $\Gamma_{a,b}$ between $x_{a}$ and $x_{b}$ is defined as
$
{\Gamma_{a,b} \triangleq \{\mathcal{E}_{1}, \mathcal{E}_{2} ... \hspace{0.5mm} \mathcal{E}_{m}\}, \hspace{0.5mm}}
$,
where $\{\mathcal{E}_{1}, \mathcal{E}_{2} ... \hspace{0.5mm} \mathcal{E}_{m}\}$ is an ordered set of m valid trajectories with $\mathcal{E}_{i} \in \mathbf{E}$ and  $\mathcal{E}_{1}(0)=x_{a}$ and $\mathcal{E}_{m}(t_{\mathcal{E}_{m}})=x_{b}$ and ${\mathcal{E}_{i}(t_{\mathcal{E}_{i}}) = \mathcal{E}_{i+1}(0)}$.
\end{definition}

\begin{definition}
Let $\mathcal{G}_{for} \triangleq \paren{\mathbf{V}_{for}, \mathbf{E}_{for}}$ be a graph data structure embedded in $\mathcal{X}$ representing a forward search tree rooted at $x_{start}$. $\mathbf{V}_{for}$ and $\mathbf{E}_{for}$ \revision{represent} the node set and edge set respectively. \revision{Let $\Gamma_{start, a}$ be a feasible path connecting $x_{start}$ and $x_{a} \in \mathbf{V}_{for}$. The cost-from-start heuristic $g$  of the state $x_{a}$ is defined as
$
{g(x_{a}) = \sum_{i=1}^{m} \mathcal{C}\paren{\mathcal{E}_{i}}}
$,
where $\mathcal{E}_{i} \in \Gamma_{start, a}$, $\mathcal{E}_{1}(0)=x_{start}$, 
$\mathcal{E}_{m}(t_{\mathcal{E}_{m}})=x_{a}$, and $m$ is the number of edges connecting $x_{start}$ to $x_{a}$.}
\end{definition}

\begin{definition}
Let $\mathcal{G}_{rev} \triangleq \paren{\mathbf{V}_{rev}, \mathbf{E}_{rev}}$ be a graph data structure embedded in $\mathcal{X}$ representing a reverse search tree rooted at $x_{goal}$. $\mathbf{V}_{rev}$ and $\mathbf{E}_{rev}$ represent the node set and edge set respectively. \revision{Let $\Gamma_{a, goal}$ be a feasible path connecting $x \in \mathbf{V}_{rev}$ and $x_{goal}$. The cost-to-goal heuristic $h$  of a state $x_{a} \in \mathbf{V}_{rev}$ is defined as
$
h(x_{a}) = \sum_{i=1}^{n} \mathcal{C}\paren{\mathcal{E}_{i}},
$
where $\mathcal{E}_{i} \in \Gamma_{a, goal}$, $\mathcal{E}_{0}(0)=x_{a}$, $\mathcal{E}_{n}(t_{\mathcal{E}_{n}})=x_{goal}$, and $n$ is the number of edges connecting $x_{a}$ to $x_{goal}$.}
\end{definition}

\begin{definition} \label{def:clearance}
The clearance of a trajectory $\mathcal{E}$ is the maximal value $\delta $ \revisionTwo{such that} 
$
{\mathcal{B}_{\delta}^{m}\paren{\mathcal{E}(t)} \subset \mathcal{X}_{free} \forall t \in [0, t_{\mathcal{E}}]}
$,
where $\mathcal{B}_{\delta}^{m}$ is the m-dimensional hyper-ball of radius $\delta$ centered along a point in $\mathcal{E}(t)$.
\end{definition}
Definition~\ref{def:clearance} is used in the proof of probabilistic completeness (\secref{sec:analysis}).

\section{PROBLEM DEFINITION}
\label{sec:problem_definition}

Let $\Gamma_{start, g}$ be a feasible path connecting $x_{start}$ and some $x_{g} \in \mathcal{X}_{goal}$.
Note that $x_{g}$ is different from $x_{goal}$ (the starting point for the reverse tree as defined in Definition 7). While $x_{goal}$ is a valid candidate for $x_{g}$ (because $x_{goal} \in \mathcal{X}_{goal}$ by construction), the definition of $x_{g}$ is deliberately more general. $x_{g}$ is allowed to be any valid end state in the goal region $\mathcal{X}_{goal}$. Even though both ${x_{g},x_{goal} \in \mathcal{X}_{goal}}$, they are usually not the same---especially in scenarios where we cannot solve the two-point BVP.

\revisionTwo{The problem definition stated below is for the feasible path planning problem, in general. It does not include a specific $x_{goal}$, because the problem can be solved with either a unidirectional or bidirectional algorithm.}

\noindent\textbf{Problem}: Single-query feasible motion planning. \\
\noindent\textit{Given ($x_{start}$, $\mathcal{X}_{free}$, $\mathcal{X}_{goal}$), find a $\Gamma_{start, g}$} if one exists.

\section{ALGORITHM DESCRIPTION}
\label{sec:algorithm_description}
The main idea behind our \revision{proposed algorithms} is to use heuristic information from the reverse search tree to guide the forward search tree's advancement towards the goal. The search begins by using the forward and reverse search trees to explore different parts of the search space starting from both the start and goal states. Once the two trees encounter each other, the forward tree combines ongoing exploration with an exploitation strategy that leverages the cost-to-goal values stored in the reverse tree to direct the search. This causes the forward search tree to maneuver around obstacles and grow towards the goal quickly.
\par \revision{A design choice is made in our algorithms to maintain a feasible but {\it not} a (near-) optimal reverse tree, even though refining the reverse tree toward optimality could arguably provide a more focused guiding heuristic after the two trees become intertwined. Instead, because our algorithms are designed to solve the feasible motion planning problem, computational effort is focused on expediting the trees' mutual encounters by exploring the free space---so that the guiding heuristic can be used as soon as possible.}  
\par In \subSecref{subSectionGBRRT}, we present the proposed GBRRT algorithm. In \subSecref{subSectionGBRRTSearchesOutline}, we provide the general outline for the exploration and exploitation searches utilized in GBRRT. \revision{In \subSecref{subSection:GABRRT}, we give a high-level description of GABRRT and then provide the algorithm details in Appendix \ref{apdx:GABRRTAlgorithms}.}

\begingroup
\removelatexerror
\begin{figure}[ht]
\centering
\scalebox{0.675}{
\begin{minipage}{10.7cm}
\centering
\begin{algorithm}[H]
    \caption{$\mathtt{GBRRT}(x_{start}, x_{goal}, \mathcal{X}_{goal}, \mathcal{U}, T_{max}, \delta_{hr}, \mathcal{P})$} \label{alg:GBRRT}
    $\mathcal{G}_{for} \gets \{\mathbf{V}_{for} \gets \{x_{start}\}, \mathbf{E}_{for} \gets \{\}\}$ \\ \label{alg:GBRRT:initGFor}
     $\mathcal{G}_{rev} \gets \{\mathbf{V}_{rev} \gets \{x_{goal}\}, \mathbf{E}_{rev} \gets \{\}\}$ \\ \label{alg:GBRRT:initGrev}
     $\mathbf{Q} \gets \{\}$ \\ \label{alg:GBRRT:initQ}
    
    \For{$k \gets 1$ to $M_{iter}$}
    {
        \revision{$r_{k} \gets min(\gamma {\paren{\frac{\log{\abs{\mathbf{V_{rev}}}}}{\abs{\mathbf{V_{rev}}}}}}^{\frac{1}{d + 1}}, \, \delta_{hr})$} \\ \label{alg:GBRRT:ShrinkingRadius}
        
        \tcp{Reverse Tree Expansion (lines 6-11)}
         $\mathcal{E}_{rev} \gets \mathtt{RevSrchFastExplore}(\mathcal{G}_{rev}, \mathcal{U}, T_{max})$ \\ \label{alg:GBRRT:ReverseExpansion}
         
         \If {$\mathtt{not \, CollisionCheck}(\mathcal{E}_{rev})$ \label{alg:GBRRT:collisioncheckReverse}}
         {
             
              $x_{rev} \gets \mathcal{E}_{rev}.\mathtt{initialNode()}$ \\
              \label{alg:GBRRT:AssignXRev}
              $\mathbf{E}_{rev} \gets \mathbf{E}_{rev} \cup \{\mathcal{E}_{rev}\}$ \\ \label{alg:GBRRT:updateErev}
             $\mathbf{V}_{rev} \gets \mathbf{V}_{rev} \cup \{x_{rev} \}$ \\ \label{alg:GBRRT:updateVrev}
             
              \revision{$\mathtt{updatePriorityQueue} \paren{\mathcal{G}_{for}, \mathbf{Q},  x_{rev}, r_{k}}$} \\ \label{alg:GBRRT:updatePriorityQueue}
         }
    
        \tcp{Forward Tree Expansion (lines 12-30)}
         $ q \gets \mathcal{P}(k)$ \\ \label{alg:GBRRT:updateExploitRatio}
         
         $c_{rand} \sim \mathit{U}\paren{[0, 1]}$ \\ \label{alg:GBRRT:initp}

        \If {$c_{rand} < q$ \label{alg:getEdge:exploitRatio}}
        {
        $\mathcal{E}_{for} = \mathtt{NULL}$ \\ \label{alg:GBRRT:Enull}
        
        \tcp{Pop node from $\mathbf{Q}$}
        $x_{pop} \gets \mathtt{Pop}\paren{\mathbf{Q}}$ \\ \label{alg:GBRRT:PopQ}
        
        \If {$x_{pop} \neq \mathtt{NULL}$}
        {
                $\mathcal{E}_{for} \gets \mathtt{ForSrchExploit}\paren{\mathcal{G}_{rev}, \mathcal{U}, T_{max}, x_{pop}, \revision{r_{k}}}$ \\ \label{alg:GBRRT:ForwardSearchExploit}
        }
        
        \If {$\mathcal{E}_{for} = \mathtt{NULL}$}
            {
        $\mathcal{E}_{for} \gets \mathtt{ForSrchFastExplore}\paren{\mathcal{G}_{for}, \mathcal{U}, T_{max}}$ \\ \label{alg:GBRRT:ForwardSearchFastExplore}
            }
        }
  
       \If {$\revision{c_{rand}} \geq q \;\; \mathtt{or} \;\; \mathcal{E}_{for} = \mathtt{NULL}$}
       {
        \tcp{For Probabilistic completeness}
        $\mathcal{E}_{for} \gets \mathtt{ForSrchRandomExplore}\paren{\mathcal{G}_{for}, \mathcal{U}, T_{max}}$ \\ \label{alg:GBRRT:ForwardSearchRandomExplore}
       }
        
        \If {$\mathcal{E}_{for} \neq \mathtt{NULL}$}
        {
         
         \If {$\mathtt{not \, CollisionCheck}(\mathcal{E}_{for})$ \label{alg:GBRRT:collisioncheckForward}}
         {
             $x_{for} \gets \mathcal{E}_{for}.\mathtt{finalNode()}$ \\
              \label{alg:GBRRT:AssignXfor}
             $\mathbf{E}_{for} \gets \mathbf{E}_{for} \cup \{\mathcal{E}_{for}\}$ \\ \label{alg:GBRRT:updateEfor}
             $\mathbf{V}_{for} \gets \mathbf{V}_{for} \cup \{x_{for} \}$ \\ \label{alg:GBRRT:updateVfor}
            
            \If {$\mathtt{goalRegionReached}(\mathcal{X}_{goal}, x_{for})$ \label{alg:GBRRT:goalRegionReached}}
            {
                \Return {$\mathtt{Path}(\mathcal{G}_{for}, x_{start}, x_{for})$} \\ \label{alg:GBRRT:Path}
            }
            
            \revision{ $\mathtt{insertToPriorityQueue} \paren{\mathcal{G}_{rev}, \mathbf{Q},  x_{for}, r_{k}}$} \\ \label{alg:GBRRT:insertPriorityQueue}
         }
        }
    }
 \Return {$\mathtt{NULL}$}
\end{algorithm}
\end{minipage}}
\vspace{-2mm}
\end{figure}
\endgroup

\begin{figure}[ht]
\captionsetup{font={color=black, small}}
\centering
\begin{subfigure}
\centering
\includegraphics[scale=0.28]{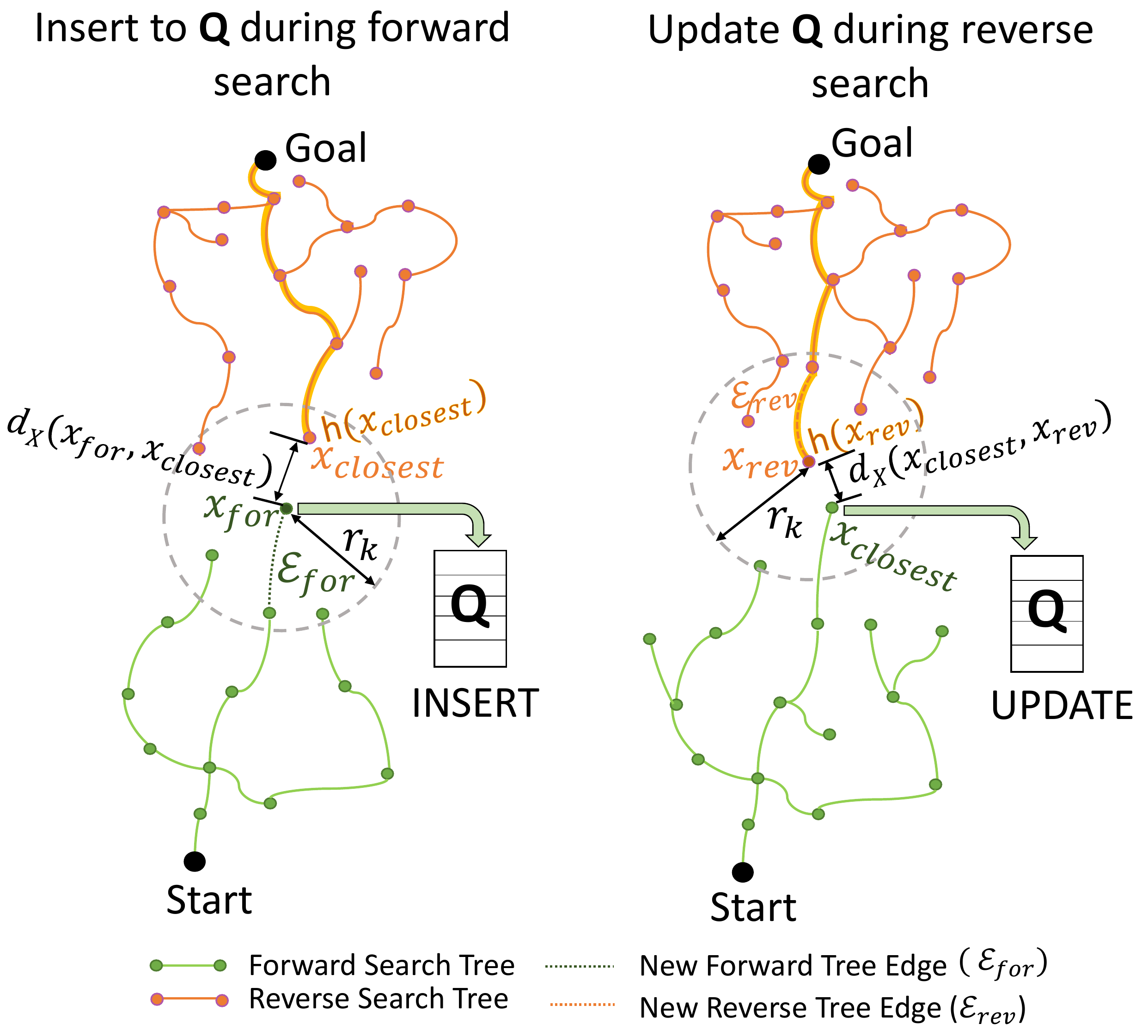}
\end{subfigure}
\caption{Left: A new forward tree node $x_{for}$ pushed to priority queue $\mathbf{Q}$ with key \mbox{$d_{X}(x_{for}, x_{closest}) + h(x_{closest})$} during forward tree expansion. Right: the cost of the same forward tree node being updated in $\mathbf{Q}$ during a reverse search expansion. The cost is updated only when the new cost \mbox{$d_{X}(x_{closest}, x_{rev}) + h(x_{rev})$} is less than its current key in $\mathbf{Q}$. \revisionTwo{The trees in the right sub-figure differ from those in the left to indicate that the cost update of $x_{for}$ (right) may occur a few iterations after the insertion operation (left) --- during which the trees may have been updated.}}
\label{fig:GBRRPushBoth}
\end{figure}

\begin{figure*}[ht]
\begin{minipage}[c]{0.55\textwidth}
\figuretitle{Search types used in forward expansion of GBRRT}
\includegraphics[scale=0.3]{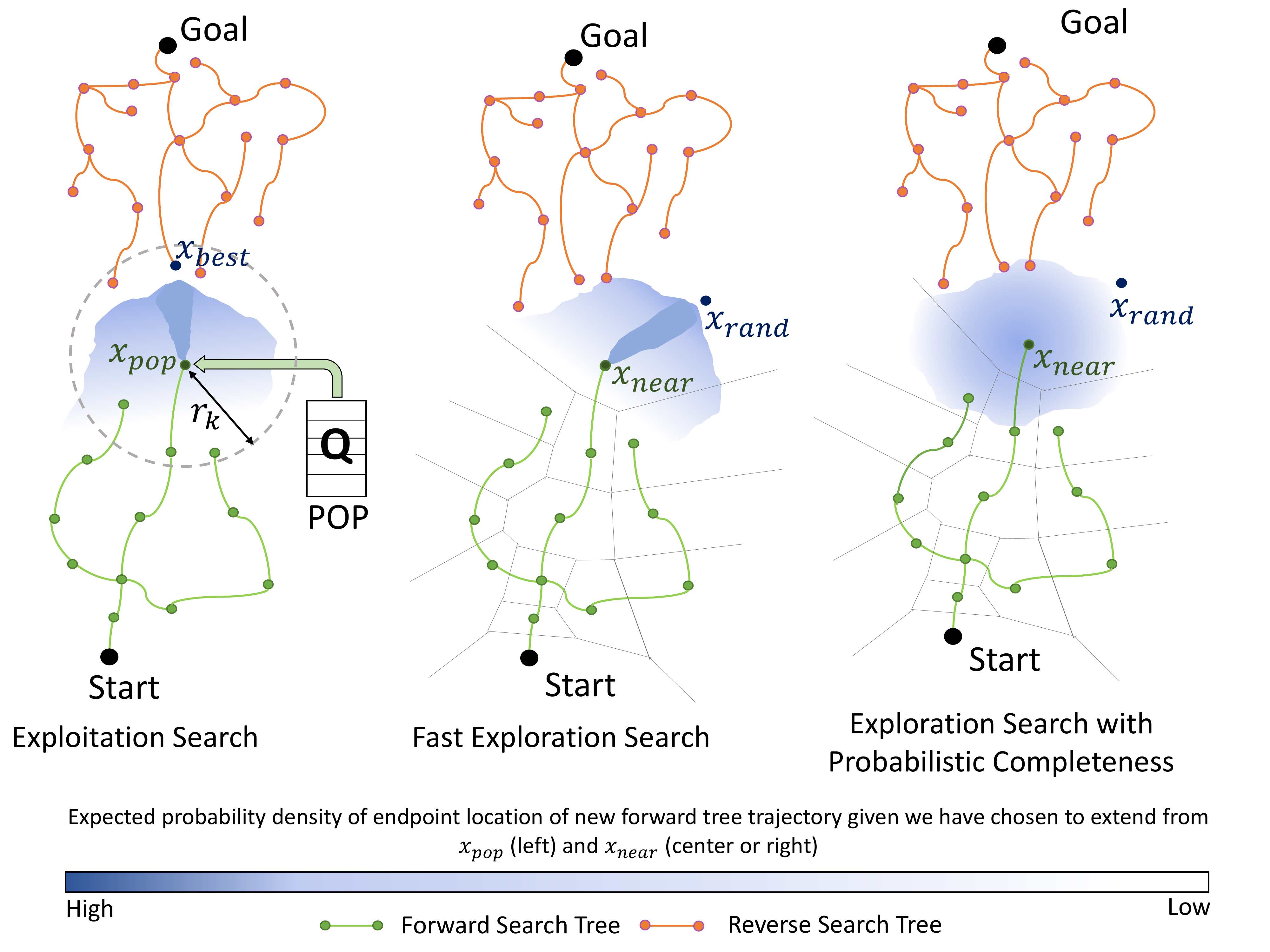}
\end{minipage}
\begin{minipage}[c]{0.45\textwidth}
\caption{\revisionTwo{\textbf{Left}: In exploitation search a forward tree node $x_{pop}$ is popped from $\mathbf{Q}$ and extension is biased in the direction of $x_{best}$, the best reverse tree node according to a user-defined cost function. The relative bias over the new endpoint's location is depicted as a tear shape oriented in the direction of $x_{best}$.
\textbf{Center}: In fast exploration search a random node $x_{rand}$ is sampled from the entire state-space and extension is biased in the direction of $x_{rand}$  from $x_{near}$, the nearest neighbor of $x_{rand}$ in the forward search tree. The relative bias over the new endpoint's location is depicted as a tear shape oriented in the direction of $x_{rand}$.
\textbf{Right}: In random exploration search a random node $x_{rand}$ is sampled from the entire state-space and extension from $x_{near}$ uses constant random control. The probability density of endpoint location is depicted as being spread out around $x_{near}$. Random exploration guarantees probabilistic completeness, but is not biased toward the search-tree frontier.
\textbf{All}: The new edges/trajectories are limited in duration, and duration may be random. Probability density blobs are depicted as being non-uniform to highlight the fact that the endpoint depends on $f\paren{x(t), u(t)}$, $\mathcal{U}$ and $T_{max}$. The embedded Voronoi diagram depicts which regions of space are nearest to each node in the forward tree. The Voronoi diagram is for illustrative purposes and is not calculated in GBRRT.}} 
\label{fig:GBRRTSearchTypes}
\end{minipage}
\end{figure*}

\begingroup
\removelatexerror
\begin{figure*}[ht]
\begin{minipage}[c]{0.55\textwidth}
\figuretitle{Specific edge-generation types used in our version of GBRRT}
\includegraphics[scale=0.3]{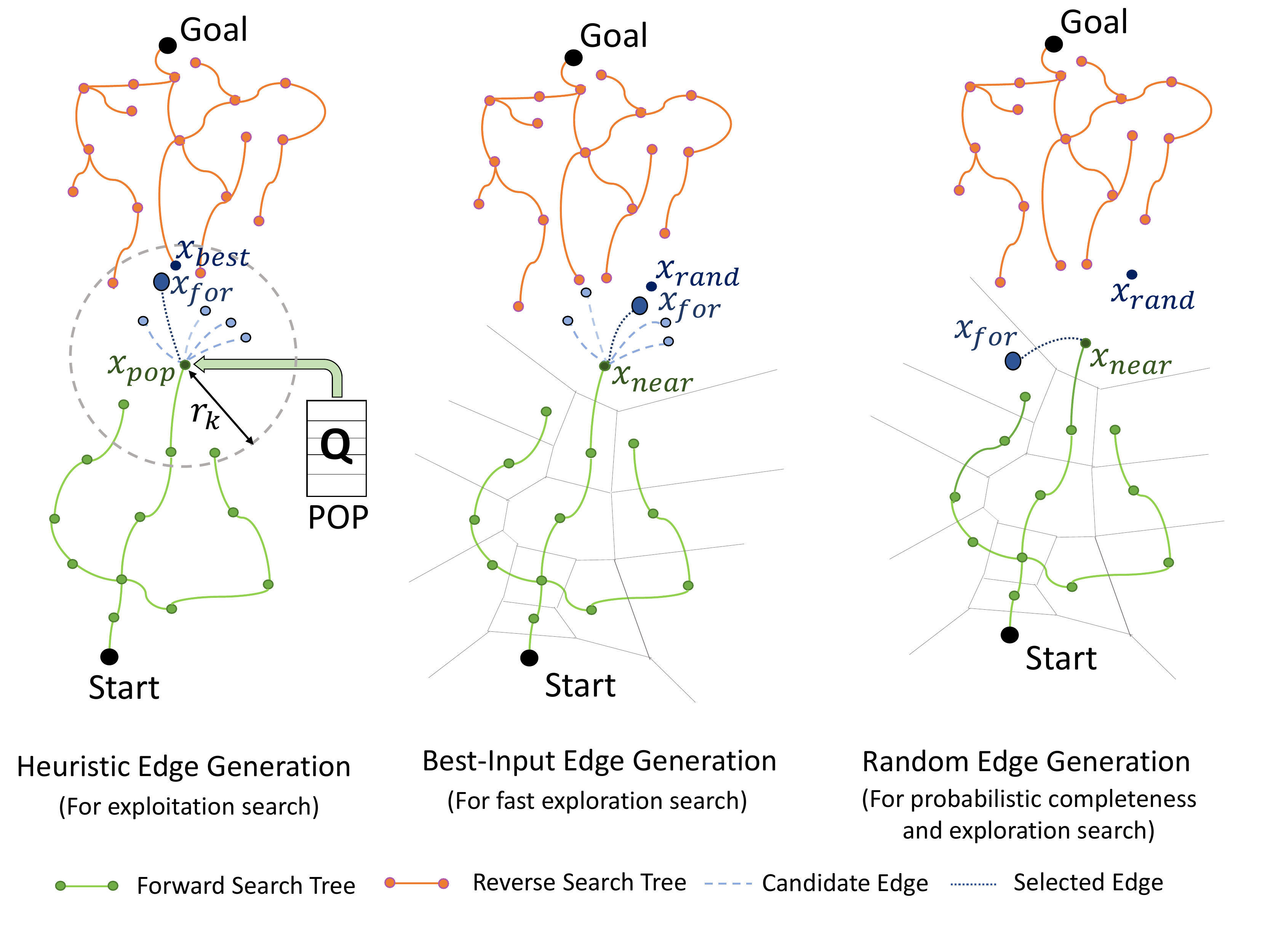}
\end{minipage}
\begin{minipage}[c]{0.45\textwidth}
\begin{minipage}[r]{\textwidth}
\caption{These figures illustrate the specific implementation of the searches outlined  in \figref{fig:GBRRTSearchTypes}. The left figure shows heuristic edge-generation where a series of random edges are propagated from $x_{pop}$ (node popped from $\mathbf{Q}$) and the best edge is chosen as the one whose final node $x_{for}$ get closest to $x_{best}$. The node $x_{best}$ is chosen using the cost function $g(x_{pop}) + d_{X}(x_{pop}, x) + h(x)$ where $x$ is the node in the reverse tree. The middle figure presents best-input edge generation where a series of random edges are propagated from $x_{near}$ and the best edge is chosen as the one that gets closest to $x_{rand}$. \revision{Observe that the selected edges in the left and middle figures lie in the darker tear-shaped region of probability density in \figref{fig:GBRRTSearchTypes}}. The right figure shows random edge-generation where an edge of random duration is propagated from $x_{near}$ to maintain probabilistic completeness.} 
\label{fig:GBRRTSpecificPropagation}
\end{minipage}
\scalebox{\algorithmScale}{
\begin{minipage}{10cm}
\centering
\begin{algorithm}[H]
  \caption{$\revision{\mathtt{insertToPriorityQueue}}(\mathcal{G}_{rev}, \mathbf{Q}, x_{for}, \revision{r_{k}})$} \label{alg:insertToPriorityQueue}
  
  \revisionTwo{$x_{closest} \gets \mathcal{G}_{rev}.\mathtt{NearestNeighbor}(x_{for})$} \\
  \label{alg:insertToPriorityQueue:NearestVertices}
  
  $\delta \gets d_{X}(x_{for}, x_{closest})$

  \If{\revisionTwo{$\delta \leq r_{k}$}}
  {
      $\mathbf{Q}\revision{\mathtt{.insert}}\paren{x_{for}, \delta +  h(x_{closest})}$ \\ \label{alg:insertToPriorityQueue:addToQ}
  }
      
\end{algorithm}
\end{minipage}}
\scalebox{\algorithmScale}{
\revisionAlgorithm
\begin{minipage}{10cm}
\begin{algorithm}[H]
  \caption{$\mathtt{updatePriorityQueue}(\mathcal{G}_{for}, \mathbf{Q}, x_{rev}, r_{k})$} \label{alg:updatePriorityQueue}
  
     \revisionTwo{$x_{closest} \gets \mathcal{G}_{for}.\mathtt{NearestNeighbor}(x_{rev})$} \\ \label{alg:updatePriorityQueue:NearestVertices}
     
      $\delta \gets d_{X}(x_{closest}, x_{rev})$
     
      \If{\revisionTwo{$\delta \leq r_{k}$}}
      {
          $\mathbf{Q}\mathtt{.update}\paren{x_{closest}, \delta +  h(x_{rev})}$ \\ \label{alg:updatePriorityQueue:addToQ}
      }
\end{algorithm}
\end{minipage}}
\end{minipage}
\vspace{-6mm}
\end{figure*}
\endgroup

\subsection{Proposed Generalized Bidirectional RRT (GBRRT)}
\label{subSectionGBRRT}
\subsubsection*{\textbf{Inputs}}
\revision{The inputs to the GBRRT (\algoref{alg:GBRRT}) include the scenario specific start node $x_{start}$, goal node $x_{goal}$, and goal region $\mathcal{X}_{goal}$; the system specific control input set $\mathcal{U}$; as well as a max propagation time $T_{max}$ (maximum duration of a trajectory); and two user defined parameters $\delta_{hr}$ and $\mathcal{P}$.
The parameter $\delta_{hr}$ defines the \revisionTwo{maximum} distance within which nodes from the reverse tree have a focusing effect on the growth of the forward tree, \revisionTwo{where ${0 < \delta_{hr} < \infty}$ ($hr$ stands for heuristic radius)}. The parameter $\mathcal{P}$ defines the algorithm's balance between performing exploration and exploitation during forward search, where \revisionTwo{$0 \leq \mathcal{P} \leq 1$} and, with a slight abuse of notation, 0 and 1 refer to the zero and one function, respectively. $\delta_{hr}$ and $\mathcal{P}$ are discussed more in \subSecref{sub:SelectionParameters}.}

\subsubsection*{\textbf{Initialization}}
The forward search tree $\mathcal{G}_{for}$ is initialized (\alglineref{alg:GBRRT:initGFor}) using the node set $\mathbf{V}_{for}$ (initialized to $x_{start}$) and edge set $\mathbf{E}_{for}$. The reverse tree $\mathcal{G}_{rev}$ is initialized (\alglineref{alg:GBRRT:initGrev}) using the node set $\mathbf{V}_{rev}$ (initialized to $x_{goal}$) and edge set $\mathbf{E}_{rev}$ (\alglineref{alg:GBRRT:initGrev}). We use a priority queue $\mathbf{Q}$ to maintain a list of potential forward tree nodes for quick expansion (\alglineref{alg:GBRRT:initQ}). 

\subsubsection*{\revisionTwo{\textbf{Shrinking neighborhood radius $r_{k}$}}}

\revisionTwo{Shrinking the radius $r_{k}$ of the neighborhood in which the two trees influence each other is necessary to balance the competing desires for a quick expected per-iteration runtime (requiring smaller $r_{k}$) with almost sure (A.S.) asymptotic graph connectivity between $x_{start}$ and ${x_g \in \mathcal{X}_{goal}}$ (requiring larger $r_{k}$) \cite{karaman2011sampling,solovey2020revisiting}.  Functions for $r_{k}$ that have previously appeared in the sampling based motion planning literature include $r_k = \min(\gamma {\parenShort{\frac{\log{\abs{\mathbf{V_{rev}}}}}{\abs{\mathbf{V_{rev}}}}}}^{{1}/{d}}, \, \delta_{hr})$ from \cite{karaman2011sampling} and $r_k = \min(\gamma {\parenShort{\frac{\log{\abs{\mathbf{V_{rev}}}}}{\abs{\mathbf{V_{rev}}}}}}^{{1}/{(d + 1)}}, \, \delta_{hr})$ from \cite{solovey2020revisiting}, where $d$ is the dimension of the system and $\gamma$ is a scenario specific percolation-theoretic parameter%
\footnote{
Regarding $\gamma$, if $r_k$ is defined as in \cite{karaman2011sampling}, then A.S.\  graph connectivity requires 
${\gamma \geq {\parenShort{2((1 + \frac{1}{d})^{{1}/{d}}) \frac{\abs{\mathcal{X}_{free}}}{\zeta_{d}}}}^{{1/d}}}$, where  $\abs{\mathcal{X}_{free}}$ is the volume of free space and $\zeta_{d}$ is the volume of the unit $d$-ball \cite{karaman2011sampling}. 
A new proof in \cite{solovey2020revisiting} uses
${\gamma \geq (2 + \theta)
\parenShort{ 
\frac{(1+\epsilon/4)c^*}{(d+1)\theta(1-\mu)} 
\frac{\abs{\mathcal{X}_{free}}}{\zeta_{d}}
}  ^  {\frac{1}{d+1}}}$, where ${\epsilon \in (0,1)}$ and $\theta \in (0, 1/4)$, and ${\mu > 0}$ , and $c^*$ is the A.S. convergent length of the path as $n \rightarrow \infty$.
Both bounds on $\gamma$ are of a theoretical interest rather than prescriptive because we typically do not know $\abs{\mathcal{X}_{free}}$ or $c^*$, though we may have bounds on them. Practical selection of $\gamma$ requires making an educated guess and refining if necessary.
Anecdotally, we observe that the range of $\gamma$ producing decent results is large in practice.
}. 
In Appendix~\ref{sub:RuntimeComplexity}, we analyze the expected per-iteration runtime that results from using either function as well from using a constant $r_{k}$.}

\par \revisionTwo{In our experiments, we calculate $r_{k}$ using the expression from \cite{solovey2020revisiting}, we discuss why this rate was chosen in Appendix~\ref{sub:RuntimeComplexity}.
$r_{k}$ is used in the insert and update $\mathbf{Q}$ operations (\alglineref{alg:GBRRT:updatePriorityQueue} and \alglineref{alg:GABRRT:insertToPriorityQueue}) and exploitation search (line 18).}

\subsubsection*{\textbf{Reverse tree expansion}}
We expand the reverse search tree in \revision{\alglinerefs{alg:GBRRT:ReverseExpansion}{alg:GBRRT:updatePriorityQueue}}. The reverse tree uses a pure exploration strategy for expansion using $\mathtt{RevSrchFastExplore}$ (\alglineref{alg:GBRRT:ReverseExpansion}). Its purpose is to explore the state space as fast as possible and get close to and provide a focusing heuristic to the forward tree. The edge $\mathcal{E}_{rev}$ returned by $\mathtt{RevSrchFastExplore}$ (\alglineref{alg:GBRRT:ReverseExpansion}) is checked for collision using $\mathtt{CollisionCheck}$ (\alglineref{alg:GBRRT:collisioncheckForward}). If no collision exists, then $\mathcal{E}_{rev}$ and its corresponding final node ${x}_{rev}$ are added to $\mathbf{E}_{rev}$ (\alglineref{alg:GBRRT:updateErev})  and $\mathbf{V}_{rev}$ (\alglineref{alg:GBRRT:updateVrev}) respectively. \revision{$x_{rev}$ is utilized to perform a heuristic update in $\mathbf{Q}$ using $\mathtt{updatePriorityQueue}$ function (\alglineref{alg:GBRRT:updatePriorityQueue}).} \revision{We describe $\mathtt{updatePriorityQueue}$ (\algoref{alg:updatePriorityQueue}) as well as $\mathtt{insertToPriorityQueue}$ (\algoref{alg:insertToPriorityQueue}) later in this section.}

\subsubsection*{\textbf{Forward tree expansion}}
We expand the forward search tree in \revision{\alglinerefs{alg:GBRRT:updateExploitRatio}{alg:GBRRT:insertPriorityQueue}}. The forward search uses a combination of exploitation (\alglineref{alg:GBRRT:ForwardSearchExploit}) and exploration (\alglineref{alg:GBRRT:ForwardSearchFastExplore}, \alglineref{alg:GBRRT:ForwardSearchRandomExplore}) strategies to expand the tree. 
\revision{The exploitation strategy uses the reverse tree's cost information for focused expansion, while the exploration strategy uses both random exploration (for probabilistic completeness) and fast exploration (\figref{fig:GBRRTSearchTypes}). These strategies are discussed in detail in \secref{subSectionGBRRTSearchesOutline}.}

The choice of whether exploitation is performed in the current iteration is determined by $\mathcal{P}$ which outputs the probability $q$ based on $k$'s value (\alglineref{alg:GBRRT:updateExploitRatio}). $q$ is compared to a value $c_{rand}$ generated uniformly at random on ${[1, 0]}$  (\alglineref{alg:GBRRT:initp}). If ${c_{rand} < q}$, the edge is generated by an exploitation process by popping the forward tree node $x_{pop}$ from $\mathbf{Q}$ (\alglineref{alg:GBRRT:PopQ}) and then using $\mathtt{ForSrchExploit}$ (\alglineref{alg:GBRRT:ForwardSearchExploit}). This routine chooses the `best' reverse tree node and generates a trajectory towards $x_{pop}$ to make focused progress towards the goal. 

The condition of ${x_{pop} = \mathtt{NULL}}$ implies $\mathbf{Q}$ is empty. This may happen, e.g.,, near the beginning of the search when the forward tree has not yet encountered the reverse tree. If ${x_{pop} = \mathtt{NULL}}$, the algorithm defaults to performing fast exploration using $\mathtt{ForSrchFastExplore}$ (\alglineref{alg:GBRRT:ForwardSearchFastExplore}). This function is designed to \revisionTwo{quickly} explore the state space (see \figref{fig:GBRRTSearchTypes}). If $\revision{{c_{rand}}} \geq q$ or ${\mathcal{E}_{for}=\mathtt{NULL}}$ (where ${\mathcal{E}_{for} = \mathtt{NULL}}$ implies both exploitation and fast exploration were unsuccessful), the algorithm performs a random exploration using $\mathtt{ForSrchRandomExplore}$ (\alglineref{alg:GBRRT:ForwardSearchRandomExplore}). \revisionTwo{The use of $\mathtt{ForSrchRandomExplore}$ provides probabilistic completeness. Using only the exploitation and {\it fast} exploration may produce an algorithm that is not probabilistically complete. Probabilistic completeness is further discussed in Section~\ref{sec:analysis}.}
\par The edge $\mathcal{E}_{for}$ is checked for collision \revisionTwo{(\alglineref{alg:GBRRT:collisioncheckForward})} and if no collision exists, $\mathcal{E}_{for}$ and its corresponding final node ${x}_{for}$ are added to $\mathbf{E}_{for}$ (\alglineref{alg:GBRRT:updateEfor}) and $\mathbf{V}_{for}$ (\alglineref{alg:GBRRT:updateVfor}) respectively. If $x_{for} \in \mathcal{X}_{goal}$, checked using $\mathtt{goalRegionReached}$ (\alglineref{alg:GBRRT:goalRegionReached}), then the output path is returned using $\mathtt{Path}$ (\alglineref{alg:GBRRT:Path}).

\subsubsection*{\textbf{Priority queue $\mathbf{Q}$ \revision{insert}}}
\par GBRRT (\algoref{alg:GBRRT}, \alglineref{alg:GABRRT:insertToPriorityQueue}) adds the new forward tree node $x_{for}$ to $\mathbf{Q}$ (\figref{fig:GBRRPushBoth}) using \revision{$\mathtt{insertPriorityQueue}$ (\algoref{alg:insertToPriorityQueue})}. This algorithm first determines the closest node $x_{closest} \in \mathcal{G}_{rev}$ to $x_{for}$ within a ball of radius $r_{k}$ (\alglineref{alg:updatePriorityQueue:NearestVertices}). If $x_{closest} \neq \mathtt{NULL}$, we use the \revisionTwo{key value} $d_{X}(x_{for}, x_{closest}) +  h(x_{closest})$  to push $x_{for}$ to $\mathbf{Q}$ (\alglineref{alg:updatePriorityQueue:addToQ}). \revisionTwo{The cost includes the actual cost of reaching $x_{for}$ plus the heuristic estimate} of reaching the goal from $x_{for}$. The condition of $x_{closest} = \mathtt{NULL}$ is true if the forward tree has not encountered the reverse tree.
\par Note that $x_{for}$ is pushed to $\mathbf{Q}$ during the forward tree's exploration and exploitation step but only popped from $\mathbf{Q}$ during the exploitation step. This necessitates the use of a \revision{priority} queue $\mathbf{Q}$ \revision{in GBRRT} to store the set of potential forward tree nodes obtained from both these steps and pop the best one for the current exploitation step.\revision{\footnote{\color{black} A design choice worthy of mention is that we do not reinsert the node back into $\mathbf{Q}$ once it is popped (\alglineref{alg:GBRRT:PopQ} in \algoref{alg:GBRRT}). We found through experiments that although reinserting the node helps in some cases, there may be cases where it causes the performance to decrease. An example is reinserting a node that is close to the goal but blocked by an obstacle. There may be intelligent ways of re-pushing such a node when certain conditions are satisfied, rather than naively re-pushing them all the time. This could be a direction for future research.}}

\subsubsection*{\textbf{\revision{Priority queue $\mathbf{Q}$ update}}}
\revision{The heuristic update in $\mathbf{Q}$ during the reverse search (\figref{fig:GBRRPushBoth}) occurs in $\mathtt{updatePriorityQueue}$ (\algoref{alg:updatePriorityQueue}). The new reverse tree node $x_{rev}$ obtained from expanding the reverse tree is used to select the closest forward tree node $x_{closest}$ within a ball of radius $r_{k}$ (\alglineref{alg:updatePriorityQueue:NearestVertices}). If $x_{closest} \neq \mathtt{NULL}$, we update the heuristic of $x_{closest}$ \revisionTwo{and corresponding key value} $d_{X}(x_{closest}, x_{rev}) +  h(x_{rev})$ (\alglineref{alg:updatePriorityQueue:addToQ}) only if the \revisionTwo{new value} is less than \revisionTwo{the old key value} in $\mathbf{Q}$.}

\revisionTwo{While $\mathtt{insertPriorityQueue}$ adds a new (forward tree) node to $\mathbf{Q}$, $\mathtt{updatePriorityQueue}$ updates the key value and position of a (forward tree) node already in $\mathbf{Q}$.}

\subsubsection*{\revisionTwo{\textbf{Naive goal biasing vs. Cost-To-Goal heuristic}}}
\revisionTwo{Naive goal biasing \cite{lavalle2001randomized}, a common approach where the goal node is sampled a certain percentage of times to improve motion planners' performance, is not incorporated in GBRRT.} GBRRT uses the cost-to-goal heuristic from the reverse search tree nodes, which is significantly more informed than traditional naive goal biasing. \revisionTwo{However, a naive goal biasing} is used to potentially speed up the execution of some of the comparison algorithms specified in \secref{sec:experimental_setup} (Experiment setup).

\subsection{General outline of exploration and exploitation searches}
\label{subSectionGBRRTSearchesOutline}

GBRRT uses an exploration strategy for the reverse tree expansion and a blend of exploration and exploitation strategies for the forward tree expansion. \revisionTwo{In this section, we discuss the general requirements of these strategies (\figref{fig:GBRRTSearchTypes}), while specific details of the particular edge-generation implementations used in our experiments appear in \secref{sec:specificImplementation}}.

\subsubsection{\textbf{Exploitation search}}
Exploitation search uses the heuristic values from the reverse tree to bias the growth of the forward tree. It begins by popping the forward tree node $x_{pop}$ from $\mathbf{Q}$. Next, $x_{pop}$ is used to select the best node $x_{best}$ (within a radius $r_{k}$) in the reverse tree according to a user-defined cost function. Finally, it performs forward propagation from $x_{pop}$ (from $\mathbf{Q}$) such that the endpoint of the resulting trajectory/edge gets close to $x_{best}$ (\figref{fig:GBRRTSearchTypes}). This search is visualized in \revisionTwo{\figref{fig:GBRRTSearchTypes}-Left, where the probability density of the new edge's endpoint is depicted as being higher in the region between $x_{pop}$ and $x_{best}$ using a tear shape)}.
GBRRT implements this search in function $\mathtt{ForSrchExploit}$.

\subsubsection{\textbf{Fast exploration search}}
A fast exploration search iteration is attempted in the event that the exploitation search iteration fails (when $x_{pop} = \mathtt{NULL}$). It involves sampling a random node $x_{rand}$ and selecting the nearest neighbor $x_{near}$ in the forward tree. Forward propagation from $x_{near}$ is performed such that the endpoint of the resulting trajectory gets close to $x_{rand}$ (\figref{fig:GBRRTSearchTypes}). This \revisionTwo{has been shown} to explore the state space \revisionTwo{relatively quickly} \cite{kuffner2000rrt} because, \revisionTwo{given the Voronoi partitioning generated by forward tree nodes, search is biased into the relatively large Voronoi regions associated with the search frontier}. \revisionTwo{\figref{fig:GBRRTSearchTypes}-Center contains a visualization in which a probability density of the new edge's endpoint being higher in the region between $x_{near}$ and $x_{rand}$ is depicted using a tear shape.} GBRRT implements this search in functions $\mathtt{ForSrchFastExplore}$ and $\mathtt{RevSrchFastExplore}$. \revisionTwo{In \secref{sec:specificImplementation}, we discuss a method of performing fast exploration that worked well in our experiments; faster methods may exist depending on problem scenario and system dynamics.}

\begingroup
\removelatexerror
\begin{figure*}[t!]
\scalebox{\algorithmScale}{
\begin{minipage}[c]{8.3cm}
\begin{algorithm}[H]
  \caption{$\mathtt{ForSrchExploit}(\mathcal{G}_{rev}, \mathcal{U}, T_{max}, x_{pop}, \revision{r_{k}}$)} \label{alg:ForSrchExploit}

\tcp{Heuristic edge-generation} 
$\mathbf{V}_{rev\_near} \gets \{\, x\in \mathcal{G}_{rev}.\mathtt{Nodes()} \mid d_{X}( x_{pop}, x) \leq \revision{r_{k}} \} \!\!\!\!\!\!\!\!$ \\ \label{alg:ForSrchExploit: NearestVertices}

$\displaystyle{x_{best} \gets \argmin_{x\in \mathbf{V}_{rev\_near}} \left({g(x_{pop}) + d_{X}(x_{pop}, x) + h(x)}\right)}$ \\ \label{alg:ForSrchExploit: x_near}

$\mathcal{E}_{new} \gets \mathtt{BestInputProp} \paren{x_{pop}, x_{best},  \mathcal{U}, T_{max}}$ \\ \label{alg:ForSrchExploit: bestInputProp}

\Return {$\mathcal{E}_{new}$} \\ \label{alg:ForSrchExploit: returnEdge}

\end{algorithm}
\begin{algorithm}[H]
  \caption{$\mathtt{ForSrchFastExplore}(\revision{\mathcal{G}_{for}}, \mathcal{U}, T_{max})$} \label{alg:ForSrchFastExplore}
 
    \tcp{Best-input edge-generation}
    $x_{rand} \gets \mathtt{RandomState} \paren{}$ \\ \label{alg:ForSrchFastExplore:RandomState}
    $x_{near} \gets \mathtt{NearestNeighbor} \paren{x_{rand}, \revision{\mathcal{G}_{for}}}$ \\ \label{alg:ForSrchFastExplore:NearestNeighbor}
    $\mathcal{E}_{new} \gets \mathtt{BestInputProp} \paren{x_{near}, x_{rand}, \mathcal{U}, T_{max}}$ \\ \label{alg:ForSrchFastExplore:BestInputProp}
    \Return {$\mathcal{E}_{new}$}
\end{algorithm}
\end{minipage}}
\scalebox{\algorithmScale}{
\centering
\begin{minipage}[c]{8.5cm}
\begin{algorithm}[H]
  \tcp{Best-input prop. for frwd. search}
  \caption{$\mathtt{BestInputProp}(x_{near}, x_{rand}, \mathcal{U}, T_{max})$} \label{alg:BestInputProp}
  
    $\mathbf{E}_{fin} \gets \{\} $ \\
    \For{$k \gets 1$ to $N_{B}$}
    {
        $ \mathbf{E}_{fin} \gets \mathbf{E}_{fin} \cup \{\mathtt{MonteCarloProp}(x_{near}, \mathcal{U}, T_{max})\} \!\!\!\!\!\!\!\!\!\!\!\!\!\!\!\!\!\!$ \\ \label{alg:BestInputProp:MonteCarloProp}
    }
    
    $\displaystyle{\mathcal{E}_{new} \gets \argmin_{\mathcal{E}\in \mathbf{E}_{fin}}   \left(d_{X}(\mathcal{E}.finalNode(),\,\, x_{rand})\right)}$ \\ \label{alg:BestInputProp:bestNew}
  
    \Return {$\mathcal{E}_{new}$}
\end{algorithm}
\begin{algorithm}[H]
  \tcp{Random edge-generation}
  \caption{$\mathtt{ForSrchRandomExplore}(\revision{\mathcal{G}_{for}}, \mathcal{U}, T_{max})$} \label{alg:ForSrchRandomExplore}
    $x_{rand} \gets \mathtt{RandomState} \paren{}$ \\ \label{alg:ForSrchRandomExplore:RandomState}
    $x_{near} \gets \mathtt{NearestNeighbor} \paren{x_{rand}, \revision{\mathcal{G}_{for}}}$ \\ \label{alg:ForSrchRandomExplore:NearestNeighbor}
    $\mathcal{E}_{new} \gets \mathtt{MonteCarloProp} \paren{x_{near}, \mathcal{U}, T_{max}}$ \\ \label{alg:ForSrchRandomExplore:MonteCarloProp}
    \Return {$\mathcal{E}_{new}$}
\end{algorithm}
\end{minipage}}
\scalebox{0.7}{
\begin{minipage}{8.0cm}
\revisionAlgorithm
\begin{algorithm}[H]
  \tcp{Monte Carlo propagation for forward search}
  \caption{$\mathtt{MonteCarloProp}(x_{initial}, \mathcal{U}, T_{max})$} \label{alg:MonteCarloProp}
    $t_{\mathcal{E}} \gets \mathtt{Random}(0, T_{max})$ \\
     $u \gets \mathtt{Random}(\mathcal{U})$\\
    $\mathcal{E}_{new} = \int_{0}^{t_{{\mathcal{E}}}} f(x, u) \,dt$ where $\mathcal{E}_{new}(0)=x_{initial}$\\
    \Return {$\mathcal{E}_{new}$}
\end{algorithm}
\begin{algorithm}[H]
  \caption{$\mathtt{RevSrchFastExplore}(\mathcal{G}_{rev}, \mathcal{U}, T_{max})$} \label{alg:RevSrchFastExplore}
 
    \tcp{Best-input edge-gen. for reverse search}
    $x_{rand} \gets \mathtt{RandomState} \paren{}$ \\ \label{alg:RevSrchFastExplore:RandomState}
    $x_{near} \gets \mathtt{NearestNeighbor} \paren{x_{rand}, \mathcal{G}_{rev}}$ \\ \label{alg:RevSrchFastExplore:NearestNeighbor}
    $\mathcal{E}_{new} \gets \mathtt{BestInputProp} \paren{x_{near}, x_{rand}, \mathcal{U}, T_{max}}\!\!\!\!\!\!\!$ \\ \label{alg:RevSrchFastExplore:BestInputProp}
    \Return {$\mathcal{E}_{new}$}
\end{algorithm}
\end{minipage}}
\vspace{-5mm}
\end{figure*}
\endgroup

\subsubsection{\textbf{Search with probabilistic completeness property}}
This search provides probabilistic completeness for GBRRT, and also helps with exploration. It involves randomly sampling a node $x_{rand}$ and then selecting the nearest neighbor $x_{near}$ in the forward tree using $d_{X}$. Next, forward propagation from $x_{near}$ is performed using a randomly selected constant control (randomness helps to ensure probabilistic completeness). \revisionTwo{This search is not biased toward unexplored regions, and so the resulting trajectory may propagate in any direction. This exploration method is visualized in \figref{fig:GBRRTSearchTypes}-Right, where the probability density of endpoint of the new edge is depicted as being spread out around $x_{near}$.} GBRRT implements this search in function $\mathtt{ForSrchRandomExplore}$.
\par \revisionTwo{The \figref{fig:GBRRTSearchTypes} depictions of expected probability density being `tear-shaped' and `spread out' are used here for illustrative purposes and to help build intuition, they are not explicitly calculated as part of the algorithm. The actual shape of these regions will depend on many factors including system dynamics, control space, max propagation time, etc.}

\subsection{Generalized Asymmetric Bidirectional RRT (GABRRT)}
\label{subSection:GABRRT}
We present a variant of GBRRT called GABRRT (\algoref{alg:GABRRT}), where non-dynamical trajectories are utilized to build the reverse tree \revisionTwo{rather than reverse trajectories based on the system's dynamics}. GABRRT generally performs well when no dynamics  are considered in the planning problem \revisionTwo{(or dynamics create `straight-line' like trajectories)}. The non-dynamical distance function $d^{\mathtt{ND}}_{X}$ is applied when updating $\mathbf{Q}$ (\algrefsTwo{alg:insertToPQueueND}{alg:UpdatePQueueND}), building the reverse tree (\algoref{alg:RevSrchND}) and generating the heuristic values from the reverse tree (\algrefsTwo{alg:ForSrchExploitND}{alg:BestInputPropND}). The advantage of using GABRRT is that the reverse tree is computationally inexpensive to build because it does not require performing online integration. On the other hand, omitting the effects of higher-order dynamics in the guiding heuristic may lead to poor performance in scenarios where the dynamical constraints play a relatively important \revisionTwo{role} versus the geometric constraints. The algorithm for GABRRT is presented in detail in Appendix\apdxref{apdx:GABRRTAlgorithms}.

\section{GBRRT SPECIFIC IMPLEMENTATION \revision{OF SUBROUTINES USED IN OUR EXPERIMENTS}}
\label{sec:specificImplementation}
In \subSecref{subSection:PropagationTypes}, we describe how the search subroutines outlined in \subSecref{subSectionGBRRTSearchesOutline} have been implemented in our experiments. In \subSecref{sub:SelectionParameters}, we explain the parameters of our \revision{specific implementations} of GBRRT and GABRRT.

\subsection{Specific edge-generation types}
\label{subSection:PropagationTypes}
Here we provide the details of the edge-generation methods for forward search \revision{(\figref{fig:GBRRTSpecificPropagation})} used in our experiments (\secref{sec:experimental_setup}). It is important to note that, \revisionTwo{while these particular implementations have the edge-generation properties required by GBRRT,  many alternative strategies may  exist that also have the required properties. We use these versions in our experiments because they are easy to implement and straightforward to analyze.}

\subsubsection{\textbf{Heuristic edge-generation}}
\par \revision{This edge-generation method (\algoref{alg:ForSrchExploit}) exploits the heuristic provided by the reverse tree. The algorithm determines the nearest vertices set $\mathbf{V}_{rev\_near}$ to $x_{pop}$ in $\mathcal{G}_{rev}$ (\alglineref{alg:ForSrchExploit: NearestVertices}) within a ball of radius $r_{k}$. It then selects $x_{best} \in \mathbf{V}_{rev\_near}$ (\alglineref{alg:ForSrchExploit: x_near}) having the minimum cost $g(x_{pop}) + d_{X}(x_{pop}, x) + h(x)$ \revision{(\figref{fig:HeuristicOperation})}. Best-input propagation (\algoref{alg:BestInputProp}), explained in the next subsection, generates a new potential edge $\mathcal{E}_{new}$ (\alglineref{alg:ForSrchExploit: bestInputProp}) that starts at $x_{pop}$ and gets close to $x_{best}$ (\figref{fig:GBRRTSpecificPropagation}-Left).
}

\subsubsection{\textbf{Best-Input edge-generation}}
This edge-generation method (\algoref{alg:ForSrchFastExplore}) performs a fast exploration search (\figref{fig:GBRRTSpecificPropagation}). It samples a random node $x_{rand}$ in the state space using $\mathtt{RandomState}$ (\alglineref{alg:ForSrchFastExplore:RandomState}), then selecting its nearest neighbor in $\mathcal{G}_{for}$ according to the distance function $d_{X}$ using $\mathtt{NearestNeighbor}$ (\alglineref{alg:ForSrchFastExplore:NearestNeighbor}), and finally performs best-input propagation (\cite{kleinbort2018probabilistic}) using $\mathtt{BestInputProp}$ (\alglineref{alg:ForSrchFastExplore:BestInputProp}). Best-Input propagation (\algoref{alg:BestInputProp}) entails generating a list of $N_{B}$ Monte Carlo propagation trajectories (explained in next subsection) and then selects the trajectory that gets closest to the random sampled state $x_{rand}$ as specified by the distance function $d_{X}$ \revision{(\figref{fig:GBRRTSpecificPropagation}-Center)}.
\revision{Although this algorithm performs fast exploration, in practice, its probabilistic completeness is still an open question  \cite{kleinbort2018probabilistic}\cite{kunz2015kinodynamic}}. \revisionTwo{The reverse search version -- $\mathtt{RevSrchFastExplore}$, is quite similar, and is presented in \algoref{alg:RevSrchFastExplore} and its sub-functions in  Appendix\apdxref{apdx:ReverseSearchAlgorithms}.}

\begin{figure*}[ht]
\begin{minipage}[c]{0.33\textwidth}
\figuretitle{Pop operation from $\mathbf{Q}$ for heuristic edge-generation}
\includegraphics[scale=0.35]{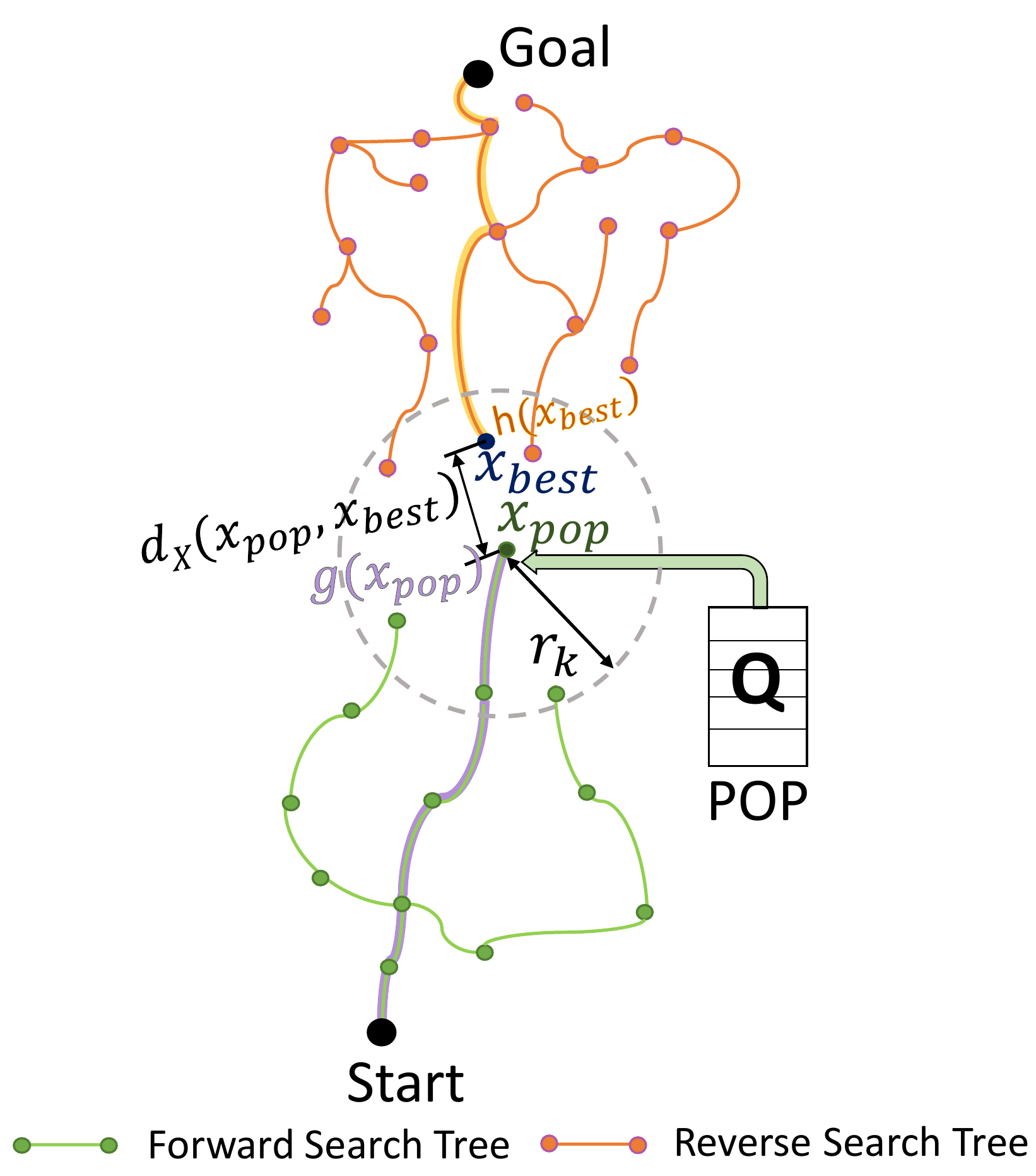}
\caption{A forward tree node $x_{pop}$ popped from $\mathbf{Q}$ and $x_{best}$ chosen from reverse tree using the cost $g(x_{pop}) + d_{X}(x_{pop}, x) + h(x)$ for heuristic edge generation. $x_{best}$ is within a ball of radius \revision{$r_{k}$} centered at $x_{pop}$.}
\label{fig:HeuristicOperation}
\vspace{8mm}
\figuretitle{Covering ball sequence}
\includegraphics[scale=0.43]{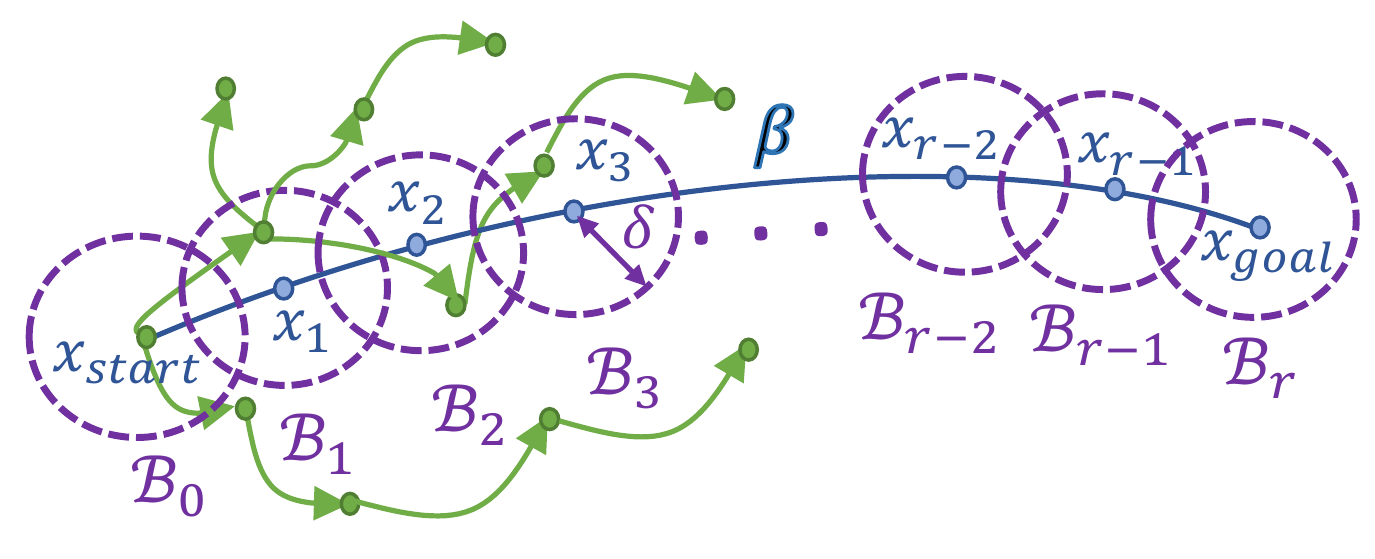}
\caption{Covering ball sequence on reference path $\beta$. Only forward search tree is shown. Figure inspired from Fig. 1 in \cite{kleinbort2020refined}.}
\label{fig:ProofDiagram}
\end{minipage}
\hspace{.05cm}
\begin{minipage}[c]{0.65\textwidth}
\figuretitle{\revision{Initial solution time vs. parameters ($N_{B}$ and $q$) with $\delta_{hr}$ fixed}}
\captionsetup{font={color=black, small}}
\begin{minipage}[c]{0.3\textwidth}
\begin{xy}
\xyimport(100, 100){\includegraphics[scale=\parameterScale]{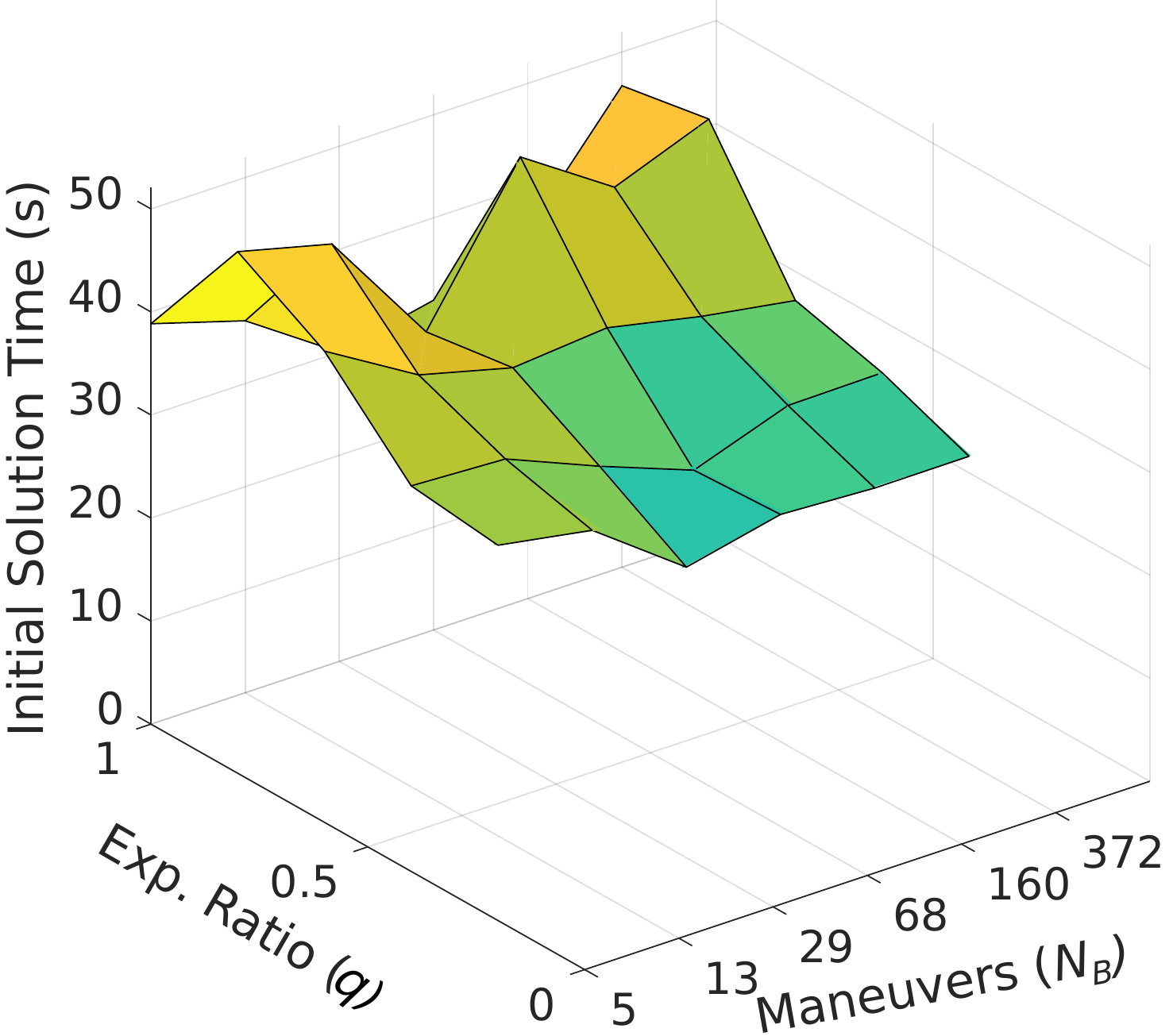}}
,(55, 100)*{\revision{\footnotesize \text{$\delta_{hr} = 4$ (Low)}}}
\end{xy}
\end{minipage}
\begin{minipage}[c]{0.3\textwidth}
\begin{xy}
\xyimport(100, 100){\includegraphics[scale=\parameterScale]{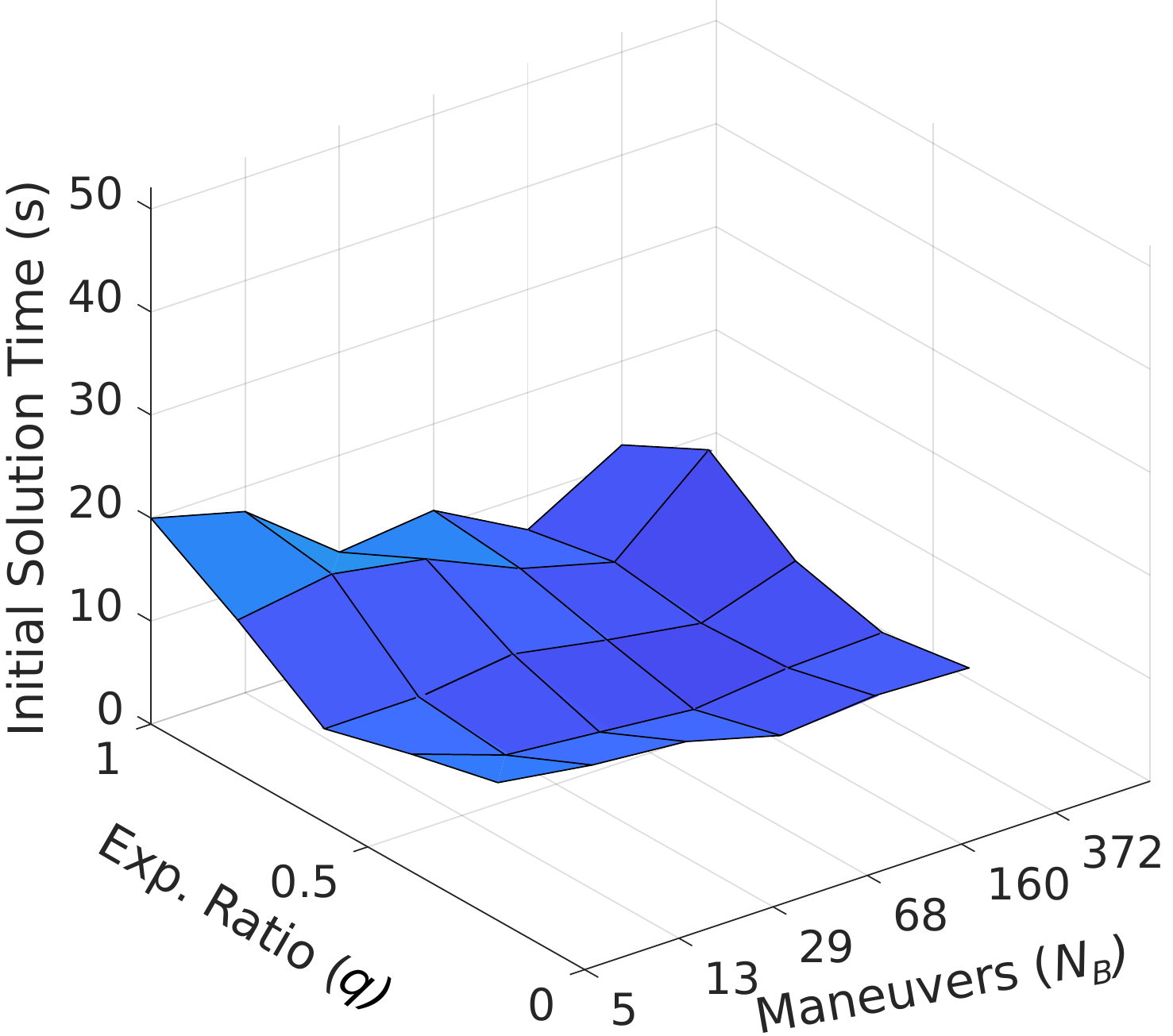}}
,(55, 100)*{\revision{\footnotesize \text{$\delta_{hr} = 16$ (Medium)}}}
\end{xy}
\end{minipage}
\begin{minipage}[c]{0.3\textwidth}
\begin{xy}
\xyimport(100, 100){\includegraphics[scale=\parameterScale]{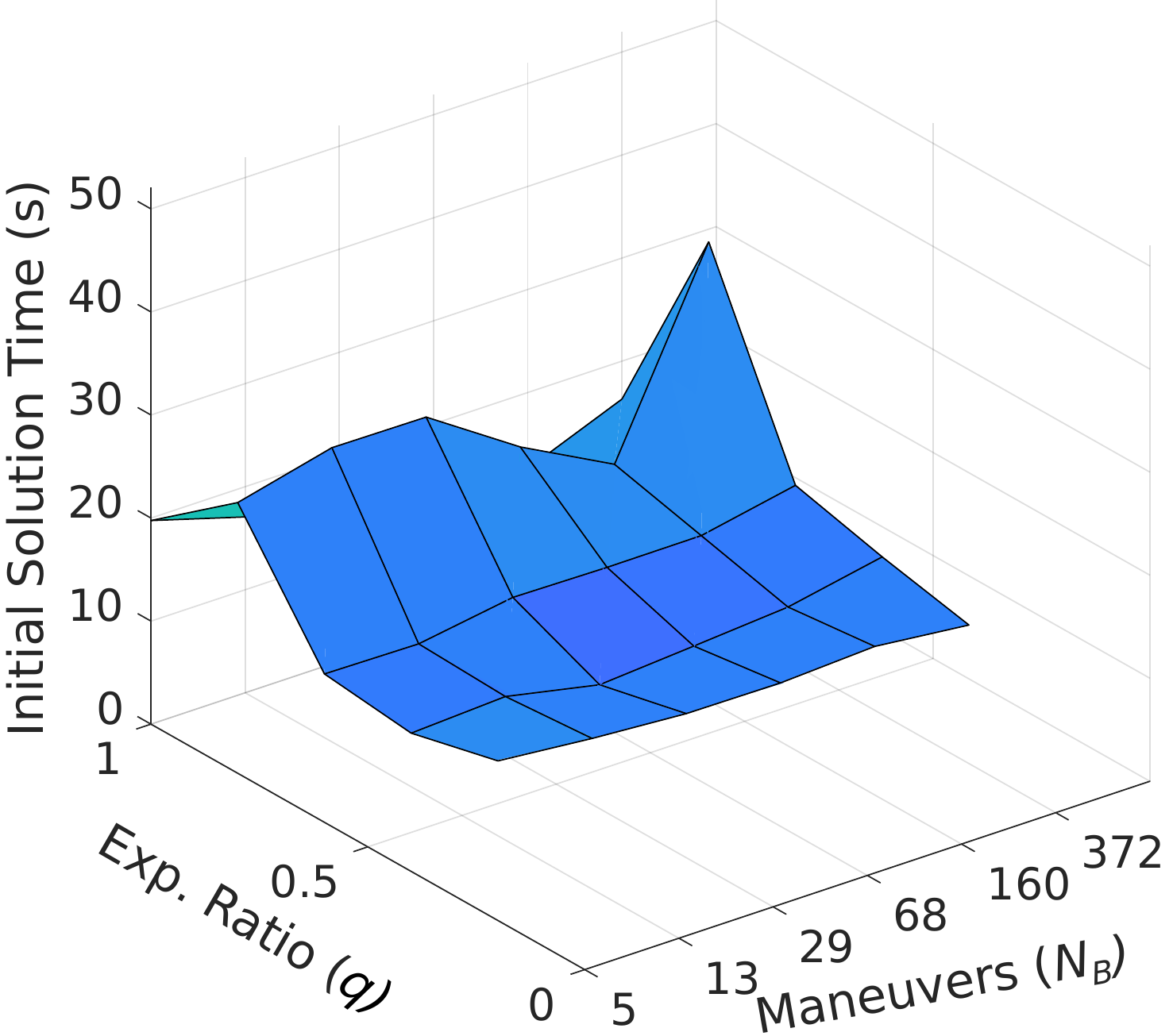}}
,(52, 100)*{\revision{\footnotesize \text{$\delta_{hr} = 64$ (High)}}}
\end{xy}
\end{minipage}
\label{fig:deltaHrFixed}
\vspace{5mm}
\figuretitle{\revision{Initial solution time vs. parameters ($\delta_{hr}$ and $N_{B}$) with $q$ fixed}}
\begin{minipage}[c]{0.3\textwidth}
\begin{xy}
\xyimport(100, 100){\includegraphics[scale=\parameterScale]{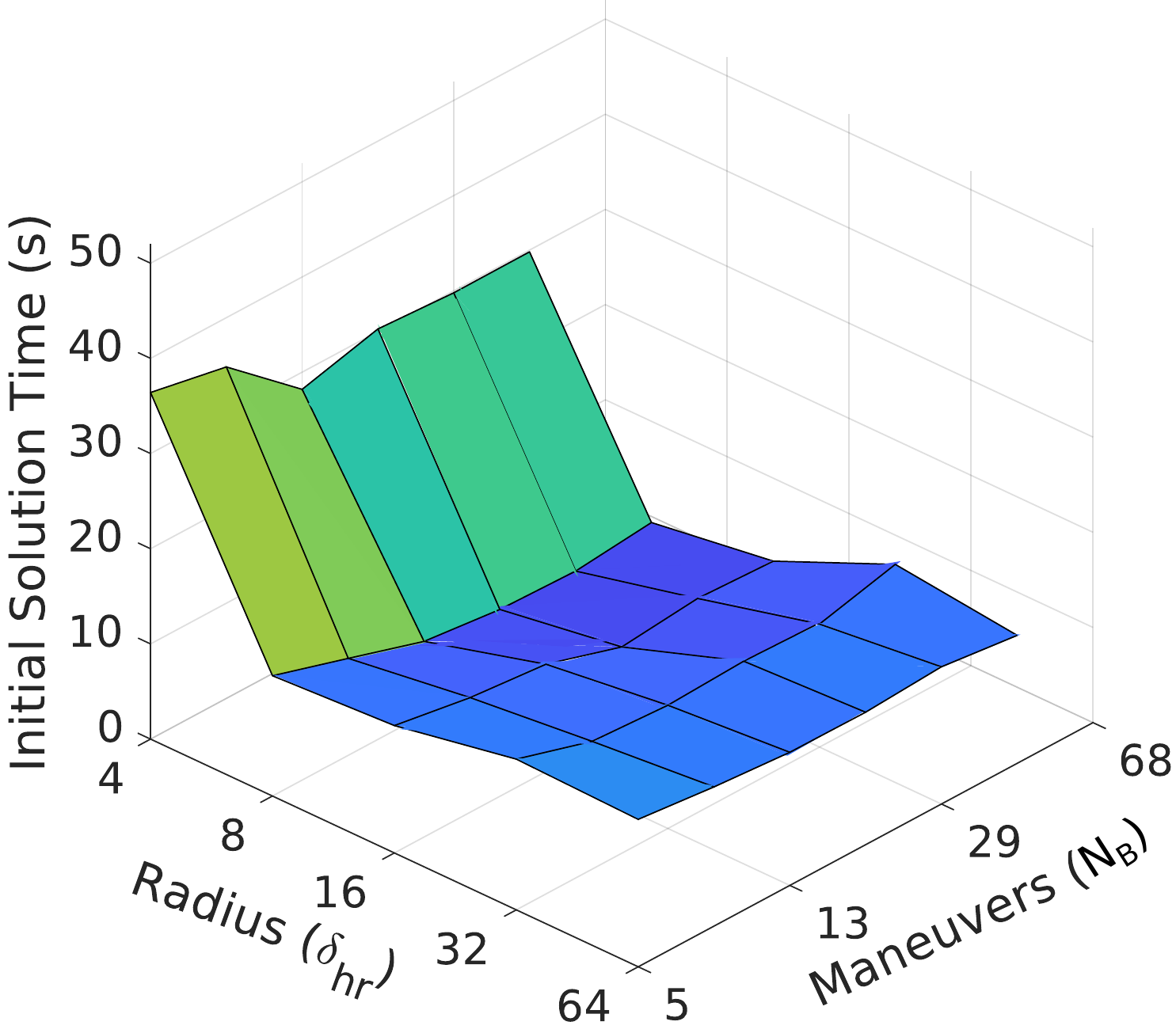}}
,(55, 100)*{\revision{\footnotesize \text{$q = 0.2$ (Low)}}}
\end{xy}
\end{minipage}
\begin{minipage}[c]{0.3\textwidth}
\begin{xy}
\xyimport(100, 100){\includegraphics[scale=\parameterScale]{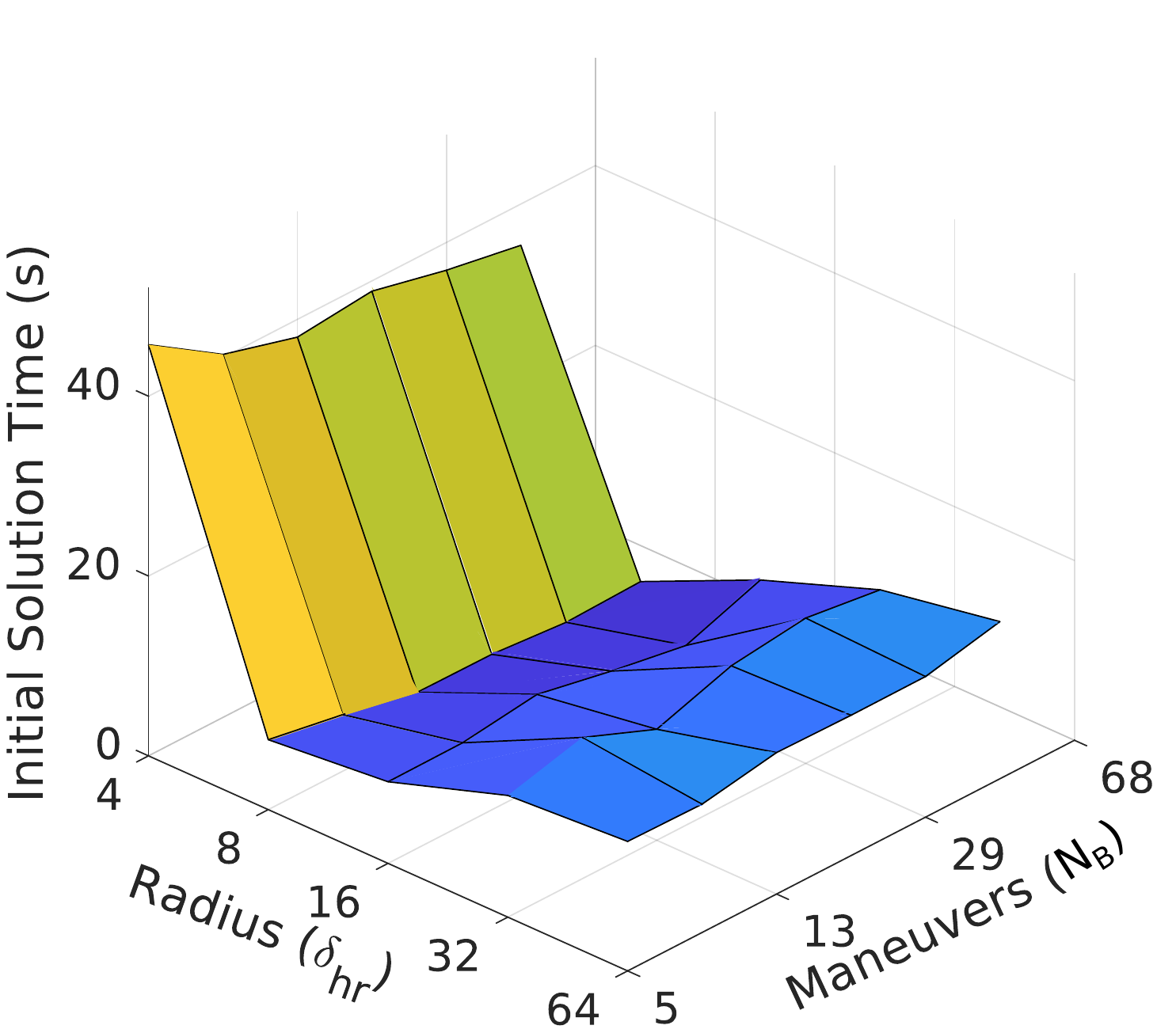}}
,(55, 100)*{\revision{\footnotesize \text{$q = 0.6$ (Medium)}}}
\end{xy}
\end{minipage}
\begin{minipage}[c]{0.3\textwidth}
\begin{xy}
\xyimport(100, 100){\includegraphics[scale=\parameterScale]{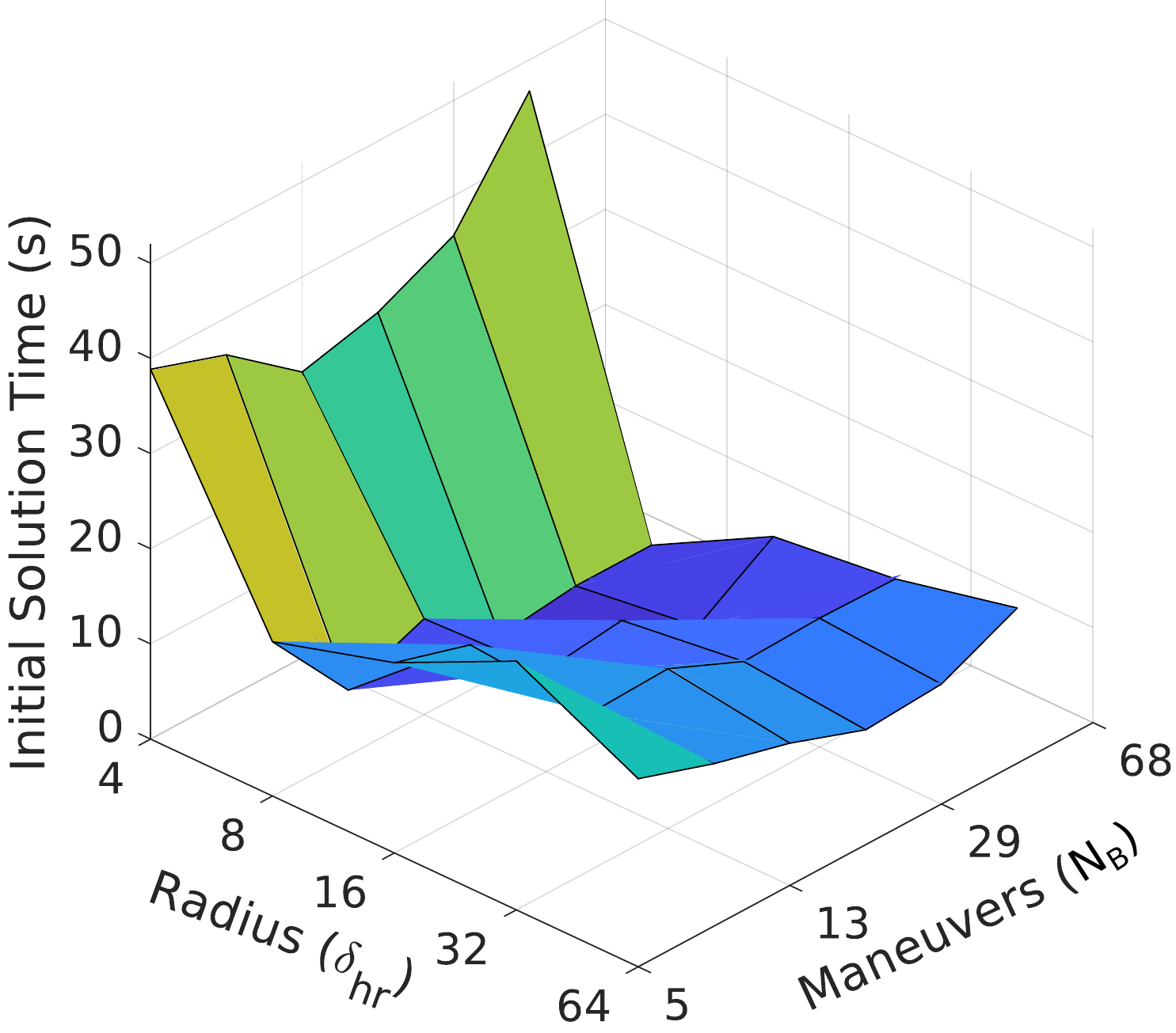}}
,(52, 100)*{\revision{\footnotesize \text{$q = 1.0$ (High)}}}
\end{xy}
\end{minipage}
\label{fig:exploitationRatioFixed}
\vspace{5mm}
\figuretitle{\revision{Initial solution time vs. parameters ($\delta_{hr}$ and $q$) with $N_{B}$ fixed}}
\begin{minipage}[c]{0.1\textwidth}
\begin{xy}
\xyimport(100, 100){\includegraphics[scale=\parameterScale]{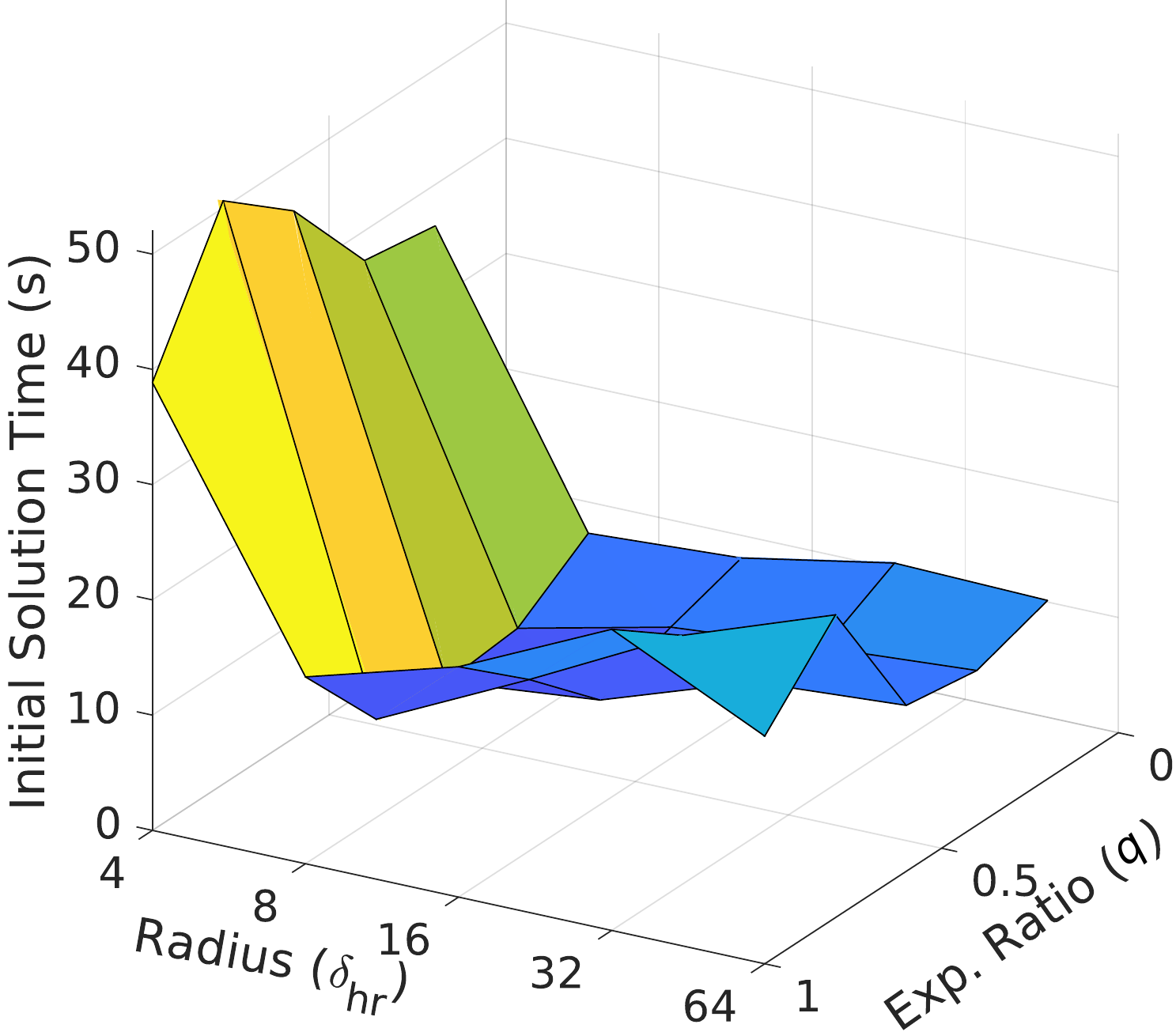}}
,(55, 100)*{\revision{\footnotesize \text{$N_{B} = 5$ (Low)}}}
\end{xy}
\end{minipage}
\begin{minipage}[c]{0.1\textwidth}
\begin{xy}
\xyimport(100, 100){\includegraphics[scale=\parameterScale]{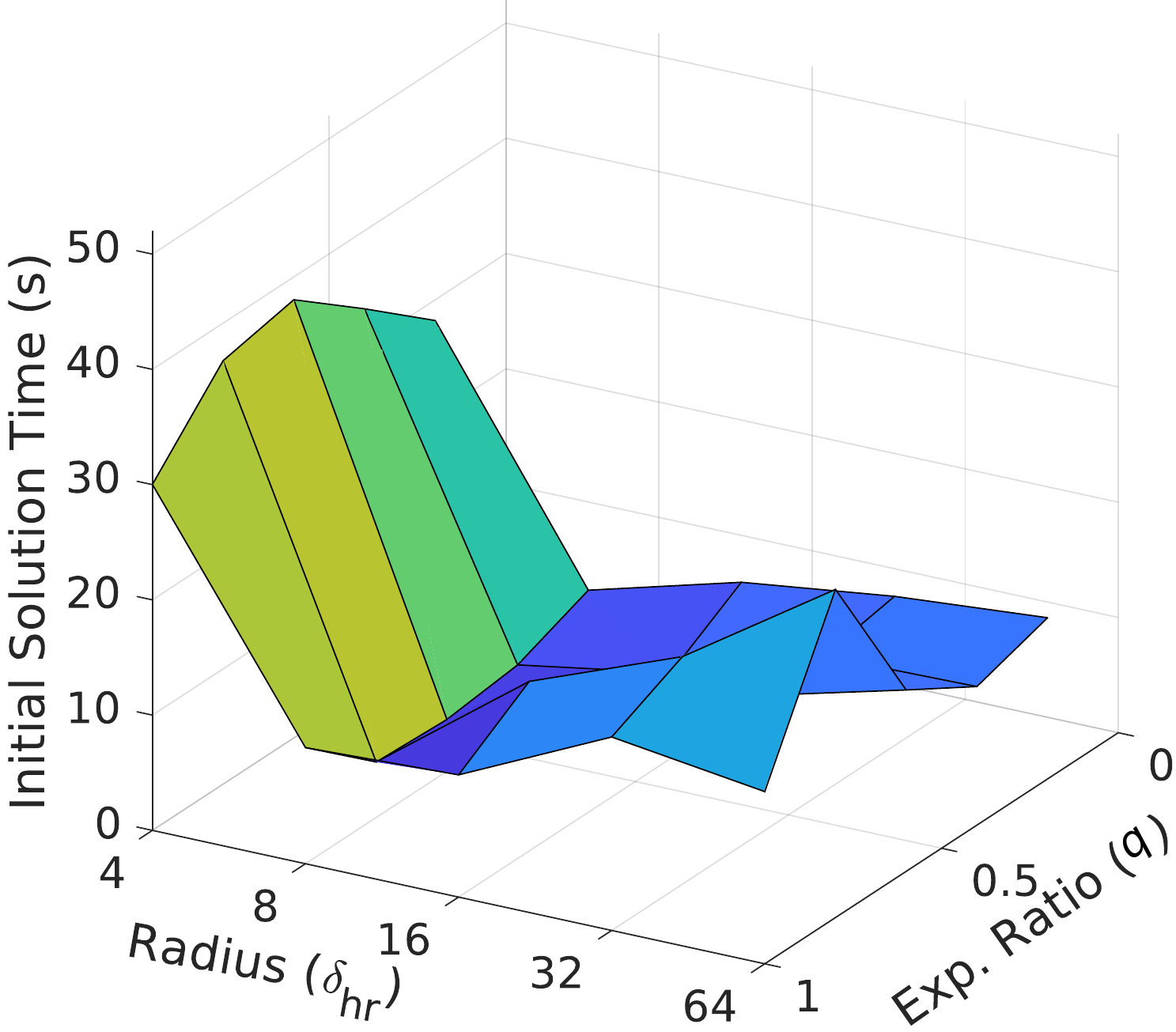}}
,(55, 100)*{\revision{\footnotesize \text{$N_{B} = 29$ (Medium)}}}
\end{xy}
\end{minipage}
\begin{minipage}[c]{0.1\textwidth}
\begin{xy}
\xyimport(100, 100){\includegraphics[scale=\parameterScale]{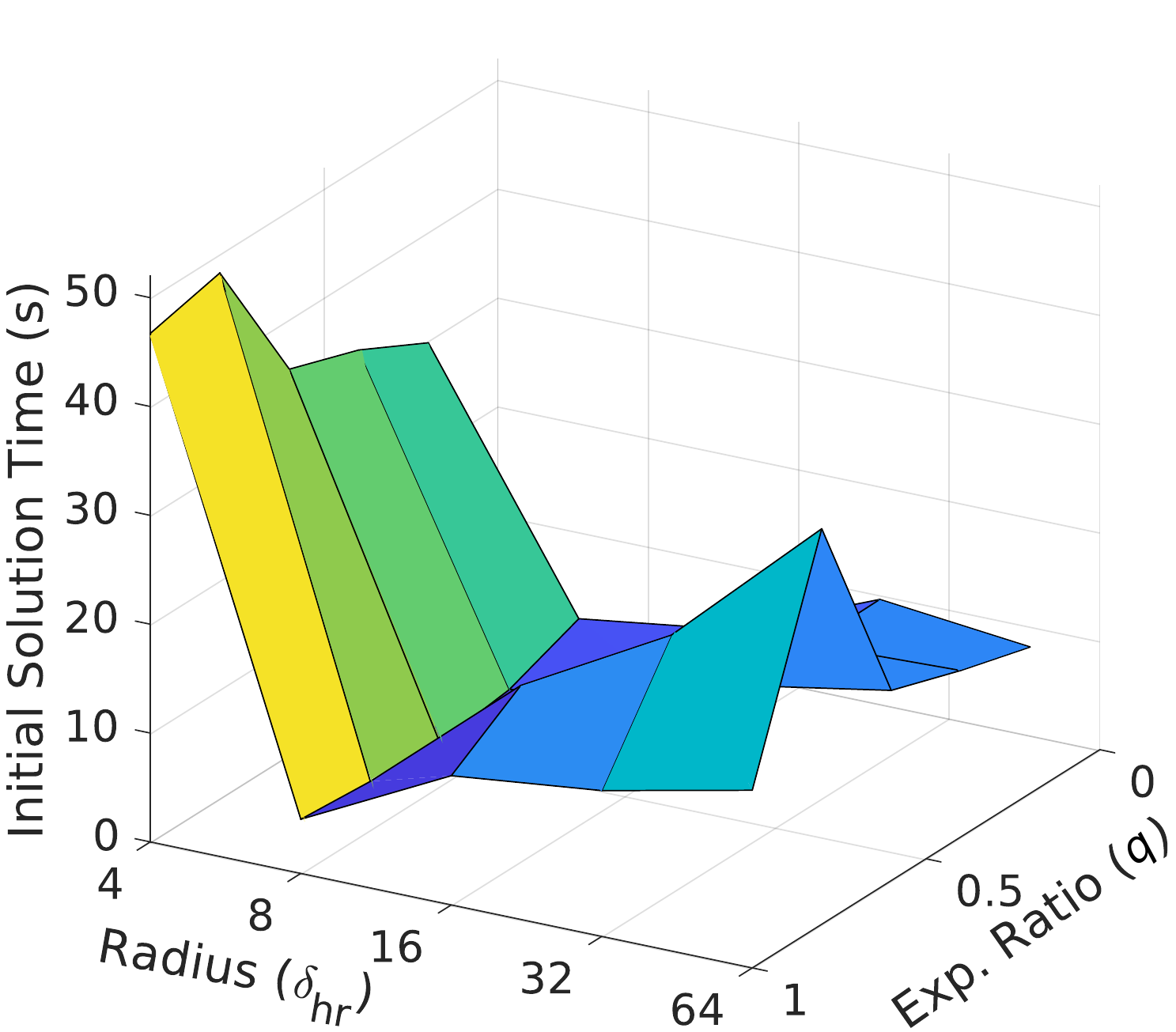}}
,(52, 100)*{\revision{\footnotesize \text{$N_{B} = 372$ (High)}}}
\end{xy}
\vspace{5mm}
\end{minipage}
\begin{minipage}[c]{0.65\textwidth}
\hspace{30mm}
\begin{xy}
\xyimport(100, 100){\includegraphics[scale=0.4]{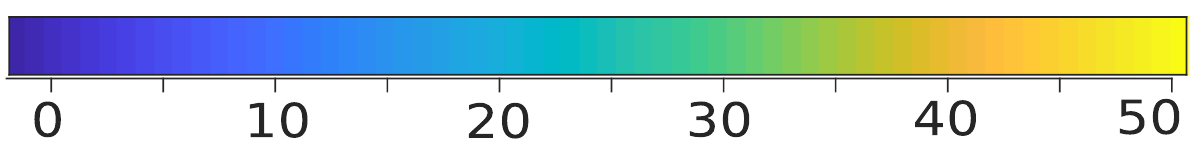}}
,(52, 100)*{}
\end{xy}
\end{minipage}
\caption{Variation of initial solution time vs. parameters ($\delta_{hr}$, $q$, $N_{B}$) of GABRRT. $\delta_{hr}$ and $N_{B}$ axis are log scales while $q$ axis is linear scale. We generate the figures using unicycle dynamics with the environment shown in \figref{fig:SoftwareSimulationSnapshot}. Each point is a mean of 70 trials. \revisionTwo{We set $r_{k} = \delta_{hr}$ (constant radius) in all trials to highlight effect of $\delta_{hr}$ on initial solution time.}}
\label{fig:parametersFixed}
\end{minipage}
\end{figure*}

\subsubsection{\textbf{Random edge-generation}}
This edge-generation method (\algoref{alg:ForSrchRandomExplore}) is required to perform search in a way that facilitates probabilistic completeness (\figref{fig:GBRRTSpecificPropagation}).
This method involves sampling a random node $x_{rand}$ in the state space, selecting its nearest neighbor in $\mathcal{G}_{for}$  (\alglineref{alg:ForSrchRandomExplore:NearestNeighbor}) and then performing Monte Carlo propagation \cite{kleinbort2018probabilistic} using $\mathtt{MonteCarloProp}$ (\alglineref{alg:ForSrchRandomExplore:MonteCarloProp}). Monte Carlo propagation  (\algoref{alg:MonteCarloProp}) requires randomly selecting \revisionTwo{an} input $u \in \mathcal{U}$ of random duration $t$ and then performing integration of system (1) to generate a new trajectory. We note that using this type of edge-generation within the Kinodynamic RRT algorithm as been proven to be probabilistic complete  \cite{kleinbort2018probabilistic}.

\subsection{Parameters of GBRRT/\revision{GABRRT}}
\label{sub:SelectionParameters}
\revision{Two user-chosen parameters affect the performance of GBRRT and GABRRT: heuristic radius ($\delta_{hr}$) and exploitation ratio function ($\mathcal{P}$). A third parameter is required for the particular implementation of $\mathtt{BestInputProp}$ that we use in our experiments: the number of best-input propagation trajectories ($N_{B}$). We now outline the \revisionTwo{role} of these parameters  and then perform a simulation study to evaluate the sensitivity of the performance of the algorithm to these (changing) parameters.} 
\par The heuristic radius $\delta_{hr}$ performs two roles. First, it defines the maximum distance to the closest reverse\revision{\slash forward} tree node that affects the cost value used to insert\revision{\slash update} a forward tree node to the priority queue $\mathbf{Q}$. 
Second, it limits the list of the potential reverse tree nodes considered for heuristic edge-generation (\alglineref{alg:ForSrchExploit: NearestVertices} in \algoref{alg:ForSrchExploit}). 
Choosing a low value of $\delta_{hr}$ will both decrease the likelihood that a forward tree node is added to the priority queue and decrease the number of reverse nodes considered for directed propagation. On the other hand, a high value of $\delta_{hr}$ will reduce the accuracy of the cost-to-goal heuristic and increase the chances that directed propagation is blocked by an obstacle collision.

\par The user-defined exploitation ratio $\mathcal{P}$ dictates the trade-off between exploration and exploitation during forward search. We choose $\mathcal{P}$ to be a constant probability $q$ in our experiments. We pick ${q > 0.5}$ to prioritize exploitation over exploration to quickly get to an initial feasible solution. \revision{However, a very high $q$ can cause insufficient exploration, especially in higher-dimensional systems leading to poor performance.}

\par The number of trajectories $N_{B}$ of best-input propagation affects the per iteration run-time of best-input and heuristic edge-generation. As the value of $N_{B}$ increases, the closer the set approximates the true reachability set of that node \cite{littlefield2018efficient}, at the expense of potentially increasing both the per-iteration run-time and the time to generate a feasible solution. A low value of $N_{B}$ decreases the per-iteration run-time of the algorithm but potentially increases the time to generate a feasible solution.
The effect of increasing $N_{B}$ \revision{has less impact} when using pre-computed maneuver libraries than using online integration.

\revision{To understand the effect of input parameters ($\delta_{hr}$, $q$, $N_{B}$) on the initial solution time, we conduct a simulation study (\figref{fig:parametersFixed}) using unicycle dynamics in the environment specified in \figref{fig:SoftwareResultGraphs}. Because we have three input parameters and one output parameter (initial solution time), we fix one input parameter while plotting the other two parameters on the x-y axis and the initial solution time on the z-axis. \revisionTwo{Due to space constraints, we choose Low, Medium, and High values to plot for each fixed input parameter (see \figref{fig:parametersFixed}).}}
\par \revision{The effect of $\delta_{hr}$ on the initial solution time is shown in \revision{\figref{fig:parametersFixed}}-Top Row. The initial solution time is higher for low (left) and high (right) values of $\delta_{hr}$ and best around the medium value (middle). The initial solution time did not vary much as $\delta_{hr}$ changed from the medium to high value.}
\par \revision{The effect of $q$ on the initial solution time is shown in \figref{fig:parametersFixed}-Middle Row. The initial solution time is lower for the medium value of $q$ (middle) when $\delta_{hr}$ is in the medium range 8-14. However, low $q$ does better when $\delta_{hr}$ is in the low range (2-4). This happens because when $\delta_{hr}$ is low, the exploitation search is not that effective, and hence experiments using a low value of $q$ (more exploration) produces faster initial solutions.}
\par \revision{The effect of $N_{B}$ on the initial solution time is shown in \figref{fig:parametersFixed}-Bottom Row. The initial solution time is low and relatively constant across the tested range of $N_{B}$ for medium values of $\delta_{hr}$ around 8-14. However, the initial solution time is higher for low (\textit{and high}) $N_{B}$ values with low $\delta_{hr}$ and high $q$ because not enough (\textit{in excess}) trajectories are considered for performing a productive exploitation search.}

\section{ANALYSIS}
\label{sec:analysis}
In this section, we prove the probabilistic completeness of GBRRT and GABRRT.
\revision{Computational complexity} is discussed in Appendix~\ref{sub:RuntimeComplexity}.
We begin by restating the definition of probabilistic completeness referenced from \cite{karaman2011sampling}. \revisionTwo{Next, we prove that GBRRT is probabilistic complete and the same for GABRRT.}

\begin{definition}
Let $\mathcal{A}$ be an algorithm that solves the single-query feasible motion planning problem ($x_{start}$, $\mathcal{X}_{free}$, $\mathcal{X}_{goal}$). Let $\{\mathbf{V}^{\mathit{\mathcal{A}}}_{n}\}_{n \in \mathbb{N}}$ and $\{\mathbf{E}^{\mathit{\mathcal{A}}}_{n}\}_{n \in \mathbb{N}}$ be the sequences of vertices and edge sets returned by algorithm $\mathcal{A}$ at each iteration $n$. Let the corresponding graph at iteration $n$ be \mbox{$\mathcal{G}^{\mathit{\mathcal{A}}}_{n}$ = ($\mathbf{V}^{\mathit{\mathcal{A}}}_{n}$, $\mathbf{E}^{\mathit{\mathcal{A}}}_{n}$)}. The algorithm $\mathcal{A}$ is probabilistic complete if

\begin{equation*}\label{ProbCompleteness}
\begin{aligned}
\liminf\limits_{n\rightarrow \infty} \mathbb{P}(\exists\, x_{g} \in \mathcal{X}_{goal} \cap \mathbf{V}^{\mathit{\mathcal{A}}}_{n}  \,\, \textit{such that} \\ \,\, x_{start} \,\, \textit{is connected to} \,\, x_{g} \,\, \textit{in} \,\, \mathcal{G}^{\mathit{\mathcal{A}}}_{n}) = 1.
\end{aligned}
\end{equation*}
\end{definition}

We leverage the proof of probabilistic completeness of RRT with random edge-generation provided by Kleinbort et al.\cite{kleinbort2018probabilistic} to prove the probabilistic completeness of GBRRT. For this proof, we assume the system represented by \eqref{systemDiffEq} resides in smooth $n$-dimensional \revision{Euclidean} manifolds and is Lipschitz continuous (\secref{sec:preliminaries}). We start with a reference solution path $\beta$ that is covered by $(r + 1)$ hyper-balls of maximal clearance radius $\delta$ (\figref{fig:ProofDiagram}) connecting $x_{start}$ and $x_{goal}$. To prove completeness property of GBRRT, we need to show that given a GBRRT vertex exists in the $i\mathit{th}$ ball, the probability $p_{GBRRT}$, that GBRRT in the next iteration will create a vertex in the $(i + 1)\mathit{th}$ ball when propagating from the $i\mathit{th}$ ball is bounded below by a positive constant. To prove this statement, we state the below propositions provided by Kleinbort et al.\cite{kleinbort2018probabilistic} for RRT with random edge-generation. We then use these propositions to show $p_{GBRRT} > 0$ (\lemmaref{lemma:gbrrtPositive}) and finally GBRRT is probabilistically complete (\theoremref{theorem:gbrrtProb}).

\begin{proposition}
\textbf{(From \cite{kleinbort2018probabilistic})} For system represented by equation \eqref{systemDiffEq} with Lipschitz continuity (\secref{sec:preliminaries}), the probability $p_{RRT}$ that RRT will generate a vertex in the \mbox{(i + 1)}$\mathit{th}$ ball when propagating from the i$\mathit{th}$ ball in the next iteration satisfies the condition $p_{RRT} > (\abs{\mathcal{B}_{\delta/5}}\cdot \rho)/\abs{\mathcal{X}}$ where $\rho$ is a positive constant, \revision{$\delta$ is the maximal clearance radius (Definition 8)}, $\mathcal{B}_{\delta/5}$ is the hyper-ball of radius $\delta/5$ centered at $x_{i}$,  $\abs{\mathcal{B}_{\delta/5}}$ and $\abs{\mathcal{X}}$ are the Lebesgue measures of $\mathcal{B}_{\delta/5}$ and $\mathcal{X}$ respectively.
\label{prop:RRTPropagation}
\end{proposition}

\begin{proposition}
\textbf{(From \cite{kleinbort2018probabilistic})} Given \propositionref{prop:RRTPropagation}, RRT is probabilistically  complete, i.e. the probability that RRT fails to reach $\mathcal{X}_{goal}$ from $x_{start}$ after $k$ iterations is at most $ae^{-bk}$ where $a, b > 0$.
\label{prop:RRTFails}
\end{proposition}
\begin{lemma}
$p_{GBRRT} \geq (1-q) p_{RRT}$ where $q$ is the exploitation ratio and $0 \leq q < 1$
\label{lemma:gbrrtPositive}.
\end{lemma}
\begin{proof}
For GBRRT, the probability of performing exploration using random edge-generation (having probablisitic completeness) during forward search is at least $1-q > 0$. Using multiplication rule of probability for independent events, it directly follows that $p_{GBRRT} \geq (1-q)p_{RRT} > 0$.
\end{proof}
\begin{theorem}
GBRRT is probabilistically complete
\label{theorem:gbrrtProb}
\end{theorem}
\begin{proof}
Using \lemmaref{lemma:gbrrtPositive} and \propositionref{prop:RRTFails}, it follows that GBRRT is probabilistically  complete.
\end{proof}
Like GBRRT, the probability of performing exploration using random edge-generation (having probabilistic completeness) for GABRRT is at least $1-q> 0$. This leads to the following corollary.
\vspace{-1mm}
\begin{corollary}
GABRRT is probabilistically complete
\label{theorem:gabrrtProb}
\end{corollary}

\revision{Note that the proof of Theorem~\ref{theorem:gbrrtProb} relies on the edge-generation method having the probabilistic completeness property in $(1 - q) > 0$ fraction of the iterations, in expectation. It is independent of the other edge-generation implementations that are used to provide fast exploration and exploitation in the remaining $q < 1$ proportion of iterations. Second, this analysis is only applicable to versions of GBRRT and GABRRT that generate asymptotically dense coverings of the sample space. It  may {\it not} apply when using pre-computed maneuver libraries.}

\section{EXPERIMENT SETUP}
\label{sec:experimental_setup}
We run multiple software simulations and hardware experiments to test the performance of \revision{our proposed algorithms}. We use the initial solution's computation time and the solution success rate as the performance metrics for comparing algorithms. We also record the initial solution cost (\figref{fig:InitialSolutionCost}) and provide it in the appendix for completeness of analysis but note that our algorithms are not designed to produce low cost solutions. We also report the total mission time for hardware experiments, which is the sum of the computation time and flight time.

\begin{figure}[t!]
\begin{subfigure}
\centering
 \figuretitle{Maneuver libraries}
  \includegraphics[width=3.5cm, height=2.5cm]{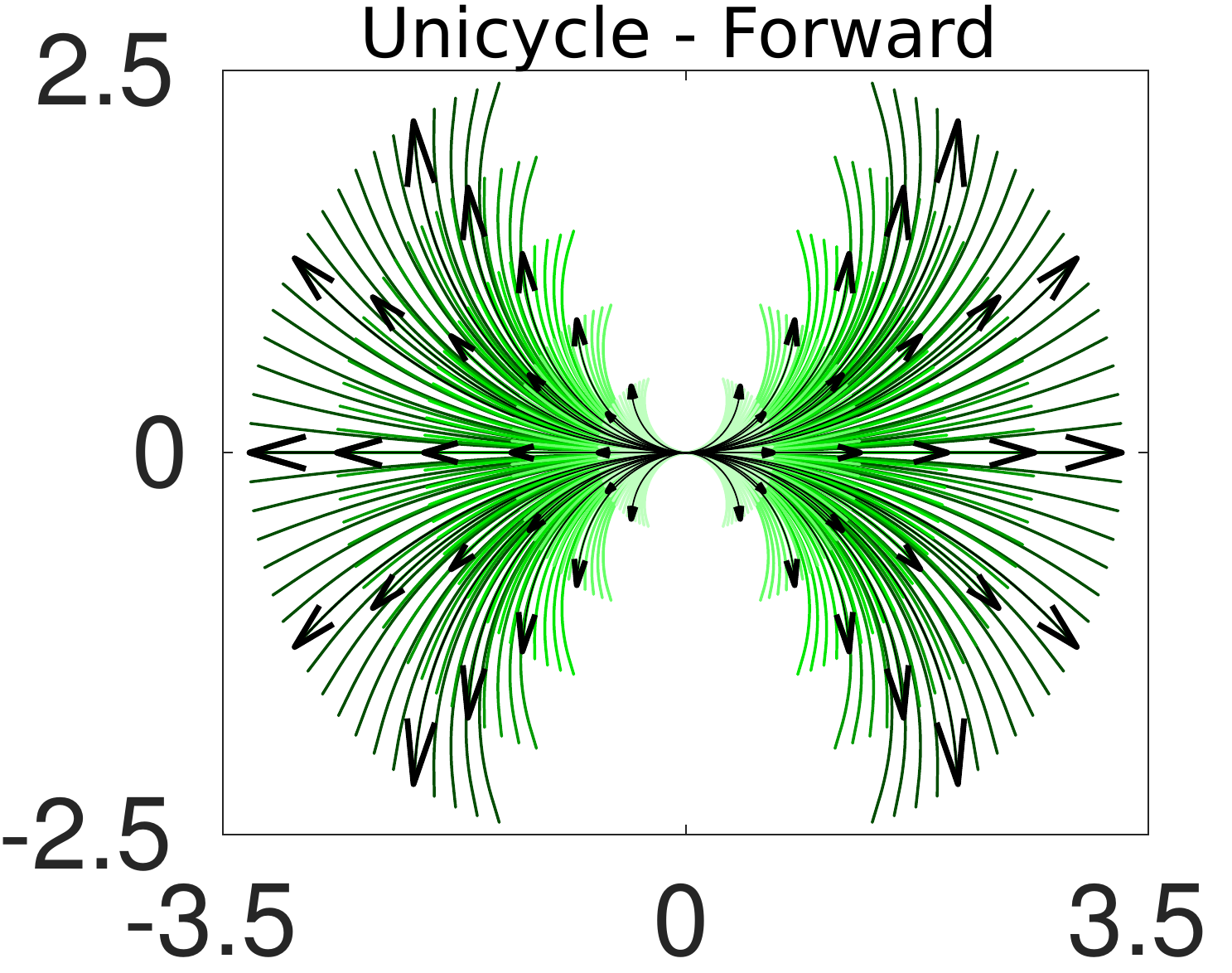}
\end{subfigure}
\begin{subfigure}
  \centering
  \includegraphics[width=3.5cm, height=2.5cm]{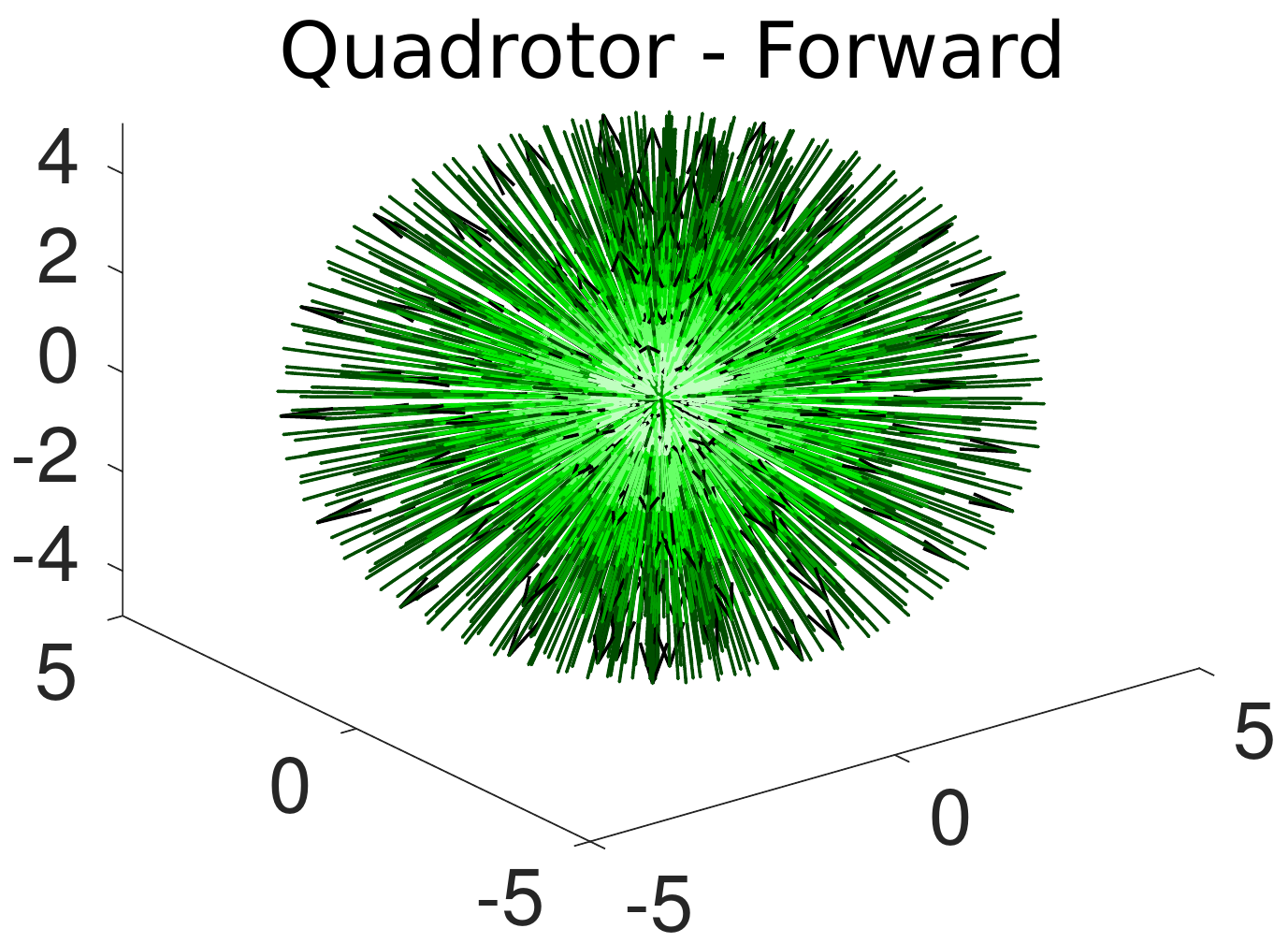}
\end{subfigure}%
\centering
\begin{subfigure}
  \centering
  \includegraphics[scale=0.2]{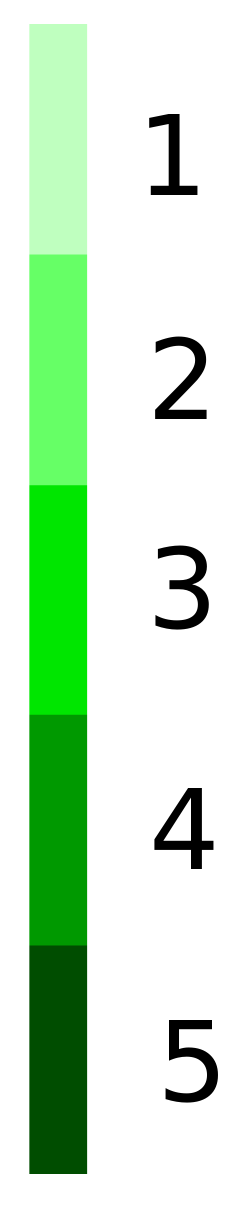}
\end{subfigure}%
\caption{Maneuver libraries for forward searches for unicycle (left) and quadrotor (right) in simulation experiments. Different shades indicate different reference distances from origin (quadrotor) and different peak velocities (unicycle) to achieve the maneuvers. Maneuvers at periodic intervals are shaded black for clarity.}
\label{fig:TrajectoryLibraries}
\vspace{-4mm}
\end{figure}

\begin{table*}
\centering
\caption{A summary of the dynamical systems used in experiments and the parameters used for evaluating GBRRT \revision{and GABRRT}}
\renewcommand{\arraystretch}{1.1}
\scalebox{0.85}{
\begin{tabular}[t]{|p{0.14\linewidth}|p{0.11\linewidth}|p{0.045\linewidth}|p{0.05\linewidth}|p{0.02\linewidth}|p{0.02\linewidth}|p{0.02\linewidth}|p{0.02\linewidth}|p{0.53\linewidth}|}
\hline
 Dynamical System & \revision{Type} & State Space Dim. & Control Space Dim. & $\delta_{hr}$ & $N_{B}$ & $q$ & \revisionTwo{$\gamma$} & Distance Function ($d_{X}$) \\ \hline
Kinematic Unicycle & \revision{Non-holonomic, Kinodynamic} & \hfil 3 & \hfil - & \hfil 7 & \hfil 40 & \hfil 0.8 & \revisionTwo{14} &
{\vspace{-.15cm} \scalebox{0.9}{ $\sqrt{\paren{x_{1} - x_{2}}^{2} + \paren{y_{1} - y_{2}}^{2}}$}} \\ \hline
Kinematic Quadrotor & \revision{Non-holonomic, Kinodynamic} & \hfil 3 & \hfil - & \hfil 8 & \hfil 90 & \hfil 0.8 & \revisionTwo{16} & {\vspace{-.15cm} $\sqrt{\paren{x_{1} - x_{2}}^{2} + \paren{y_{1} - y_{2}}^{2} + \paren{z_{1} - z_{2}}^{2}}$} \\ \hline
Cart-Pole & \revision{Holonomic \cite{cai2016fundamentals}, Kinodynamic} & \hfil 4 & \hfil 1 & \hfil 6 & \hfil 7 & 0.7 & \revisionTwo{10} & {\vspace{-.15cm} $\sqrt{\paren{x_{1} - x_{2}}^{2} + \paren{1.5}^{2} \paren{\Delta{\paren{\theta_{1}, \theta_{2}}}}^{2} + \paren{v_{1} - v_{2}}^{2} + \paren{\omega_{1} - \omega_{2}}^{2}} $} \\ \hline
Treaded-Vehicle & \revision{Non-holonomic, Kinodynamic} & \hfil 5 & \hfil 2 & \hfil 3 & \hfil \revision{7} & \hfil 0.7 & \revisionTwo{7} & {\vspace{-.15cm} $\sqrt{\parenShort{x_{1} - x_{2}}^{2} + \parenShort{y_{1} - y_{2}}^{2} + \revision{\parenShort{0.25}^{2} \parenShort{\parenShort{v_{L1} - v_{L2}}^{2} + \parenShort{v_{R1} - v_{R2}}^{2}}}}$} \\ \hline
Car with Trailer & \revision{Non-holonomic, Kinodynamic } & \hfil 6 & \hfil 2 & \hfil 4 & \hfil 7 & 0.7 & \revisionTwo{8} & {\vspace{-.15cm}$\sqrt{\paren{x_{1} - x_{2}}^{2} + \paren{y_{1} - y_{2}}^{2} + \revision{\paren{0.25}^{2} \paren{v_{1} - v_{2}}^{2}}}$} \\ \hline
Fixed-Wing Airplane & \revision{Non-holonomic, Kinodynamic} & \hfil 9 & \hfil 3 & \hfil 6 & \hfil \revision{7} & 0.7 & \revisionTwo{10} & {\vspace{-.15cm} \scalebox{0.87}{$\sqrt{\paren{x_{1} - x_{2}}^{2} + \paren{y_{1} - y_{2}}^{2} + \paren{z_{1} - z_{2}}^{2} + \revision{\paren{0.9}^{2}}(\paren{\dot{x}_{1} - \dot{x}_{2}}^{2} + \paren{\dot{y}_{1} - \dot{y}_{2}}^{2} + \paren{\dot{z}_{1} - \dot{z}_{2}}^{2}})$}\vspace{-.1cm}} \\ \hline
\end{tabular}}
\label{Tab:ComComparison}
\vspace{1mm}
\caption*{
The $\Delta{\paren{\theta_{1}, \theta_{2}}}$  in cart-pole $d_{X}$ represents the ``shortest distance" \cite{kuffner2004effective} between two angles accounting for wrap-around with $\Delta\paren{{\theta_{1}, \theta_{2}}} \in [-\pi, +\pi]$. The $\Delta\paren{{\theta_{1}, \theta_{2}}}^{2}$ term is multiplied by 2.25 to give more weight to angular position of the pole. \revision{The velocity components of the treaded, car-trailer and fixed-wing systems are multiplied by values $<$ 1 to provide more priority to position than velocity. For GABRRT, we use the same parameters as GBRRT and determine $d^{\mathtt{ND}}_{X}$ from $d_{X}$ by dropping the velocity terms. \revisionTwo{Finally, we choose the value of $\gamma$ such that $r_{k} < \delta_{hr}$ for approximately $80\%$ of the initial solution planning time on average (\figref{fig:SoftwareResultGraphs})}.}}
\end{table*}

\begin{figure*}[ht]
\centering
\begin{minipage}[c]{0.16\textwidth}
\begin{xy}
\xyimport(100, 100){\includegraphics[width=\simSolWidth, height=\simSolHeight]{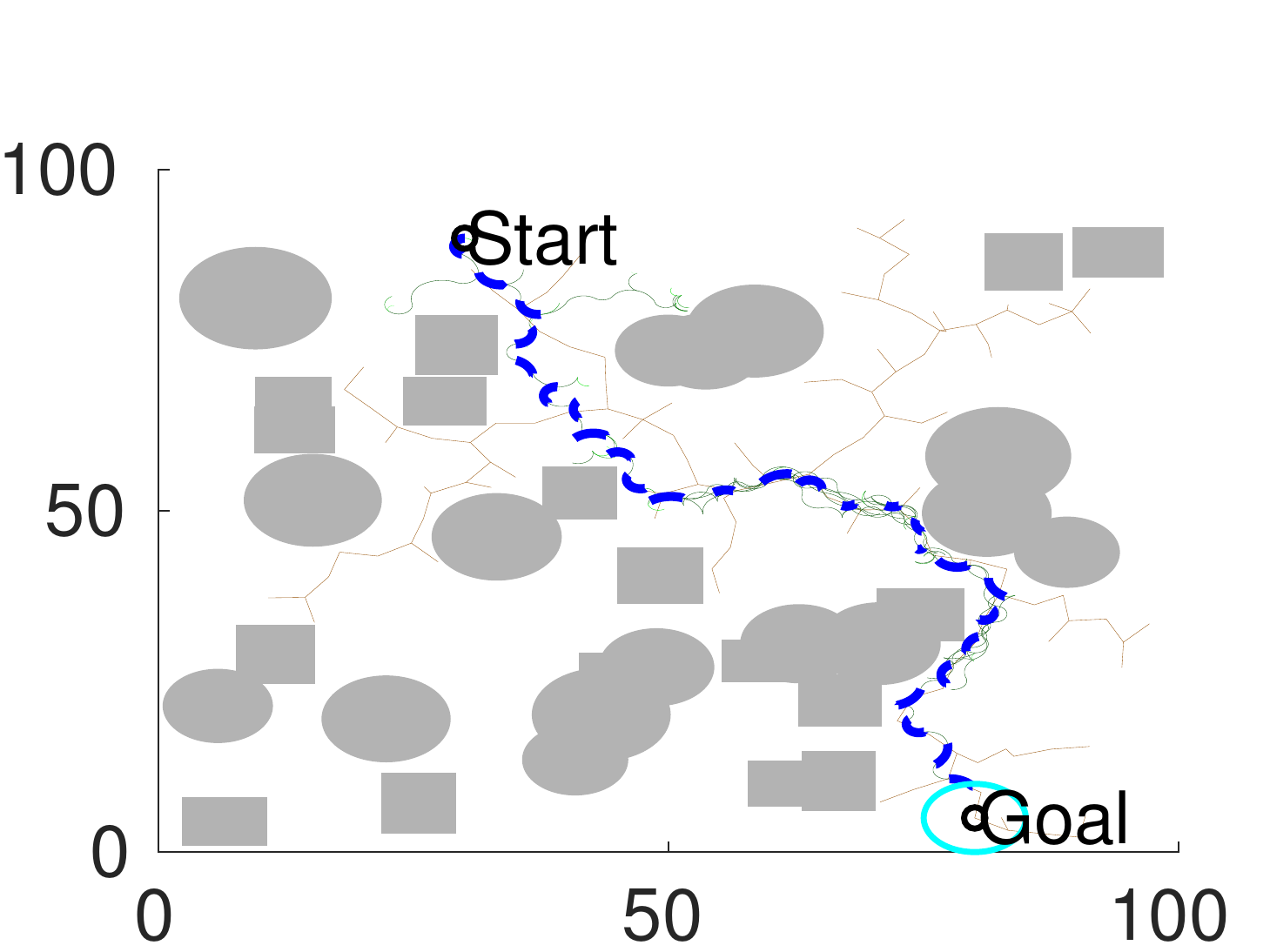}}
,(55, 85)*{\text{Unicycle (3D)}}
\end{xy}
\end{minipage}
\begin{minipage}[c]{0.16\textwidth}
\begin{xy}
\xyimport(100, 100){\includegraphics[width=\simSolWidth, height=\simSolHeight]{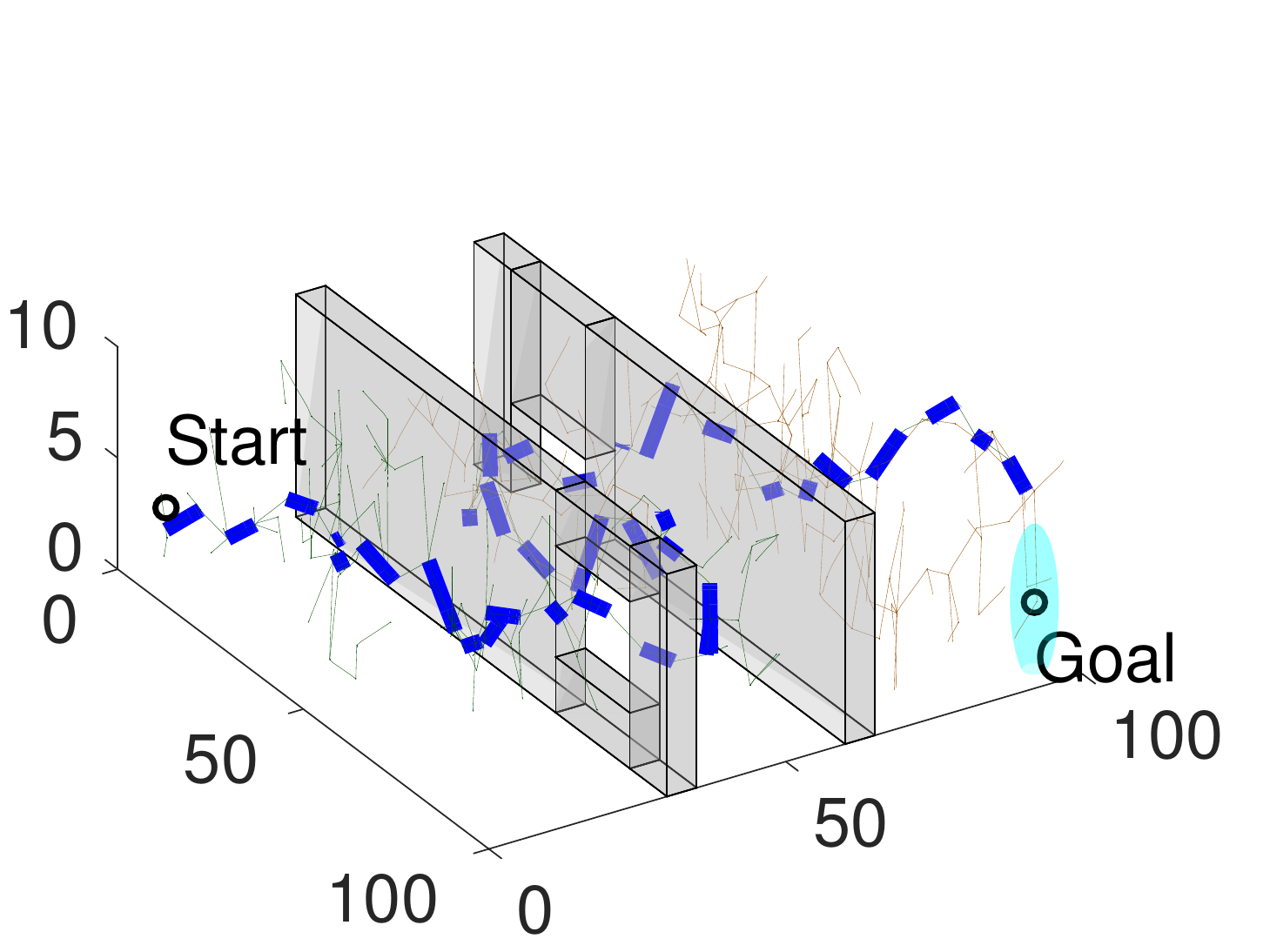}}
,(55, 85)*{\text{Quadrotor (3D)}}
\end{xy}
\end{minipage}
\begin{minipage}[c]{0.16\textwidth}
\begin{xy}
\xyimport(100, 100){\includegraphics[width=\simSolWidth, height=\simSolHeight]{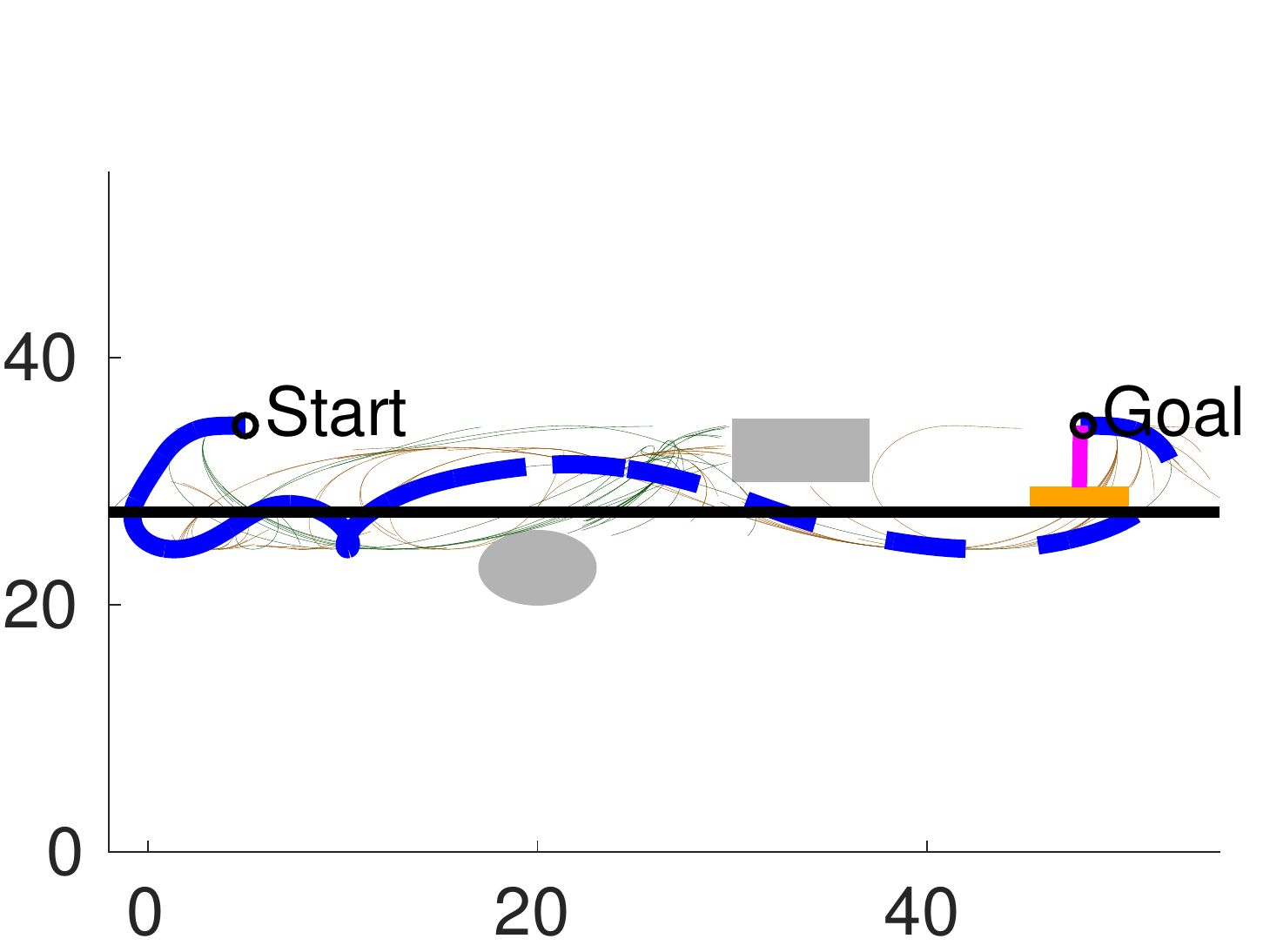}}
,(55, 85)*{\text{Cart-Pole (4D)}}
\end{xy}
\end{minipage}
\begin{minipage}[c]{0.16\textwidth}
\begin{xy}
\xyimport(100, 100){\includegraphics[width=\simSolWidth, height=\simSolHeight]{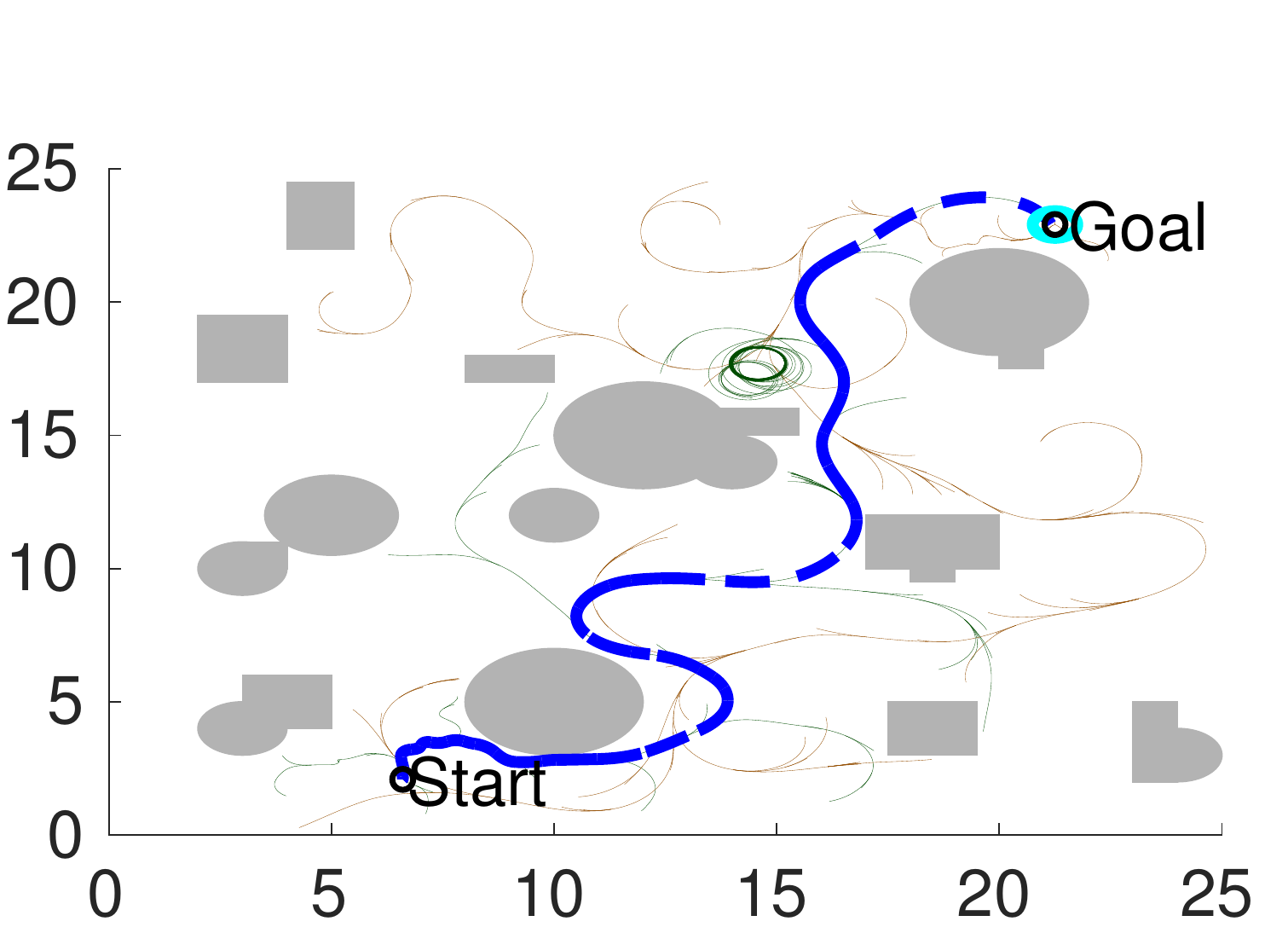}}
,(55, 85)*{\text{Treaded (5D)}}
\end{xy}
\end{minipage}
\begin{minipage}[c]{0.16\textwidth}
\begin{xy}
\xyimport(100, 100){\includegraphics[width=\simSolWidth, height=\simSolHeight]{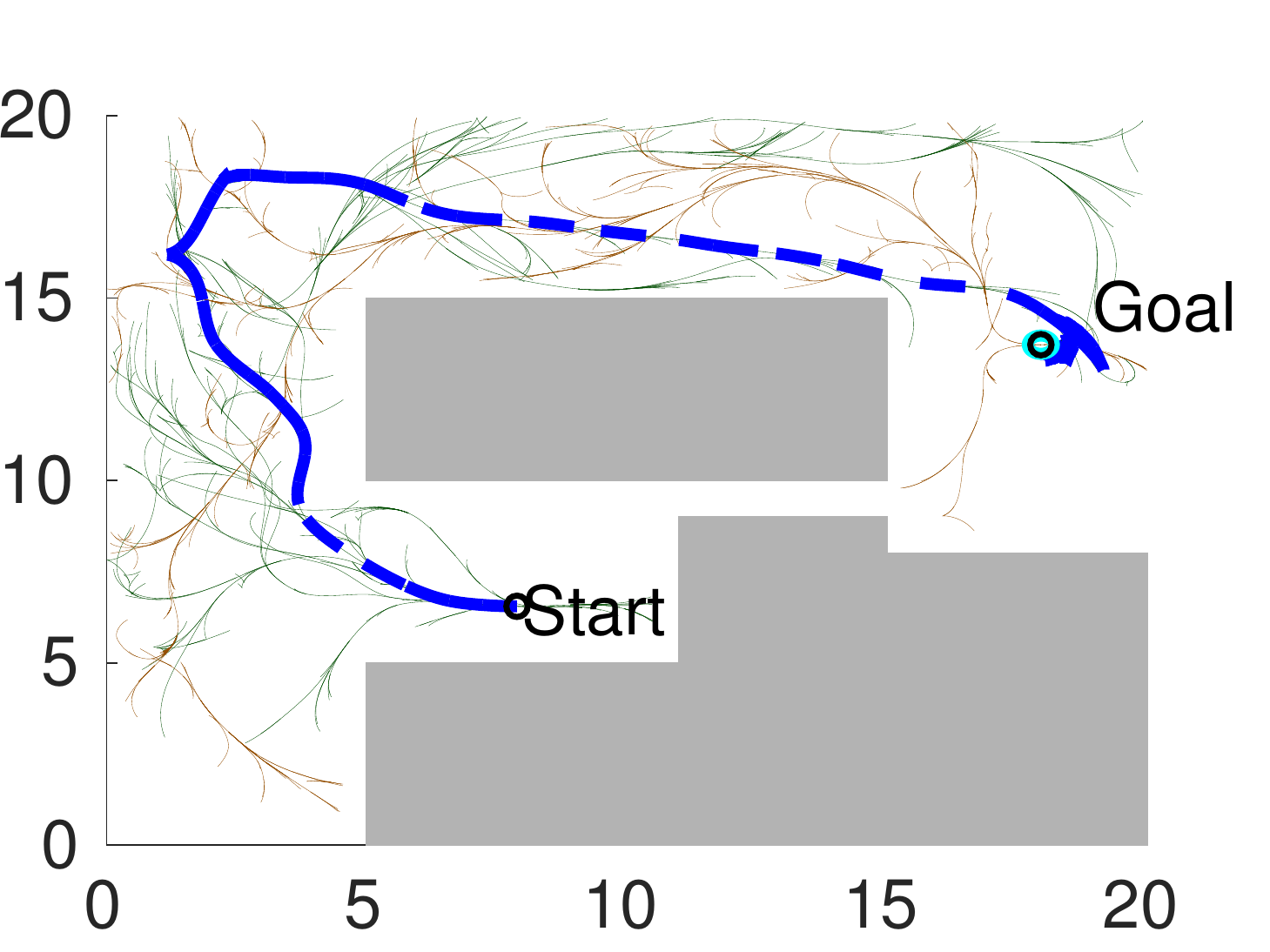}}
,(55, 92)*{\text{Car-Trailer 6D)}}
\end{xy}
\end{minipage}
\begin{minipage}[c]{0.16\textwidth}
\begin{xy}
\xyimport(100, 100){\includegraphics[width=\simSolWidth, height=\simSolHeight]{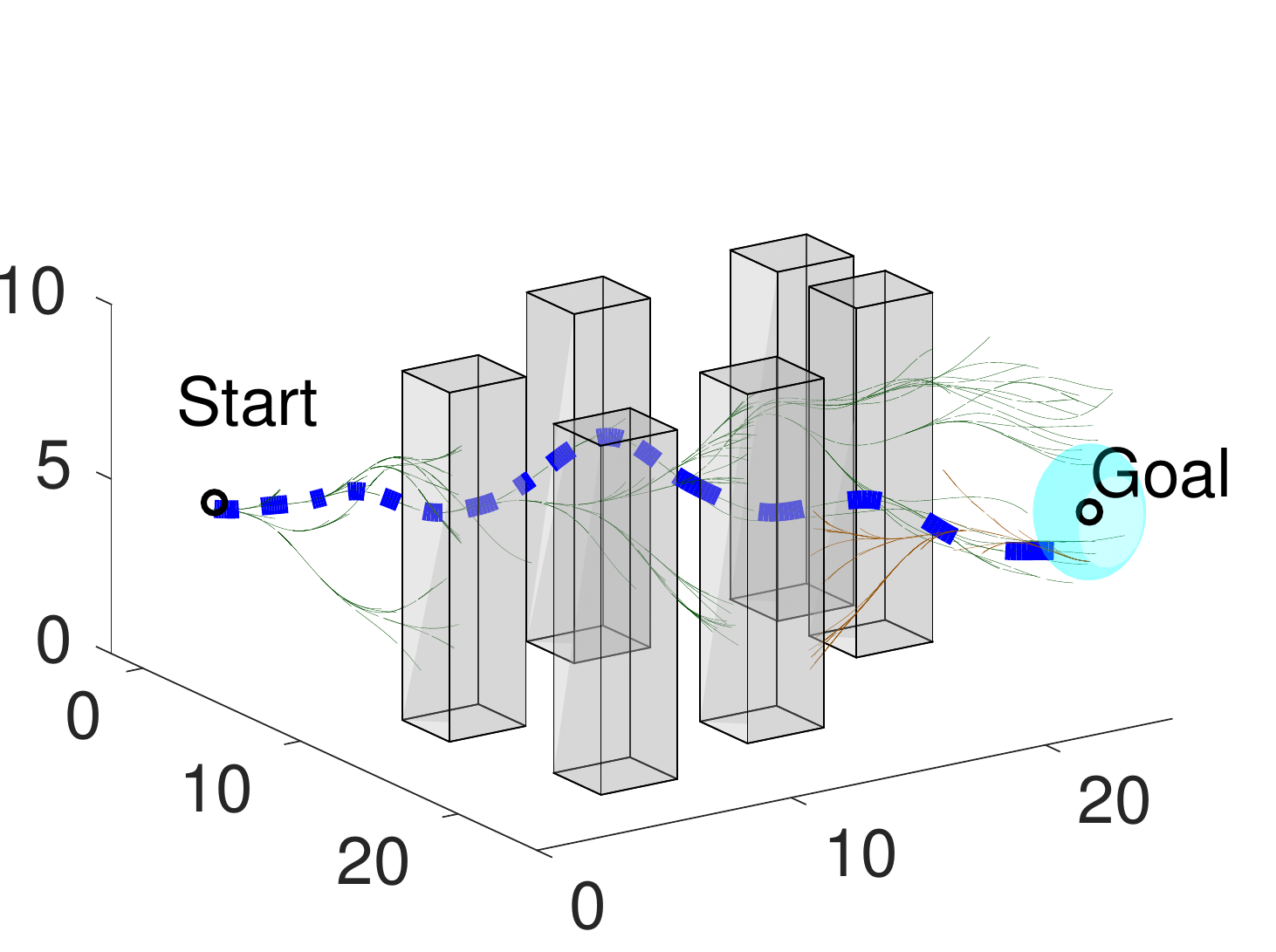}}
,(55, 85)*{\text{Fixed-Wing (9D)}}
\end{xy}
\end{minipage}
\begin{subfigure}
\centering
  \includegraphics[scale=0.55]{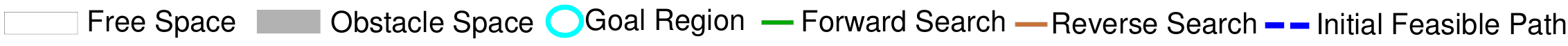}
\end{subfigure}
\caption{\revision{Initial feasible solution found by GABRRT (unicycle and quadrotor) and GBRRT (cart-pole, threaded, car-trailer, and fixed-wing).}}
\label{fig:SoftwareSimulationSnapshot}
\vspace{6mm}
\begin{subfigure}
  \centering
  \includegraphics[scale=\barDimScale]{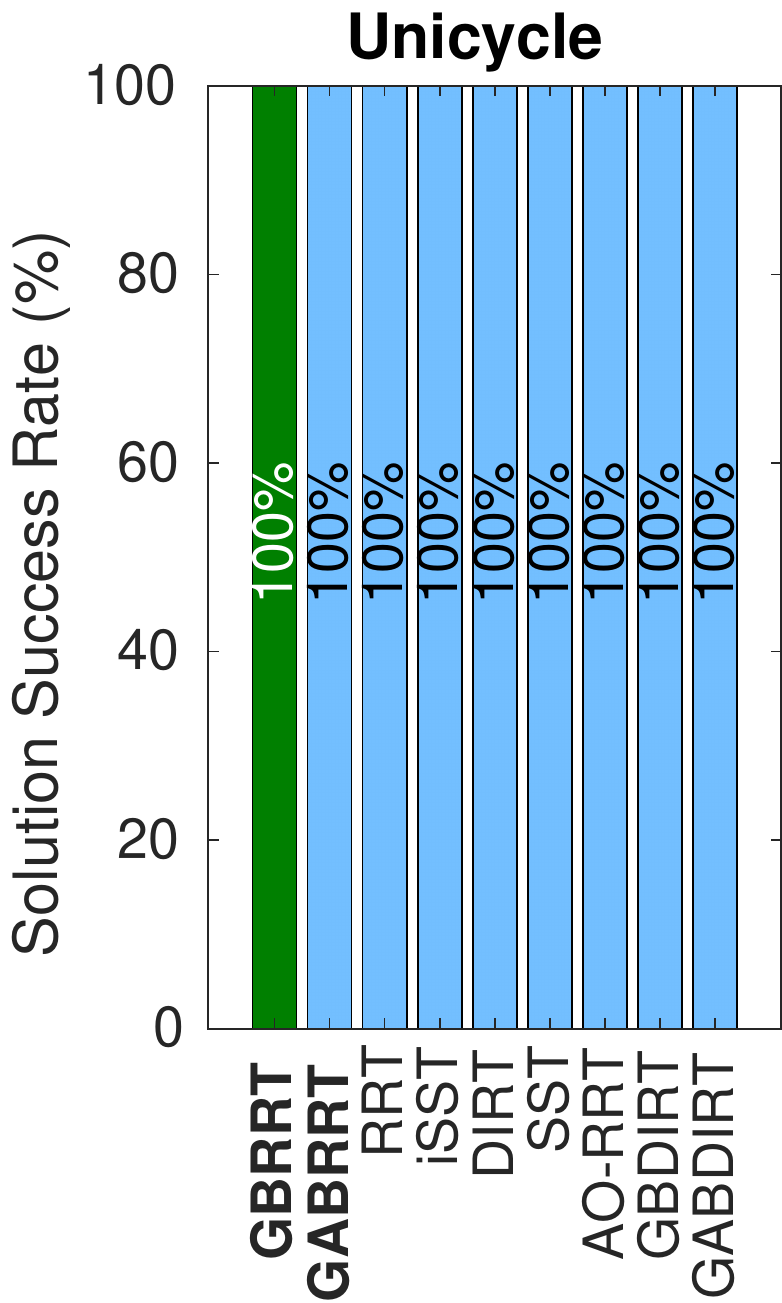}
\end{subfigure}%
\begin{subfigure}
  \centering
  \includegraphics[scale=\barDimScale]{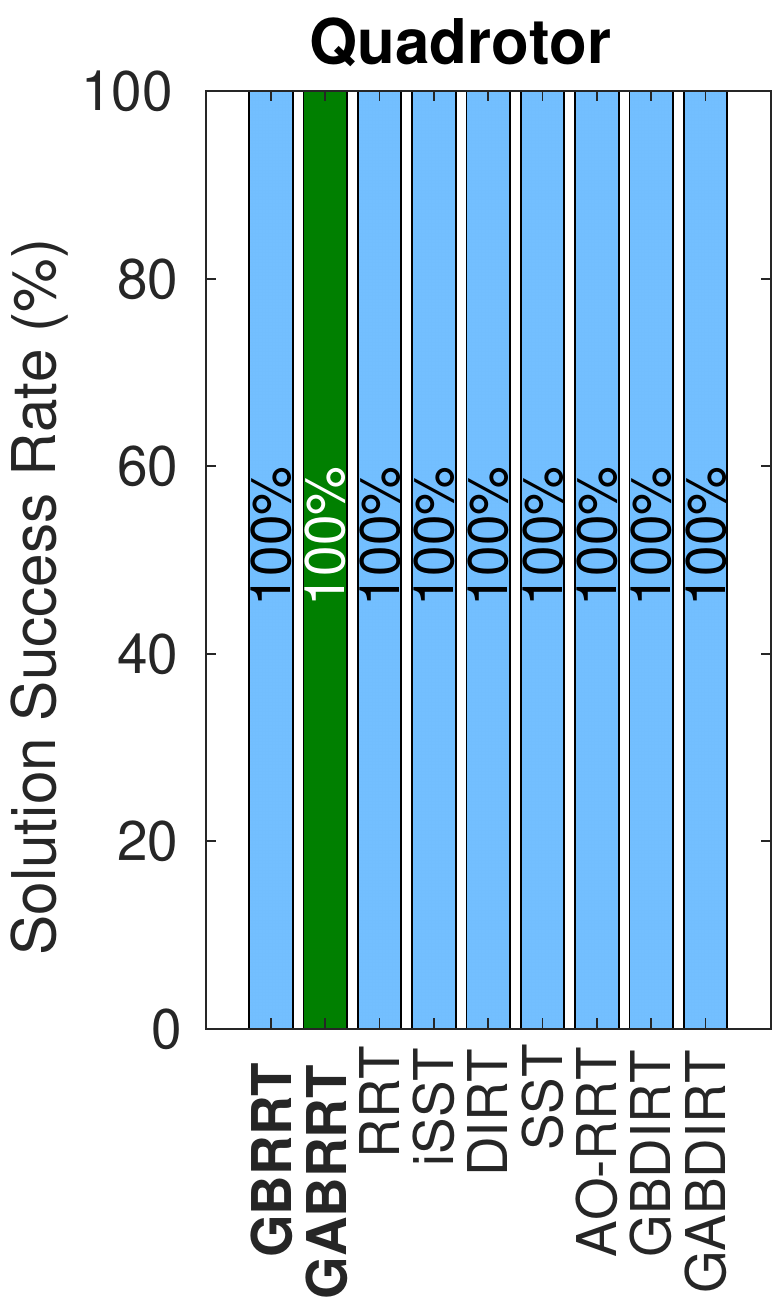}
\end{subfigure}%
\begin{subfigure}
  \centering
  \includegraphics[scale=\barDimScale]{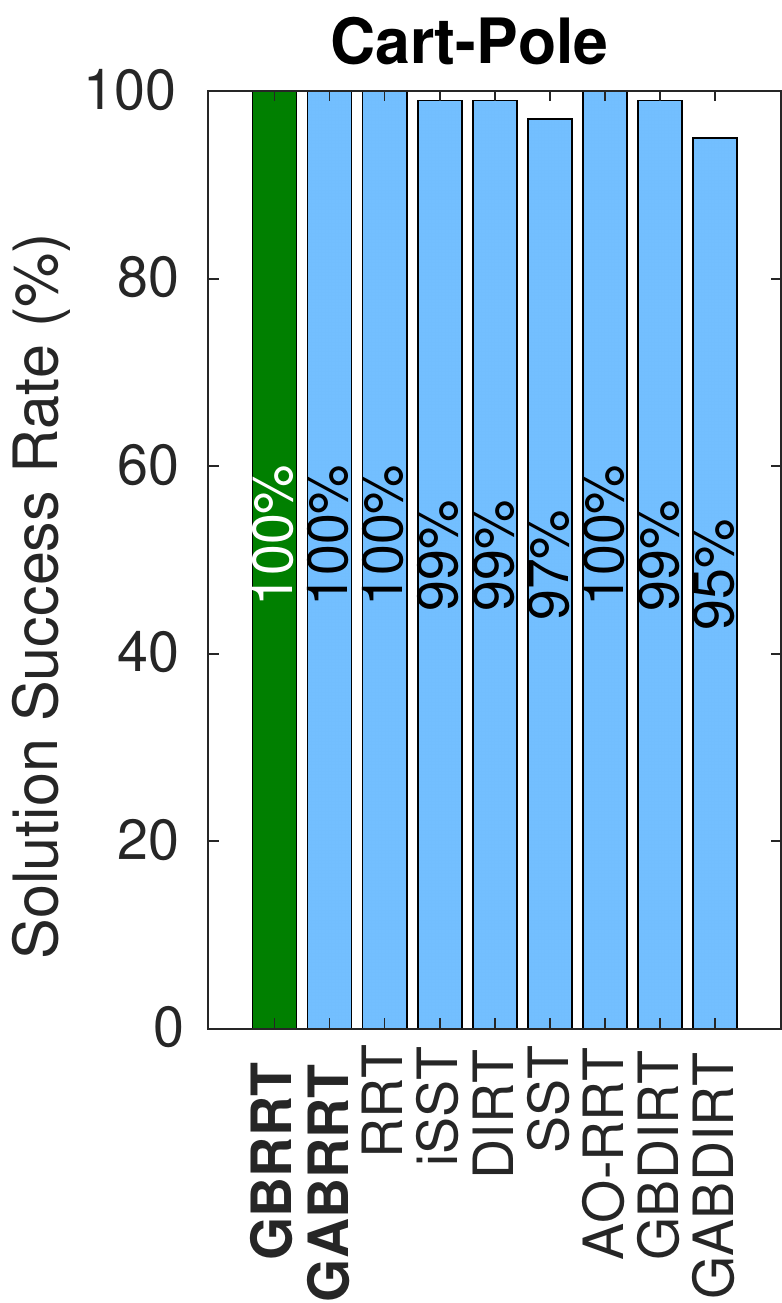}
\end{subfigure}%
\begin{subfigure}
  \centering
  \includegraphics[scale=\barDimScale]{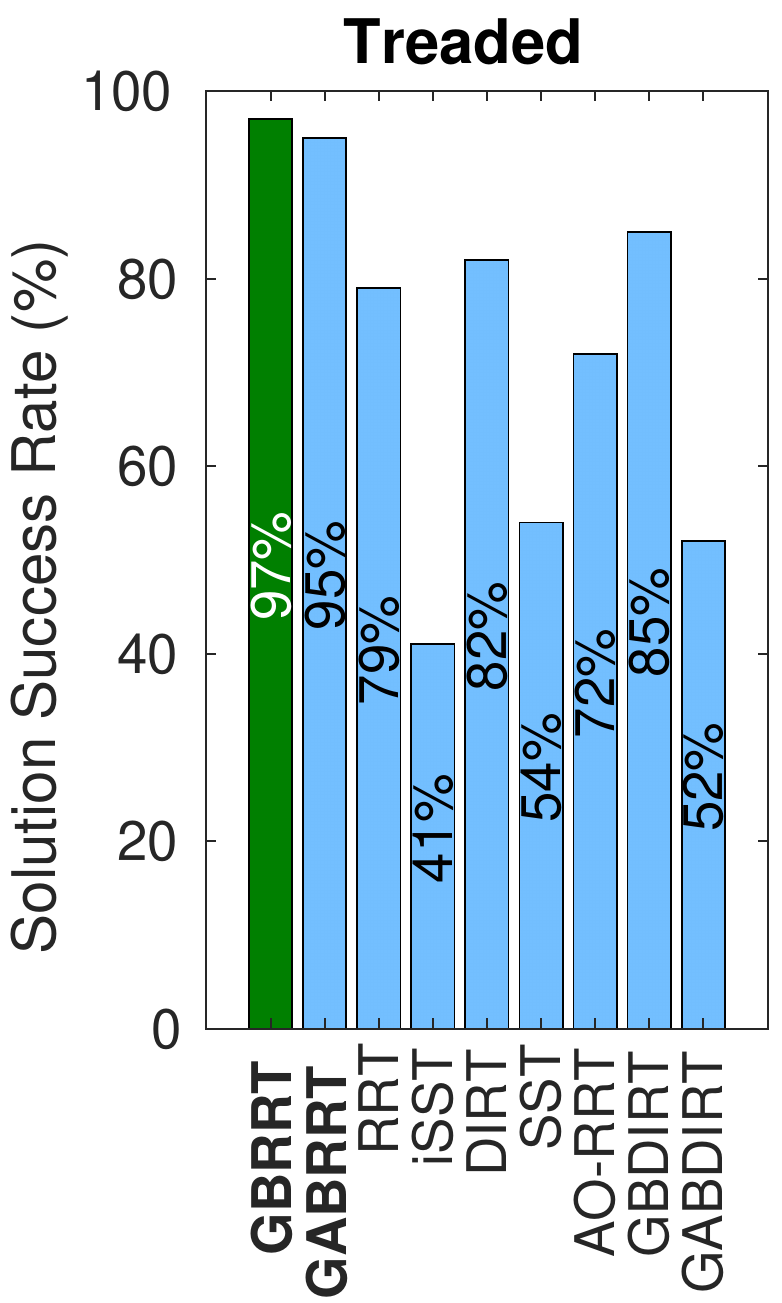}
\end{subfigure}%
\begin{subfigure}
  \centering
  \includegraphics[scale=\barDimScale]{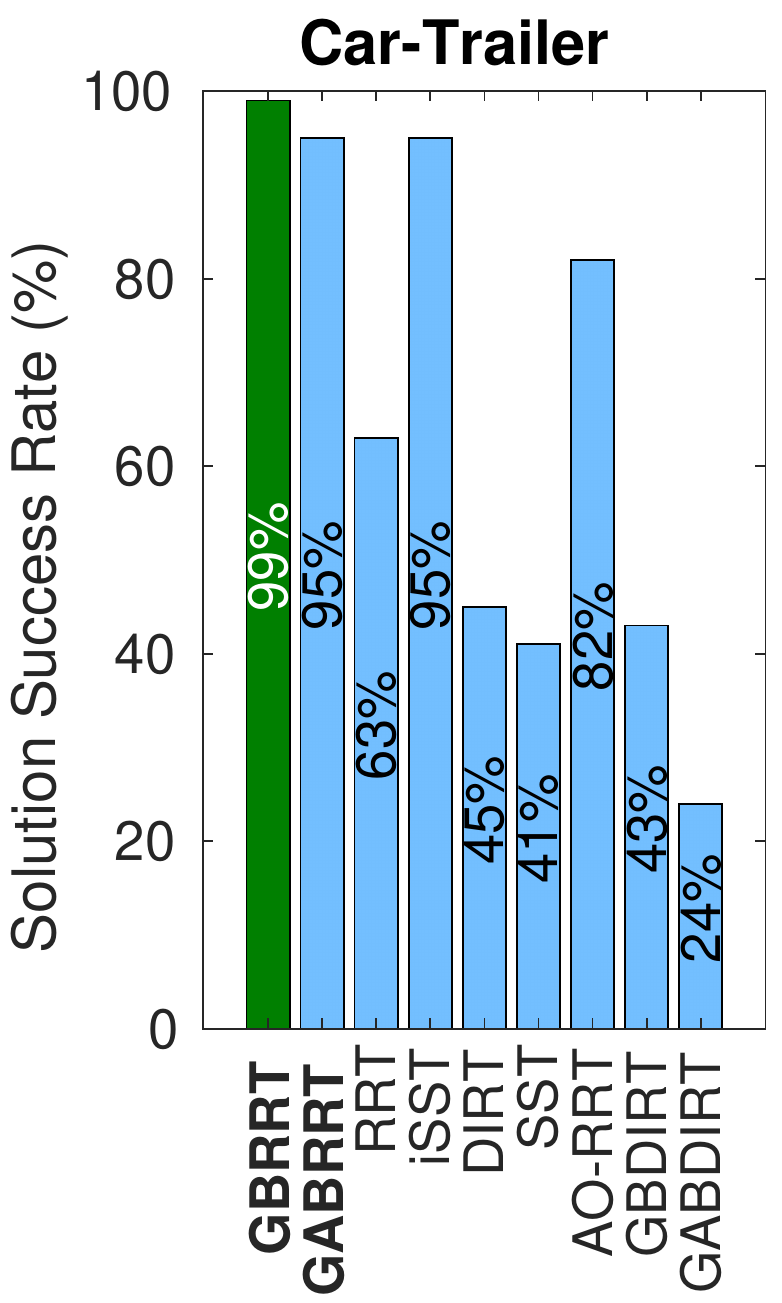}
\end{subfigure}%
\begin{subfigure}
  \centering
  \includegraphics[scale=\barDimScale]{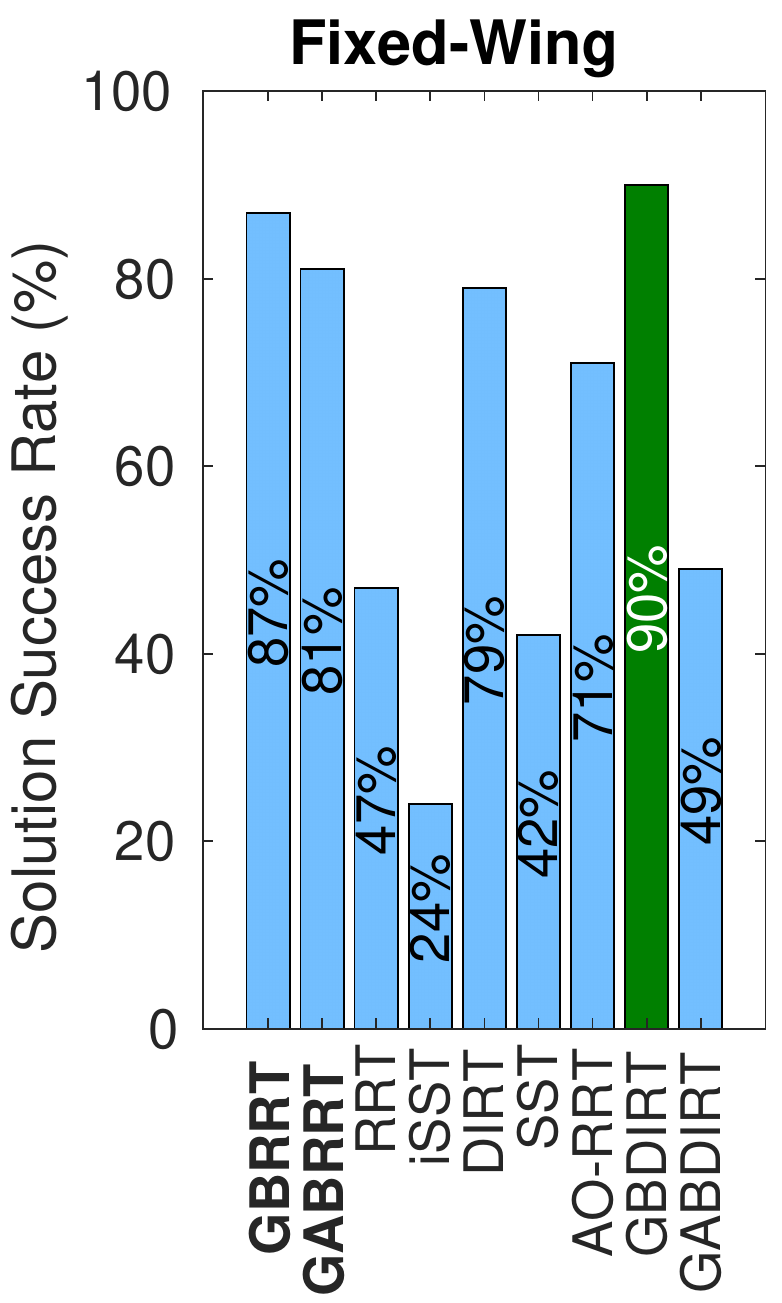}
\end{subfigure}%

\begin{subfigure}
  \centering
  \includegraphics[scale=\compareResultScale]{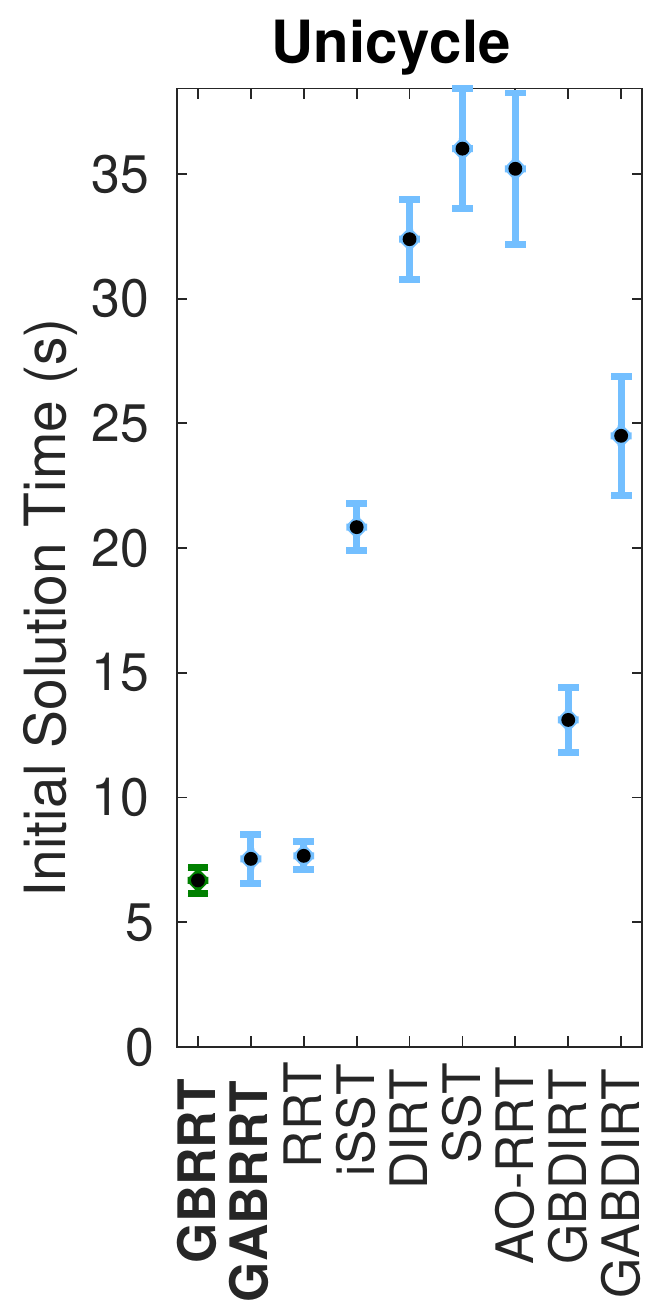}
\end{subfigure}%
\begin{subfigure}
  \centering
  \includegraphics[scale=\compareResultScale]{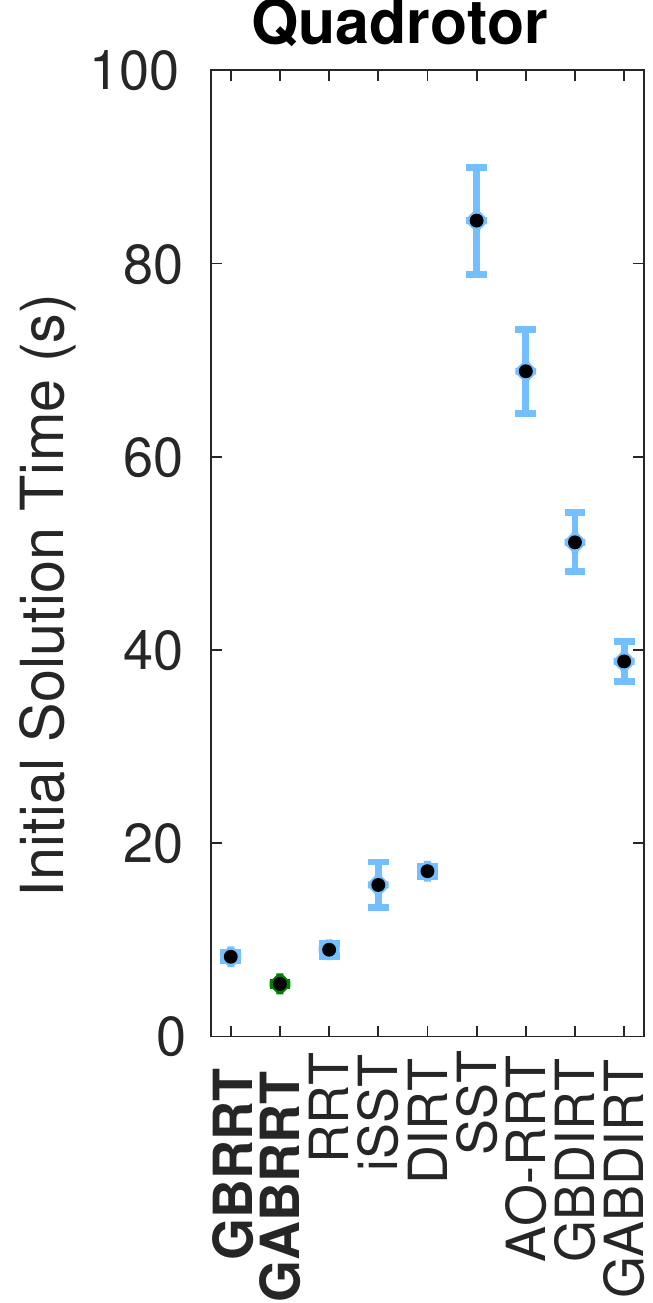}
\end{subfigure}%
\begin{subfigure}
  \centering
  \includegraphics[scale=\compareResultScale]{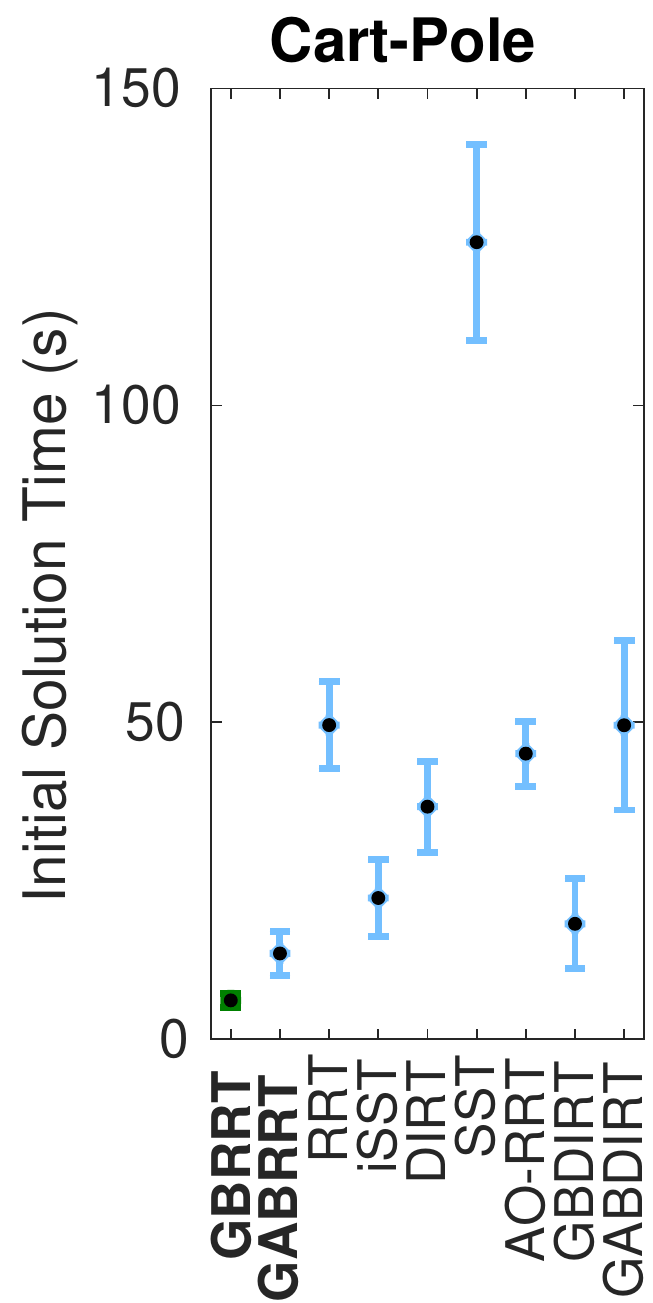}
\end{subfigure}%
\begin{subfigure}
  \centering
  \includegraphics[scale=\compareResultScale]{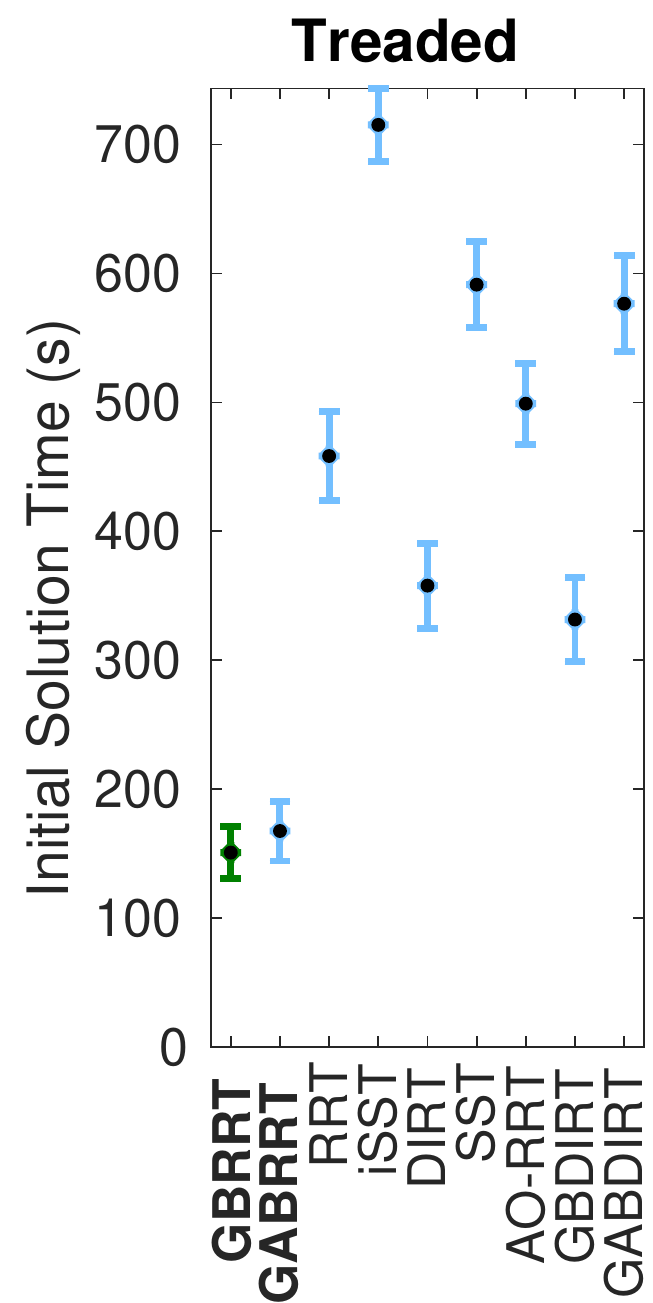}
\end{subfigure}%
\begin{subfigure}
  \centering
  \includegraphics[scale=\compareResultScale]{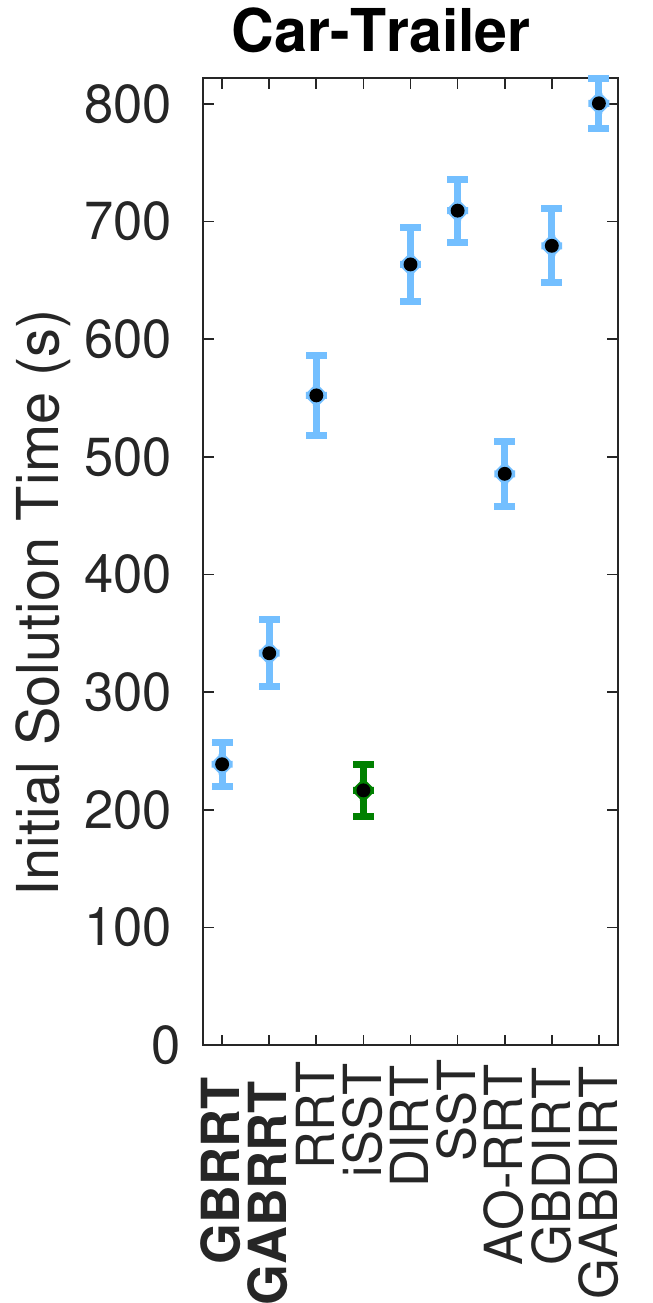}
\end{subfigure}%
\begin{subfigure}
  \centering
  \includegraphics[scale=\compareResultScale]{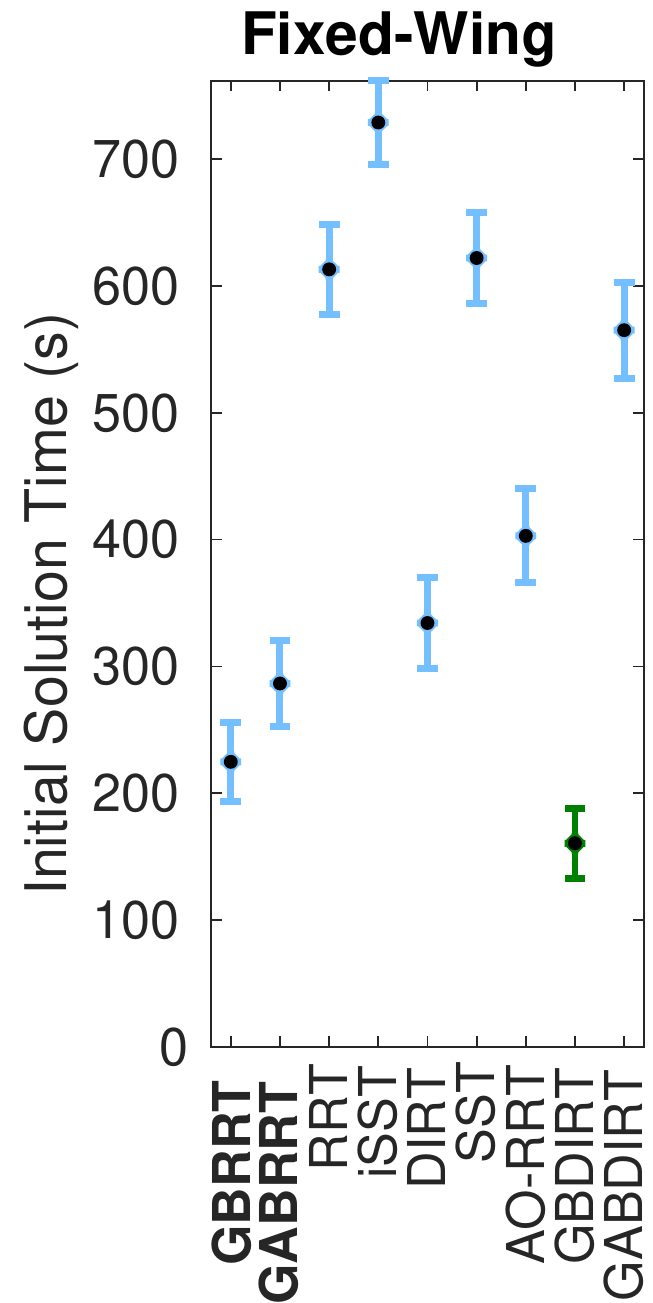}
\end{subfigure}%
\caption{Top row shows the percentage solution success rate, and the bottom row shows the mean and standard error of the initial solution time for 100 trials of running algorithms using various systems. The best performing algorithm is highlighted in green. The solution success rate is less than 100 \% for algorithms because they are not able to find a solution within the max time $S_{max}$. The mean and standard error are calculated using results from all 100 trials (successful and unsuccessful) with $S_{max}$ value used for unsuccessful trials.}
\label{fig:SoftwareResultGraphs}
\vspace{-2mm}
\end{figure*}

\subsection{Simulation experiments}
We run simulation experiments using six different dynamical systems - a kinematic unicycle (3D), a kinematic quadrotor (3D), cart-pole (4D), treaded (skid-steer) vehicle (5D), second-order car with trailer (6D), and fixed-wing airplane (9D). The distance function $d_{X}$ and GBRRT's parameter values for each system are specified in \tableref{Tab:ComComparison}. We run 100 trials for each dynamical system, with each trial having randomized start and goal states. \revision{At each goal state, we set the corresponding velocities to zero so that the systems comes to rest (except for the fixed-wing aircraft which is assumed to require a minimum velocity). We use $S_{max}$ to represent the maximum time allotted to each algorithm to generate a system solution. We set $S_{max}=$ 10 minutes for all algorithms for unicycle, quadrotor and cart-pole and $S_{max}=15$ minutes for threaded, car-trailer and fixed-wing systems.} 
\par For trajectory generation, we use pre-built trajectory libraries for the unicycle and quadrotor and online integration using Runge-Kutta order 4 (RK4) \cite{lambert1991numerical} for the other four systems \revision{(for all algorithms)}. We use these different methods of generating trajectories to show that \revision{our algorithms} work with distinct types of trajectory generation methods. The advantage of using trajectory libraries is that there is little overhead in considering many trajectories for the ``best-input" propagation. We limit the number of trajectories in the considered libraries (\figref{fig:TrajectoryLibraries}) by designing them such that linear and angular velocities are zero at the start and endpoints. This \revision{velocity} constraint is not assumed while performing online integration for the remaining systems. The backward integration for the reverse tree trajectories is carried out by flipping the time resolution sign and reversing the integrator's time limits.

\par The dynamical systems are explained in detailed below.

\subsubsection{Unicycle}
The unicycle is 3-dimensional ($x$, $y$, $\theta$) where x and y correspond to Cartesian positions and $\theta$ is the angle from the horizontal. The zero terminal (linear and angular) velocities are achieved  by using a time-varying non-linear controller \cite{carona2008control}. The maneuvers are generated by varying the maximum linear velocity along the trajectory from 1 to 5 m/s in 1 m/s increments and the angular velocities varying from 0 to $\pi/2$ rad/s in $\pi/18$ rad/s increments in all four quadrants. Its kinematics is represented by $\dot{x} = u\cos(\theta)$, $\dot{y} = u\sin(\theta)$, $\dot{\theta} = \omega$ where  $u$ is linear speed and $\omega$ is angular velocity.

\subsubsection{Quadrotor}
The quadrotor system is 3-dimensional ($x$, $y$, $z$) where x, y, and z correspond to Cartesian positions. The zero terminal velocities are achieved using a feedback linearization controller \cite{palunko2011adaptive}. The maneuver library is generated by providing final positions 1 to 5 meters distance away from the initial position (origin) in 10\degree \, increments in the heading. Its dynamics is represented by $M\vec{\dot{v}} + C(\vec{v})\vec{v} + D\vec{v} + g(\vec{\eta}) = \vecm{u}$, where $M$ is the mass matrix, $C(\vec{v})$ is the Coriolis-centripetal matrix, D is the damping matrix, $g$ is the gravity vector, $\vec{\eta}$, $\vec{v}$ and $\vecm{u}$ are the state, velocity and control input vectors respectively. The definition of these terms are given in \cite{palunko2011adaptive}.

\subsubsection{Cart-Pole}
The cart-pole system \cite{li2016asymptotically}\cite{papadopoulos2014analysis} consists of a pendulum attached to a cart which is restricted to move along a line. The state space is 4-dimensional ($x$, $\theta$, $v$, $\omega$) where $x$ and $v$ corresponds to the position and linear velocity of the cart and $\theta$ and $\omega$ corresponds to the angular position (from downward vertical) and angular velocity of the pendulum. The control space is 1-dimensional which is the force F on the cart. Its dynamics is represented by 
${\dot{x} = v}$, 
${\dot{\theta} = \omega}$,
${\dot{v} = \frac{(I + mL^{2})(F + mL\omega^{2}\sin(\theta)) +  (mL)^{2}\cos(\theta)\sin(\theta)g}{(M + m)(I + mL^{2}) - (mL)^{2}\cos^{2}{\theta}}}$,
${\dot{\omega} = \frac{(-mL\cos(\theta))(F + mL\omega^{2}\sin(\theta)) + (M + m)(-mgL\sin(\theta))}{(M + m)(I + mL^{2}) - (mL)^{2}\cos^{2}{\theta}}}$,
where $I$ is the moment of inertia, $L$ is the length of the pendulum, $g$ is the acceleration due to gravity, $M$ and $m$ are the masses of the cart and pendulum respectively.
\subsubsection{Treaded (Skid-Steer) Vehicle}
The skid steer vehicle \cite{pentzer2014model}\cite{sivaramakrishnan2019towards} is 5-dimensional ($x$, $y$, $\theta$, $v_{L}$, $v_{R}$) where $x$ and $y$ are Cartesian positions, $\theta$ is the angle of vehicle from horizontal, and $v_{L}$ and $v_{R}$ are the velocities of left and right wheels respectively. The control space is 2-dimensional ($a_{L}$, $a_{R}$) which corresponds to accelerations of left and right wheels. Its dynamics is represented by
${\dot{x} = v_{x}\cos(\theta)-v_{y}\sin(\theta)}$, 
${\dot{y} =  v_{x}\sin(\theta)-v_{y}\cos(\theta)}$, 
${\dot{\theta} = w_{z}}$, 
${\dot{v}_{L} = a_{L}}$, 
${\dot{v}_{R} = a_{R}}$, 
where $v_{x}$, $v_{y}$ and $w_{z}$ are functions of $v_{L}$ and $v_{R}$ and their definitions are provided in \cite{pentzer2014model}.
\subsubsection{Second-Order Car with Trailer}
The second order car with trailer \cite{littlefield2018efficient}  \cite{murray1993nonholonomic} is a 6-dimensional ($x$, $y$, $\theta$, $v$, $\omega$, $\theta_{1}$) system where $x$ and $y$ are Cartesian positions, $v$ is the linear velocity, $\theta$ is the angle of car from the horizontal, $\omega$ is the angular velocity, and $\theta_{1}$ is the angle of the trailer from the horizontal. The control space is 2-dimensional \revision{($a$, $\alpha$)} which corresponds to the desired linear and angular accelerations. Its dynamics is represented by $\dot{x} = vcos(\theta)\cos(\omega)$, $\dot{y} = v\sin(\theta)\cos(\omega)$, $\dot{\theta} = v\sin(\omega)$, $\dot{v} = \revision{a}$, $\dot{\omega} = \revision{\alpha}$, $\dot{\theta_{1}} = v\sin{\paren{\theta - \theta_{1}}}$ respectively.

\subsubsection{Fixed-Wing Airplane}
The fixed-wing airplane \cite{li2016asymptotically} \cite{paranjape2015motion} is 9-dimensional ($x$, $y$, $z$, $v$, $\omega$, $\beta$, $T$, $\alpha$, $\mu$) where $x$, $y$, $z$ are the Cartesian positions, $v$ - speed, $\omega$ - flight path angle, $\beta$ - heading angle, $T$ - thrust per unit mass, $\alpha$ - angle of attack and $\mu$ - wind axis roll angle. The control space is 3-dimensional ($T_{c}$, $\alpha_{c}$ $\mu_{c}$) which are the commanded thrust per unit mass, commanded angle of attack and commanded wind axis roll angle respectively. Its dynamics is represented by 
${\dot{x} = v\cos(\omega)\cos(\beta)}$, 
${\dot{y} = v\cos(\omega)\sin(\beta)}$, 
${\dot{z} = v\sin(\omega)}$, 
${\dot{T} = c_{T}\paren{T_{c} - T}}$, 
${\dot{\alpha} = c_{\alpha}\paren{\alpha_{c} - \alpha}}$, 
${\dot{\mu} = c_{\mu}\paren{\mu_{c} - \mu}}$,
${\dot{v} = \parenShort{T\cos(\alpha) - kv^{2}C_{D}(\alpha)} - g\sin(\omega)}$,
${\dot{\omega} = \paren{\frac{Tsin(\alpha)}{v} + kvC_{L}(\alpha)}cos(\mu) - \frac{gcos(\omega)}{v}}$,
${\dot{\beta} = \parenShort{\frac{T\sin(\alpha)}{v} + kvC_{L}(\alpha)}\frac{\sin(\mu)}{\cos(\omega)}}$,
where $g$ is the acceleration due to gravity, $k$ is the scaled inverse of the wing loading, $C_{D}(\alpha)$ and $C_{L}(\alpha)$ are the coefficients of lift and drag respectively.

\begin{figure*}[ht]
\centering
\begin{subfigure}
\centering
  \includegraphics[width=2.9cm, height=2.9cm]{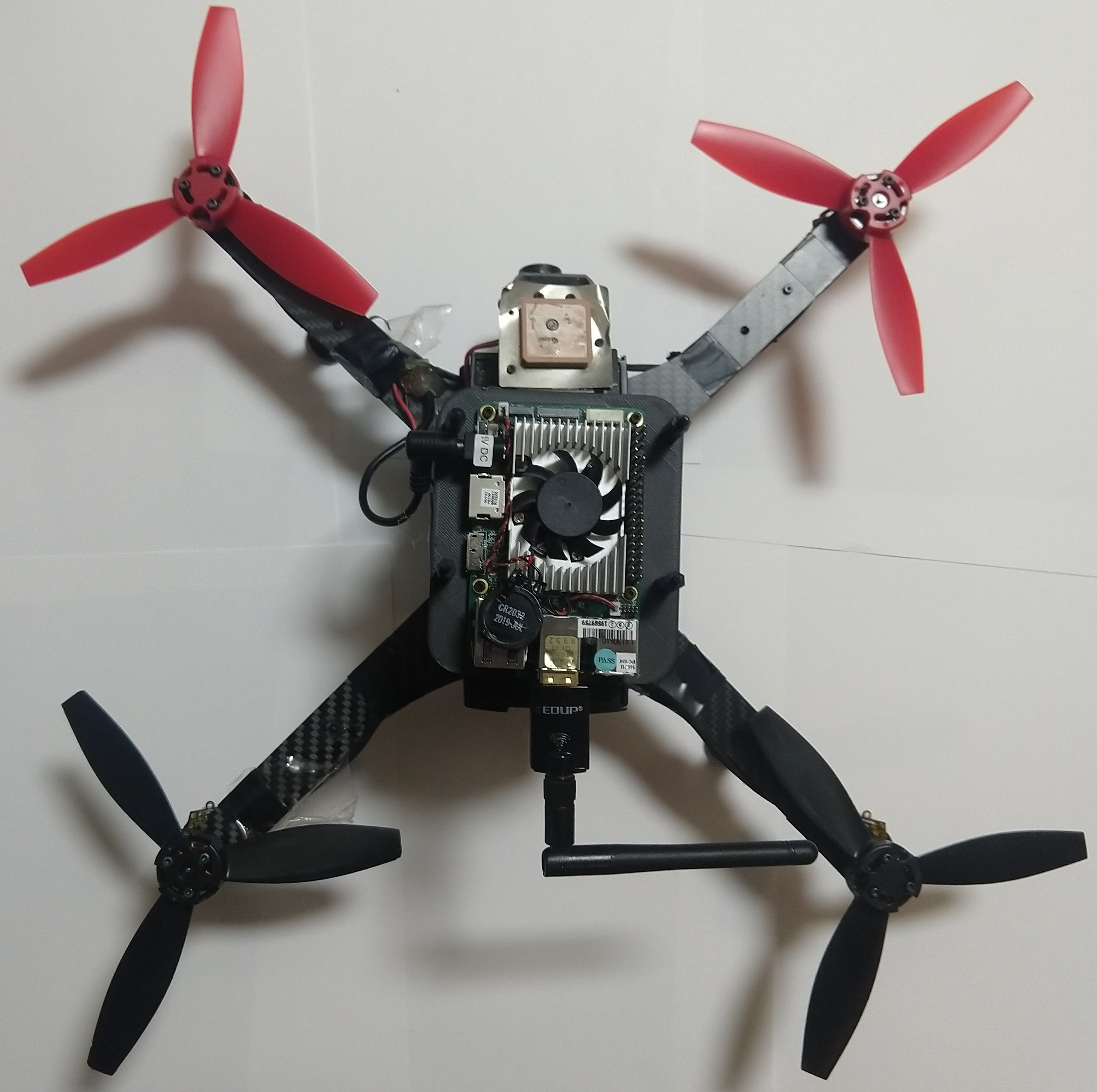}
\end{subfigure}
\begin{subfigure}
  \centering
  \includegraphics[scale=0.2]{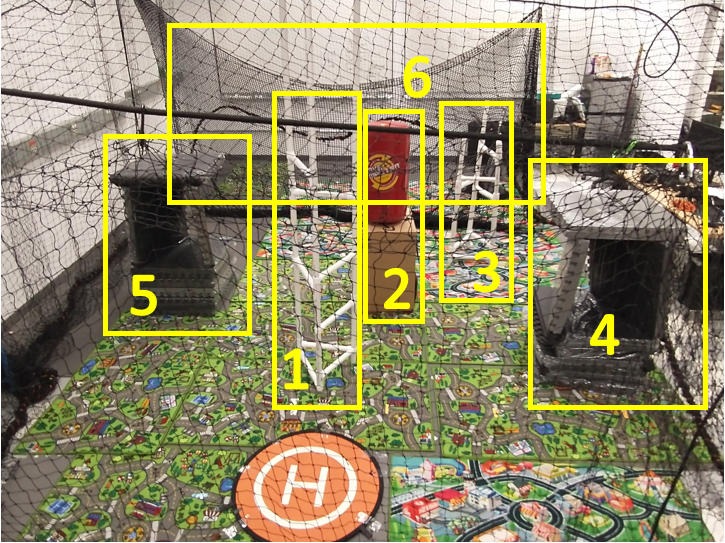}
\end{subfigure}%
\begin{subfigure}
\centering
  \includegraphics[width=4cm, height=3.1cm]{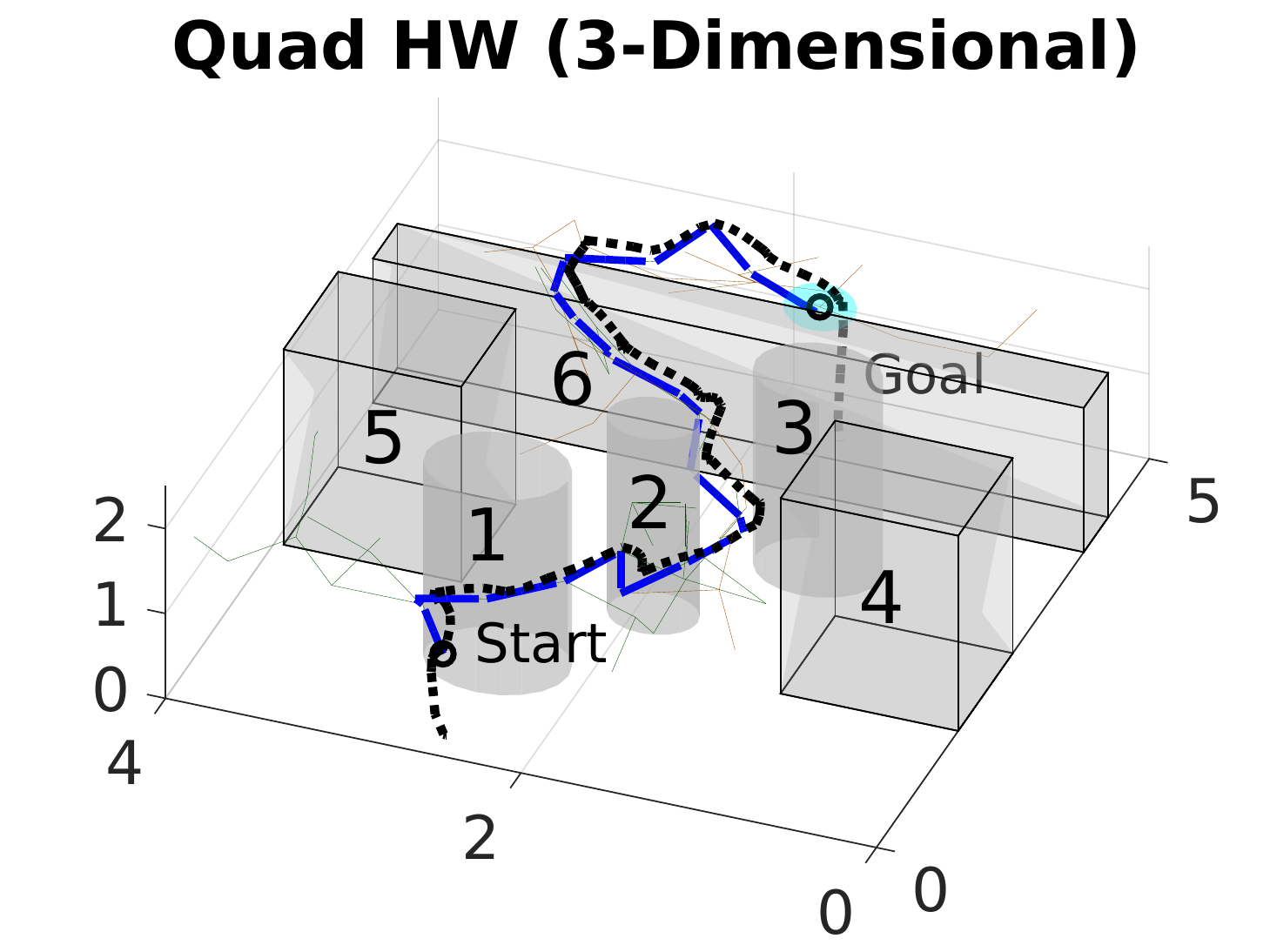}
\end{subfigure}
\begin{subfigure}
\centering
  \includegraphics[scale=0.225]{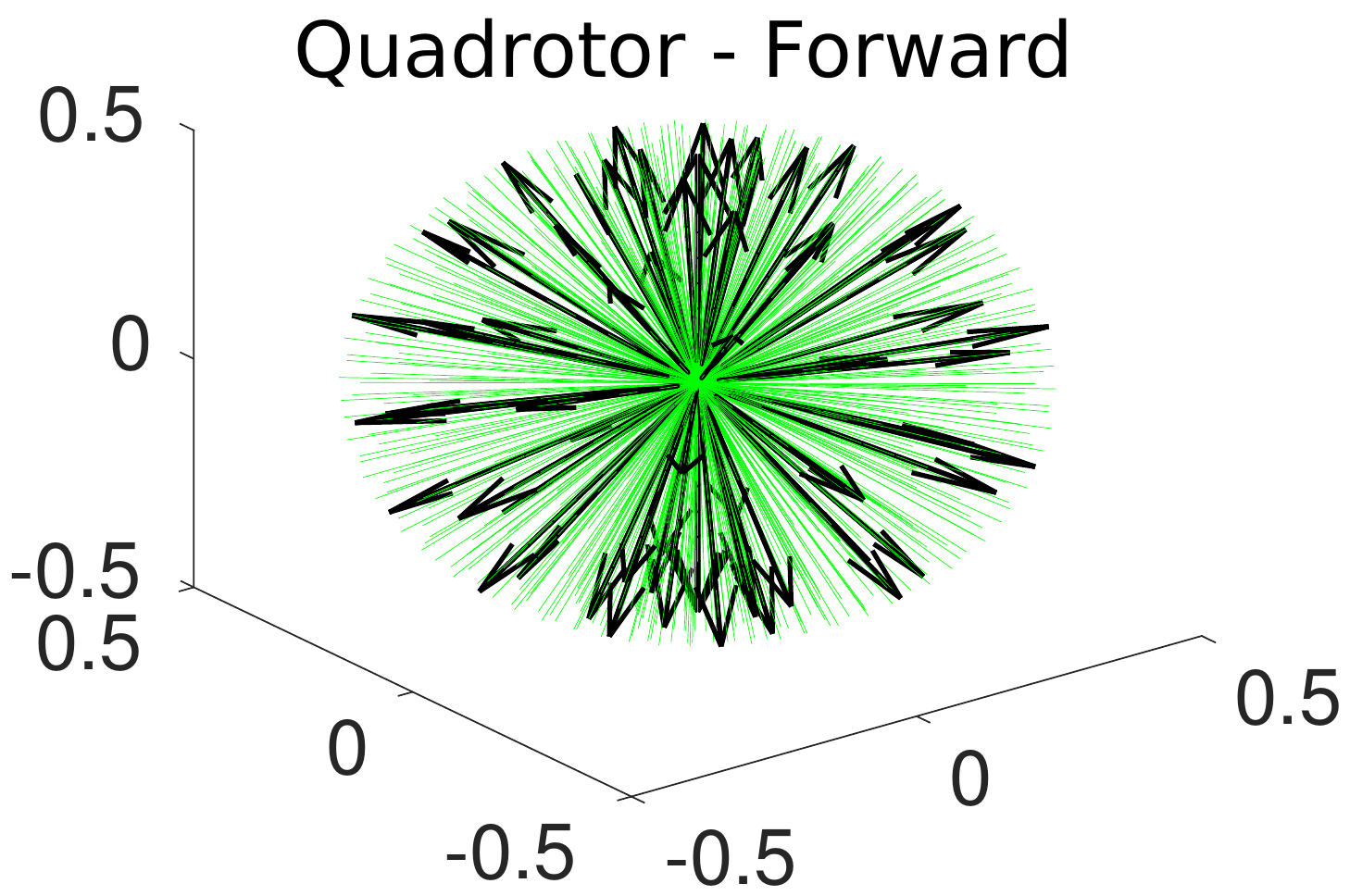}
\end{subfigure}
\begin{subfigure}
\centering
    \includegraphics[scale=0.25]{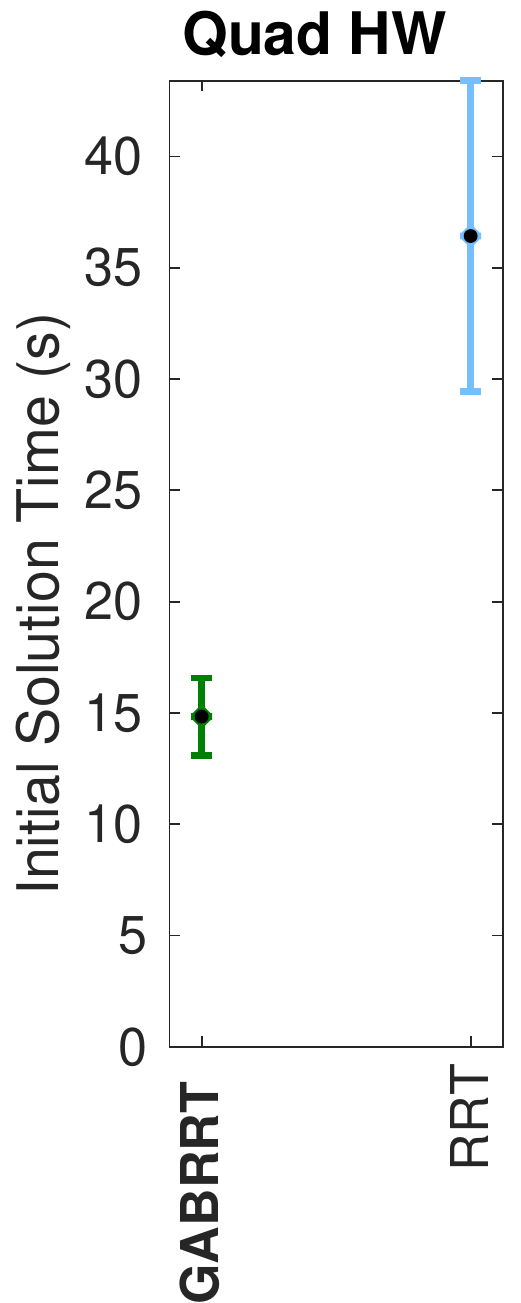}
\end{subfigure}
\begin{subfigure}
\centering
   \includegraphics[scale=0.25]{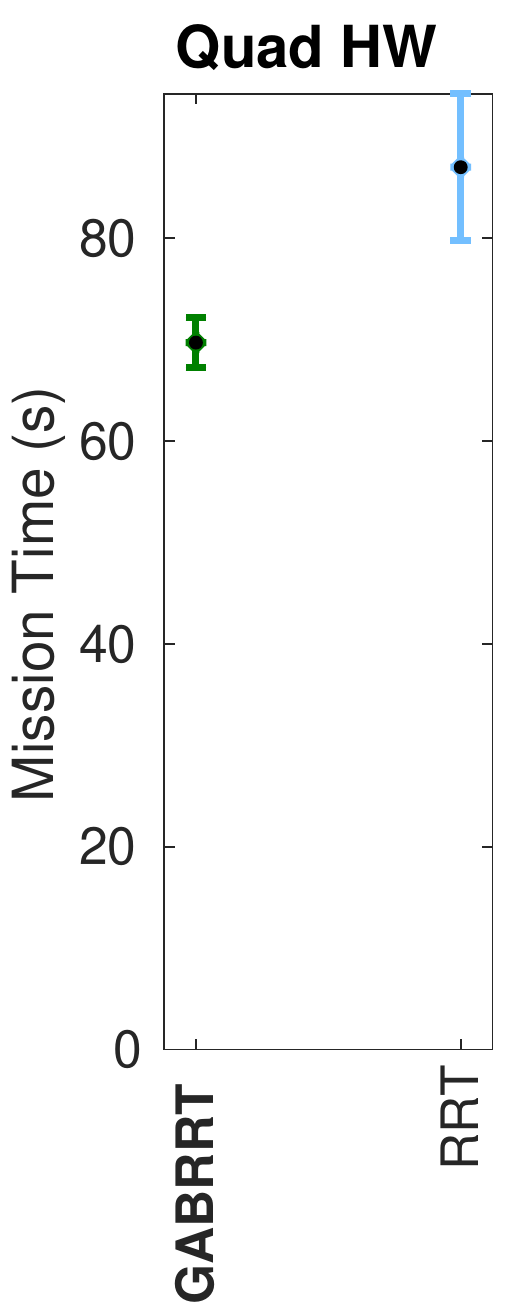}
\end{subfigure}
\begin{subfigure}
\centering
  \includegraphics[scale=0.56]{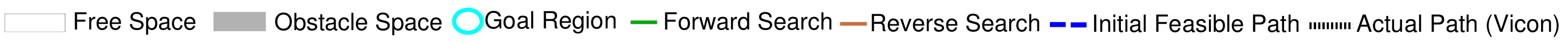}
\end{subfigure}
\caption{Left image shows the quadrotor used in our hardware experiments. The second image shows the actual view of the obstacle course and the third image shows the model view with actual path (Vicon) overlaid. The fourth image shows the maneuver library used for performing forward search. \revisionTwo{The right figures show the initial solution time and mission time for 100 trials of hardware experiments.}}
\label{fig:HardwareResult}
\end{figure*}

\begin{table}
\centering
\captionsetup{font={color=black, small}}
\captionof{table}{\revision{Experimental statistics of GBRRT for different systems}}
\renewcommand{\arraystretch}{1.8}
\scalebox{0.75}{
\revision{
\begin{tabular}[t]{|p{0.2\linewidth}|p{0.08\linewidth}|p{0.08\linewidth}|p{0.08\linewidth}|p{0.08\linewidth}|p{0.08\linewidth}|p{0.08\linewidth}|}
\hline
Informative Statistic & Uni. & Quad. & C.Pole & Tread. & C.Trail. & F.Wing \\ \hline
 $R_{r_{k}}$ & \hfil 24.95 & \hfil 13.27 & \hfil 58.41 & \hfil 61.28 & \hfil 54.96 & \hfil 58.79 \\ \hline
 $R_{edge}$ & \hfil 2.48 & \hfil 2.07 & \hfil 1.33 & \hfil 3.33 & \hfil 10.65 & \hfil 3.90 \\ \hline
$R_{x}$  & \hfil 0.07 & \hfil 0.08 & \hfil 0.11 & \hfil 0.12 & \hfil 0.2 & \hfil 0.24 \\ \hline
$R_{max}$ & \hfil 0.049 & \hfil 0.057 & \hfil 0.095 & \hfil 0.084 & \hfil 0.14 & \hfil 0.158 \\ \hline
\end{tabular}}}
\vspace{1mm}
\caption*{\revision{$R_{r_{k}}$ and $R_{edge}$ are averaged over 100 trials while $R_{x}$ and $R_{max}$ remain same across all trials.}}
\label{Tab:MetricValues}
\end{table}

\subsection{Comparison of algorithms}
We compare \revision{our proposed algorithms} with five \revision{existing} state-of-the-art modern algorithms that do not require a two-point BVP solver. This includes RRT using best-input propagation (5\% goal bias), SST\cite{li2016asymptotically} (5\% goal bias), iSST\cite{littlefield2018informed}, DIRT \cite{littlefield2018efficient}, and AO-RRT \cite{kleinbort2020refined} \cite{hauser2016asymptotically} (5\% goal bias). Even though SST, iSST, DIRT, and AO-RRT are AO algorithms which improve solution over time, we only run these algorithms until an initial solution is found because we are solving the quick initial feasible planning problem. DIRT and iSST are heuristic-based algorithms that require a heuristic during the execution of the algorithms. 
\revision{As suggested in \cite{littlefield2018informed} (the original iSST paper) we generate the heuristics for iSST and DIRT using a lower dimensional roadmap built using  PRM \cite{kavraki1996probabilistic}.} We use the same random seed to generate the roadmap for both iSST and DIRT in each trial. \revision{Because we are interested in single-query feasible planning, we factor in the time needed to generate the roadmap into the execution time that we report for these algorithms. We do not compare with AIT* \cite{aitstar} because it is a geometric planner and requires solving the two-point BVP for its operations.} 

\par \revision{We create two new bidirectional variants of DIRT called GBDIRT (Generalized Bidirectional DIRT) and GABDIRT (Generalized Asymmetric Bidirectional DIRT) to include in our comparison. These variants use our continually growing reverse search tree to generate the heuristic instead of using a roadmap. Similar to GBRRT and GABRRT, GBDIRT grows a dynamical reverse tree whereas GABDIRT uses a non-dynamical reverse tree to obtain the heuristic. Since the original DIRT algorithm (line 4 in Algorithm 1 in \cite{littlefield2018efficient}) requires a guiding heuristic during the entire execution of the algorithm, we choose the closest reverse tree node to the current forward tree to obtain the heuristic for GBDIRT and GABDIRT.}

\par Finally, we also compare our proposed algorithms with a numerical bidirectional variant of RRT (NBRRT for brevity) that uses exploration to grow the forward and reverse trees and a numerical boundary value solver to connect the two trees when they are within $\delta_{hr}$ distance of each other. This algorithm requires solving only a single two-point BVP to connect the forward and reverse trees, unlike algorithms like AIT* that require solving many two-point BVPs for their operations. We use Python 2.7 scipy library's $\mathbf{\mathtt{solve\_bvp}}$ \cite{kierzenka2001bvp} solver that uses a collocation formulation to solve the two point BVP. This solver requires the number of boundary conditions be equal to the number of equations describing the system's dynamics. We have $2n$ boundary conditions ($n$ for both the beginning and end of the connecting trajectory) but only $n$ equations describing the system's dynamics. Therefore, we obtain $n$ additional equations by differentiating the original $n$ equations with respect to time. We verify the solution returned by the solver is valid (kinodynamically feasible) by forward integrating the system using the input solution returned by the solver, and then accept the trajectory if it ends within a tolerance value of the intended point (\tableref{Tab:NBRRTSolverValues}). For the unicycle and quadrotor, we select the best trajectory from their respective trajectory library that reduces the error on the final intended point. We implement the algorithms in Python 2.7 and perform the experiments on a system with an Intel i7-7700 4-core CPU with 32GB RAM.

\subsection{\revisionTwo{Statistics chosen for GBRRT insights}}
We report various experimental statistics for GBRRT for all systems considered  (\tableref{Tab:MetricValues}). These statistics provide insights into the inner workings of GBRRT across the different systems, and include:
\begin{itemize}
  \item $R_{r_{k}}$ - Average percentage of times the forward and reverse trees are within $r_{k}$ distance away. 
  \item $R_{edge}$ - Average ratio of $\delta_{hr}$ to the average length/cost of kinodynamic edges in the forward and reverse tree.
  \item $R_{x}$ - Ratio of $\delta_{hr}$ to the length of the \singleQuote{$x$} dimension of state space common to all systems.
   \item $R_{max}$ - Ratio of $\delta_{hr}$ to the maximum distance function value possible across the state space.
\end{itemize}
We get these statistics using the same set of 100 GBRRT trials used in the comparison experiments  (\figref{fig:SoftwareResultGraphs}).

\vspace{-0.2cm}
\subsection{Hardware experiments}
\revision{We perform the hardware experiments in an} obstacle course of size 5m $\times$ 4m $\times$ \revision{2.3m} (\figref{fig:HardwareResult}) with start and end locations set at (0.2, 2.5, 1.0) and (4.5, \revision{1.7}, 1.5) respectively. The obstacle course is traversed using a modified Bebop 2 quadrotor \cite{prgflyt} in 3D space. The quadrotor is equipped with an Intel UP board (\figref{fig:HardwareResult}) having Intel Atom x5-z8350 4-core, 1.44 GHz CPU with 4GB RAM for running the algorithms. We run \revision{100} experimental trials, with each trial having a different seed. We use a Proportional-Derivative (PD) controller for position tracking and a Vicon motion capture system for state information. The obstacles in the course are modeled using 3D cylinders and rectanguloids. The quadrotor is modeled as a sphere of radius 20 cm for easy collision detection. \revision{We compare GABRRT with RRT in our experiments. We choose GABRRT because it is best-performing proposed algorithm in our quadrotor simulation experiments and RRT is the best non-proposed algorithm (\figref{fig:SoftwareResultGraphs})}. We use a maneuver library with each trajectory having a length of 0.5 m with the start and end velocities set to zero (\figref{fig:HardwareResult}). We use the same distance function $d_{X}$ as for the quadrotor simulation. We set \revision{GABRRT} parameters $\delta_{hr}$ = 1.0, $N_{B}$ = \revision{90} and $q$ = 0.8.

\section{RESULTS}
\label{sec:results}

\begin{table}
\centering
\captionsetup{font={color=black, small}}
\caption{\revision{Numerical boundary value solver results}}
\renewcommand{\arraystretch}{2.0}
\scalebox{0.9}{
\revision{
\begin{tabular}[t]{|p{0.27\linewidth}|p{0.06\linewidth}|p{0.06\linewidth}|p{0.08\linewidth}|p{0.08\linewidth}|p{0.08\linewidth}|p{0.08\linewidth}|}
\hline
Metric & Uni. & Quad. & C.Pole & Tread. & C.Trail. & F.Wing \\ \hline
 Solution Time & \hfil 15.18 & \hfil 35.93 & \hfil 171.29 & \hfil 171.33 & \hfil 528.91 & \hfil 796.32 \\ \hline
Total Success Rate & \hfil 100 & \hfil 100 & \hfil 91 & \hfil 99 & \hfil 82 & \hfil 23  \\ \hline
Conn. Success Rate & \hfil 67 & \hfil 1 & \hfil 58 & \hfil 92 & \hfil 70 & \hfil 0  \\ \hline
Tolerance & \hfil 0.01 & \hfil 0.01 & \hfil 0.75 & \hfil 0.2 & \hfil 0.3 & \hfil 0.75  \\ \hline
\end{tabular}}}
\label{Tab:NBRRTSolverValues}
\vspace{1mm}
\caption*{\revision{Results averaged over 100 trials. The connected success rate refers to the percentage of trials the solution was found using the numerical boundary value solver. It is less than the total success rate because there were some trials where the forward tree reached the goal region without connection to the reverse tree (attempts to connect using 2-point BVP solver were unsuccessful).}}
\vspace{-2mm}
\end{table}

\revision{This section provides the simulation experiment results (\subSecref{sub:simResults}), NBRRT comparison results (\subSecref{sub:NBRRTResults}), GBRRT experimental statistics interpretation (\subSecref{sub:insightResults}), the hardware experiment results (\subSecref{sub:HardwareResults}) and finally the limitations of our algorithms.}

\subsection{\revision{Simulation experiment results}}
\label{sub:simResults}
\revision{Example feasible paths generated by GBRRT and GABRRT for the systems we consider are shown in \figref{fig:SoftwareSimulationSnapshot}. The plots for the solution success rate percentage and initial solution time metrics are presented in \figref{fig:SoftwareResultGraphs}.
Considering these metrics, GBRRT performs best for the \revisionTwo{unicycle}, cart-pole, threaded and, car-trailer \revisionTwo{(success rate metric)} systems, GABRRT performs the best for the quadrotor systems, and, GBDIRT performs the best for the fixed-wing system.}
\par \revision{GABRRT uses an inexpensive non-dynamical reverse tree to generate the cost-to-go heuristic. \revisionTwo{For this reason, it performs better for systems having `straight-line' like trajectories (quadrotor)}. On the other hand, although iSST and DIRT also use heuristics from a non-dynamical PRM roadmap, they take longer to generate a feasible solution. One reason for this is that roadmaps usually takes longer to construct than trees containing the same number of nodes (due to the fact that they typically involve more edges per node).
}
\par \revision{GBRRT uses the cost-to-go heuristic from a reverse tree with dynamics. This makes GBRRT perform better for higher-dimensional systems (cart-pole, threaded, car-trailer, and fixed-wing) that have velocity constraints in addition to position constraints at their goal states, and the distance functions have higher-order dynamics.  We believe this also causes GBRRT to perform better than DIRT and iSST for these systems (DIRT and iSST use a heuristic generated from a lower-dimensional non-dynamical roadmap). In such cases we also observe that if the goal region is made large for these systems, then GABRRT performs on par or better than GBRRT. We believe this happens because the size of the state space's positional dimensions is relatively large compared to the size of the velocity dimensions.}

\par \revision{GBDIRT performs the best for the fixed-wing system and performs better than DIRT for majority of the systems. We believe this demonstrates that, at least in the case of single query feasible planning, trees are usually better suited to the task of generating a cost-to-go heuristic than roadmaps. We note that this parallels the broader use of trees (vs.\ roadmaps) for single query planning more generally. }

\subsection{\revision{NBRRT comparison results}}
\label{sub:NBRRTResults}
\par \revision{NBRRT (\tableref{Tab:NBRRTSolverValues}) does not perform the best in terms of the initial solution time metric among the algorithms tested. The results for the solution success rate were mixed, with NBRRT performing the best for the threaded system and the worst for the fixed-wing system. We found the connection success rate is zero for all systems if a low enough tolerance value is used. Although NBRRT's performance is mixed, we state that if the problem can be solved efficiently using numerical methods or has special symmetries such that solving the two-point BVP is trivial, then it is advisable to use one of the pre-existing numerical bidirectional algorithms. However, if numerical methods are not easy to apply, then our approach provides a general solution method that can be used.}

\vspace{-2mm}
\subsection{\revision{Interpretation of GBRRT experimental statistics}}
\label{sub:insightResults}
\par Table II presents the experimental statistics for GBRRT. The interpretation of the various statistic results follows:
\par $R_{r_{k}}$ - The percent overlap between the forward and reverse trees is lower for lower-dimensional systems (unicycle, quadrotor). We observe that once the trees meet each other, the reverse tree heuristic directs the forward tree towards the goal. However, the percent overlap is much higher for higher-dimensional systems (sometimes called ``curse of dimensionality'' \cite{koppen2000curse}). This causes the feasible solution to be found much later after the initial encounter of the forward and reverse trees.

\par \revision{$R_{edge}$ - This value is $>1$ for all systems (because $\delta_{hr}$ chosen is always greater than the average length/cost of the kinodynamic edge). This is necessary to make the forward tree consider the cost-to-go heuristic from reverse tree nodes sufficiently far from its current location to make quick progress towards the goal.}

\par \revision{$R_{x}$ and $R_{max}$ - These values are less than one for all the systems because the $\delta_{hr}$ value chosen is much less than the diameter  of the systems' state space. These values are higher for the higher-dimensional systems (car-trailer, fixed-wing) because their state-spaces have smaller positional dimensions compared to the lower-dimensional systems (\figref{fig:SoftwareSimulationSnapshot}).}

\vspace{-2mm}
\subsection{\revision{Hardware experiment results}}
\label{sub:HardwareResults}
\par \revision{The initial feasible path generated by GABRRT in a hardware experiment trial with the actual path (obtained using Vicon motion capture system) and the results for 100 trials of hardware experiments are shown in \figref{fig:HardwareResult}. GABRRT, on average, performs better than RRT with respect to initial solution time and total mission time (solution time + flight time). A Student's T-test \cite{boslaugh2012statistics} and Kolmogorov Smirnov (KS) \cite{massey1951kolmogorov} tests verify that the initial solution times for GABRRT and RRT are significantly different. However, we did not observe statistical significance for the mission-time (KS test). This is not surprising due to the fact that flight time is proportional to the path length, and GABRRT does not optimize for path length since it is a feasible motion planning algorithm.}

\vspace{-2mm}
\subsection{Limitations of \revision{our proposed algorithms}}
\revision{Our proposed algorithms} are parameter-based, and their performance depends on the appropriate selection of parameters, including: heuristic radius $\delta_{hr}$, exploitation ratio function $\mathcal{P}$ and number of best-input trajectories $N_{B}$. The effects of using different $\delta_{hr}$, $\mathcal{P}$ and $N_{B}$ values are discussed in \subSecref{sub:SelectionParameters}.
We provide guidelines for parameters selection as follows:

\subsubsection{Number of Maneuvers ($N_{B}$)}
  The solution time is largely unaffected for systems using maneuver libraries when we increase $N_{B}$ within an extensive range (\figref{fig:parametersFixed}). However, if we choose $N_{B}$ to be a small value, the solution time does increase. Hence for systems using maneuver libraries, we recommend a high value dependent on the number of maneuvers in the library. For systems using online integration, a value between 5-10 worked well in our experiments. Higher-dimensional systems generally performed well in our experiments when $N_{B}$ was lower.
\subsubsection{Exploitation ratio ($q$)} We recommend using a value between 0.5-1.0 to prioritize exploitation over exploration to find the initial solution quickly. For lower-dimensional systems, a value between 0.8-0.95 worked well in our experiments. For higher-dimensional systems, the range 0.5-0.7 worked well.
\subsubsection{Heuristic radius ($\delta_{hr}$)} The heuristic radius is dependent on the system and the environment used. A general rule of thumb for $\delta_{hr}$ is that it should be chosen less than the ``average size" of obstacles in the environment (assuming positional coordinates used in $d_{X}$) so that heuristic from reverse tree nodes blocked by obstacles are not considered. For example, $\delta_{hr} > 10$ may not work well for the quadrotor environment (\figref{fig:SoftwareSimulationSnapshot}) since the large $\delta_{hr}$ may cause our algorithms to consider the heuristic from reverse tree nodes blocked by the walls (whose thickness is 10). In addition, $\delta_{hr}$ should be chosen, so it is greater than the average size of the kinodynamic edge (\tableref{Tab:MetricValues}) so that it considers reverse tree nodes sufficiently far away from its current point to make progress towards the goal. In our experiments, a value between 3-10 worked well. 
Finally, if none of these guidelines work, the last resort is to perform a simple parameter sweep for the system in question or a similar problem.

\par GBRRT \revision{(GABRRT)}, like RRT \cite{lavalle2001randomized}\cite{kleinbort2018probabilistic}, is a feasible motion planning algorithm. As such, it is not designed to be asymptotically optimal. That said, it is of potential relevance to any-time planning with asymptotic algorithms because one potential application of GBRRT \revision{(GABRRT)} is that it can be used to provide an initial solution and then be improved using another method. \revision{This is because feasible motion planning algorithms tend to find an initial solution more quickly than AO motion planning algorithms. Thus if an any-time solution is desired, it is possible to start planning with a feasible algorithm until the first solution is found and then switch to an AO algorithm. In practice, both algorithms often use the same underlying graph data structure and subroutines. GBRRT \revision{(GABRRT)} could potentially be made asymptotically optimal by replacing the node selection process with nearest neighbors to random selection \cite{karaman2011sampling}\cite{kleinbort2018probabilistic} at the expense of reduced convergence rates. In addition, the cost function used for inserting the node in the priority queue $\mathbf{Q}$ has to be updated to use $g(x_{for}) + d_{X}(x_{for}, x_{closest}) + h(x_{closest})$ (see \figref{fig:GBRRPushBoth}) instead of $d_{X}(x_{for}, x_{closest}) + h(x_{closest})$ to prevent the algorithm from only selecting nodes close to the goal for expansion and thus consider the total cost from start to goal.} 
\par Applying GBRRT to systems involving interactions with the environment like rigid body systems having frictional contacts is challenging. This is because computing the reverse propagation primitive is hard for these systems since we cannot uniquely solve the backward simulation problem \cite{twigg2008backward}. A partial solution based on Linear Complementary Problem (LCP) formulation is provided in \cite{twigg2008backward} to alleviate the issue, but future research challenges remain.

\section{CONCLUSION}
\label{sec:conclusion}
We present Generalized Bidirectional RRT (GBRRT), a new bidirectional \revision{single-query} sampling-based motion planning algorithm that does not require a two-point BVP solver. GBRRT produces a continuous motion plan for solving the initial feasible motion planning problem. Instead of connecting the forward and reverse trees,  the forward tree to uses a cost-to-goal heuristic provided by the reverse search tree to quickly find an initial feasible solution. The cost-to-goal heuristic calculated on the fly applies to the specific problem instance and alleviates the user from coming up with a heuristic. 
\par \revision{We provide a variant of GBRRT called Generalized Asymmetric Bidirectional RRT (GABRRT) that uses a computationally inexpensive non-dynamical reverse search to generate the cost-to-go heuristic. In experimental trials GABRRT performs better than GBRRT in scenarios without  higher-order dynamics. We also combine the idea used in GBRRT/GABRRT (reverse tree provided heuristic) to an existing state-of-the-art unidirectional algorithm called DIRT \cite{littlefield2018efficient} to create Generalized Bidirectional DIRT (GBDIRT) and Generalized Asymmetric Bidirectional (GABDIRT). GBDIRT outperforms DIRT in the majority of the systems tested.}

\par We prove that GBRRT and GABRRT are probabilistically complete.
Finally, we run multiple simulation experiments using systems of varying dimensions and hardware experiments using a quadrotor. These show that our algorithms quickly find a feasible solution in many different scenarios.

\par \revision{Considering solution success rate and average initial solution time, GBRRT performs the best for the unicycle, cart-pole, threaded and car-trailer systems \revisionTwo{(solution success rate)} and GBDIRT for the fixed-wing system. This is because these systems generally consider higher-order dynamics in their planning and distance functions. GABRRT performs the best for the quadrotor system where higher-order dynamics \revisionTwo{(`straight-line' like trajectories)} are not considered. GABRRT also outperforms RRT in the hardware experiments we run in initial solution time and mission time metrics. These results shows the power of performing a bidirectional search and using an on-the-fly calculated problem-specific heuristic in \revision{improving the performance of} finding an initial solution.}
\par A potential avenue for future research is developing an AO version of \revision{our proposed algorithms}. This version should be able to provide good convergence rates while still maintaining asymptotically optimality. Other possible directions include testing the performance of GBRRT \revision{and GABRRT} using different definitions of \revision{cost and distance functions}.

\vspace{-1mm}
\section*{Acknowledgments}
We thank Chahat Deep Singh and Nitin Sanket from the Perception and Robotics Group at the University of Maryland (UMD) for providing the modified bebop 2 quadrotor used in our experiments, and Mohamed Khalid M. Jaffar in the Motion and Teaming Lab at UMD for his valuable feedback during manuscript preparation. We genuinely appreciate the reviewers' and editor's feedback and comments, many of which helped us to improve the structure and performance of the proposed algorithms, as well as helping us to substantially improve the manuscript from its preliminary version. This research was supported by the Minta Martin Research Fellowship Fund, the Clark Doctoral Fellowship, and U.S.\ Office of Naval Research (ONR) award \#N00014-20-1-2712.

\appendix

\subsection{\revision{Initial solution cost}}

\label{apdx:InitialSolutionCost}
\begin{figure}[ht]
\begin{minipage}[c]{0.5\textwidth}
\captionsetup{font={color=black, small}}
\begin{minipage}[c]{0.19\textwidth}
\includegraphics[scale=\compareResultCostScale]{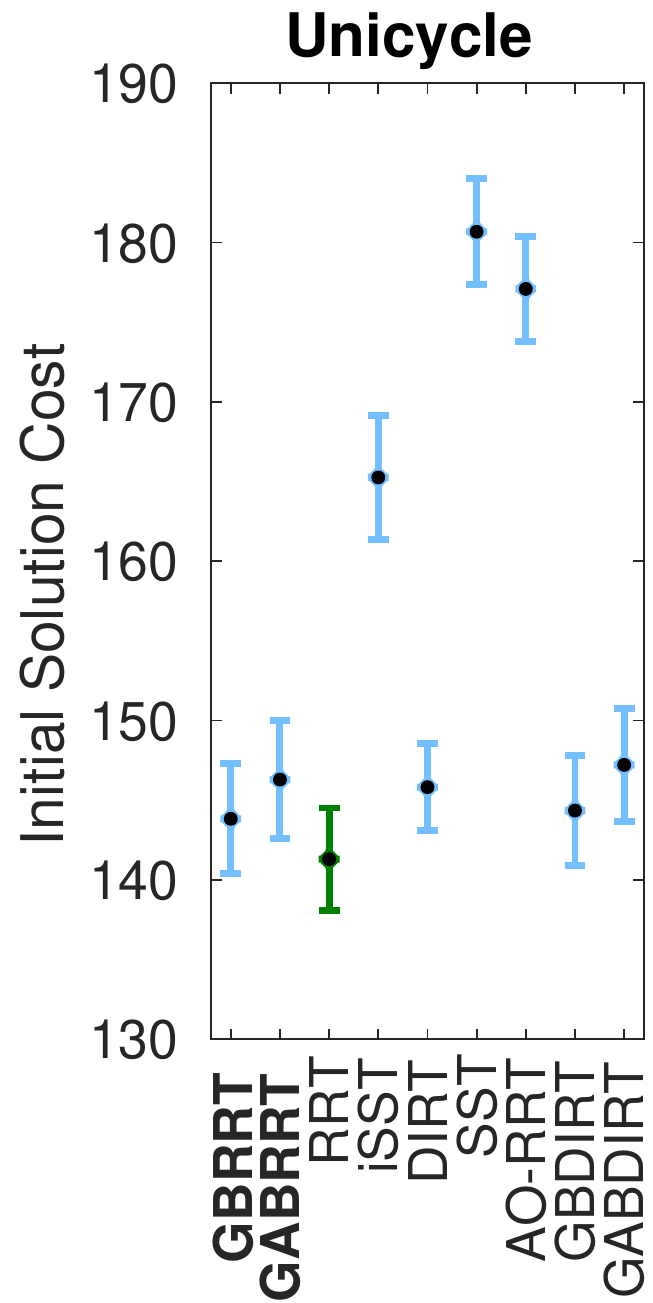}
\end{minipage}
\begin{minipage}[c]{0.19\textwidth}
\includegraphics[scale=\compareResultCostScale]{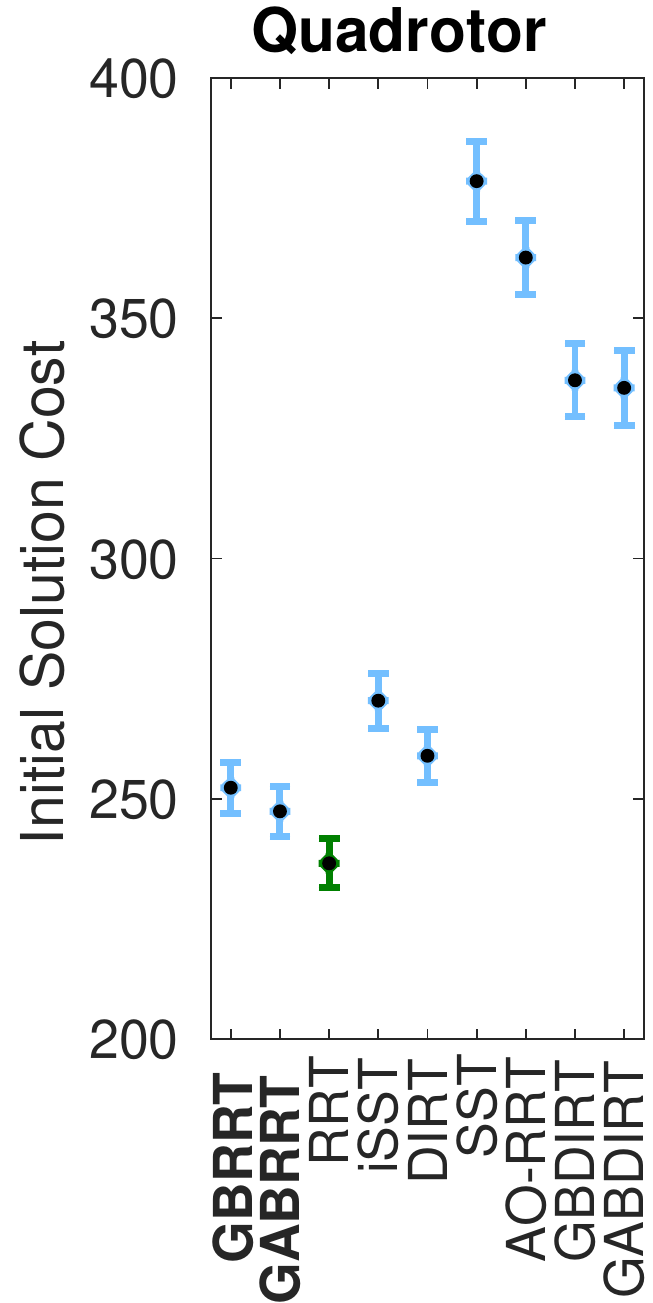}
\end{minipage}
\begin{minipage}[c]{0.19\textwidth}
\includegraphics[scale=\compareResultCostScale]{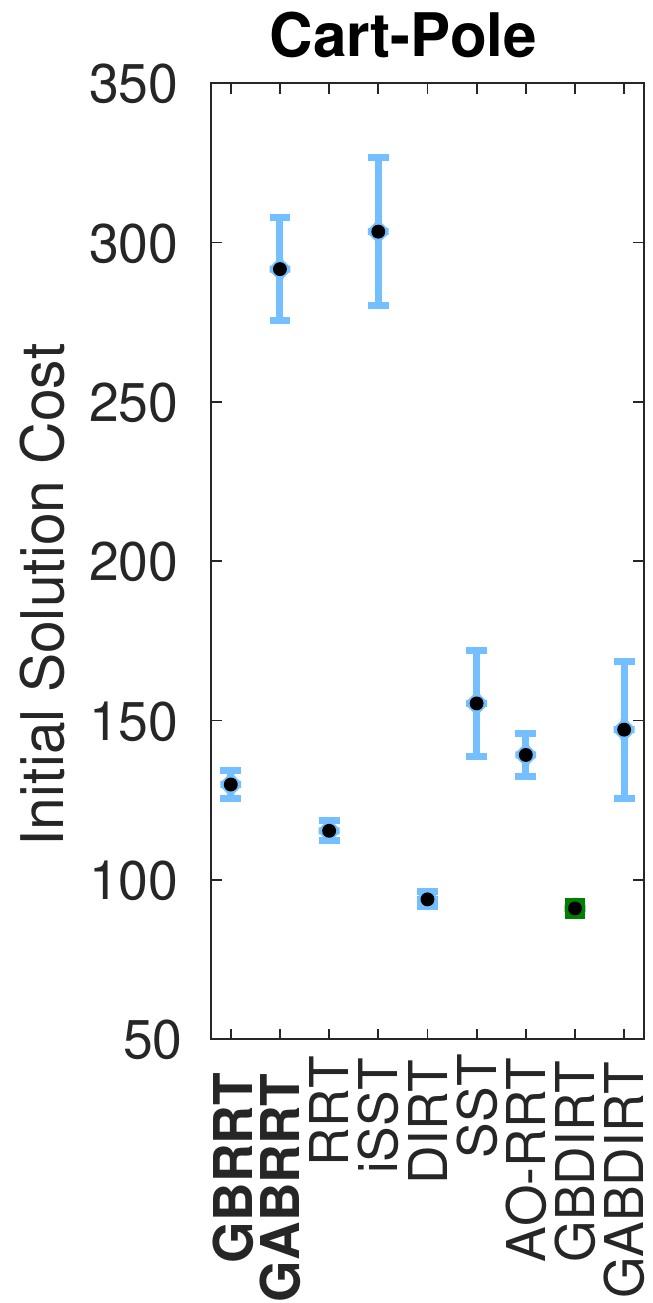}
\end{minipage}
\begin{minipage}[c]{0.19\textwidth}
\includegraphics[scale=\compareResultCostScale]{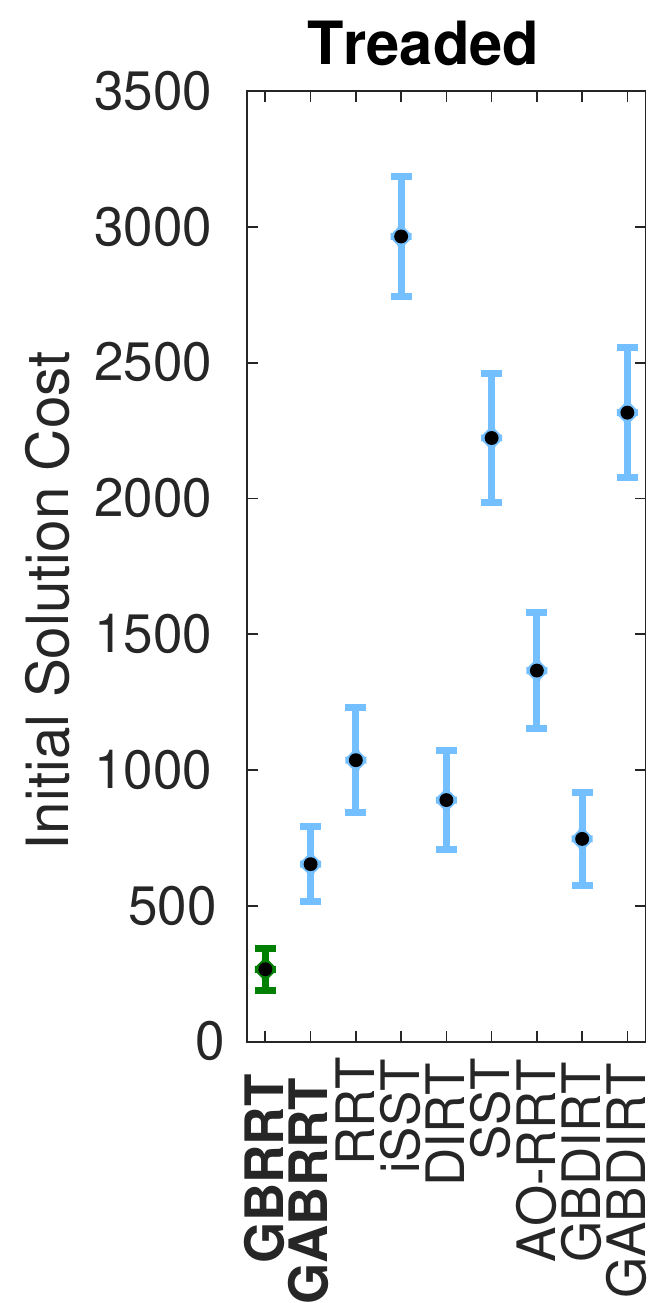}
\end{minipage}
\begin{minipage}[c]{0.19\textwidth}
\includegraphics[scale=\compareResultCostScale]{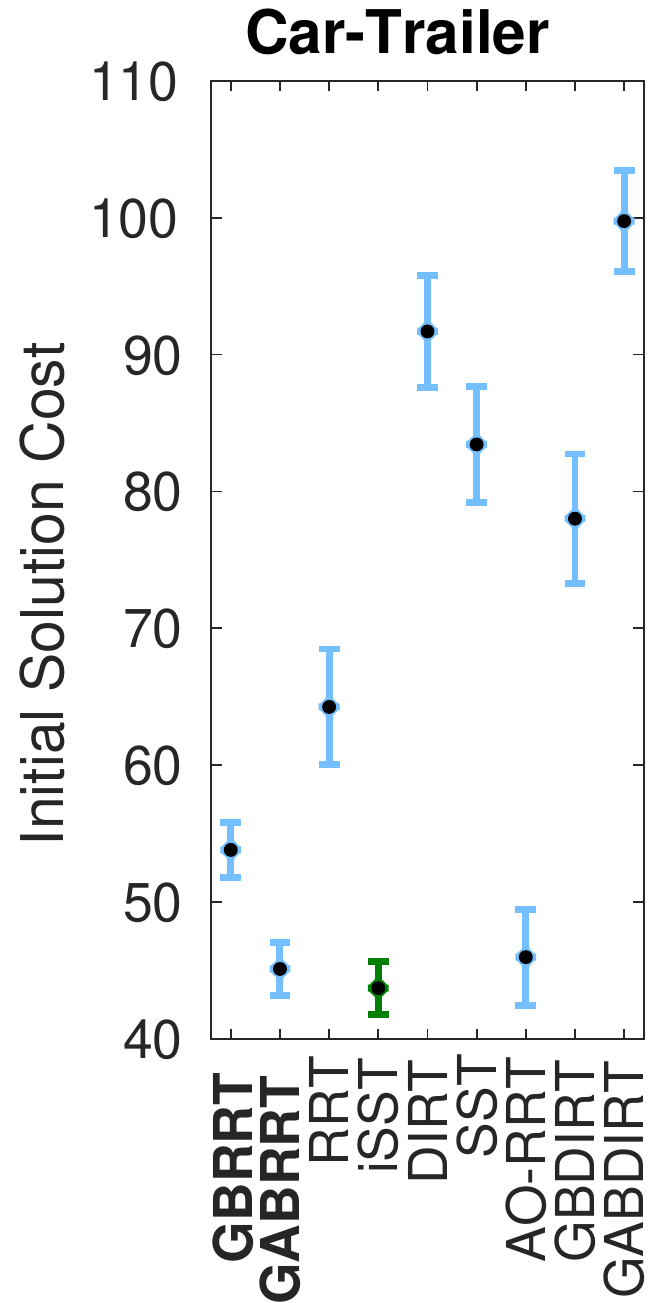}
\end{minipage}
\end{minipage}
\begin{minipage}[c]{0.5\textwidth}
\begin{minipage}[c]{0.19\textwidth}
\includegraphics[scale=\compareResultCostScale]{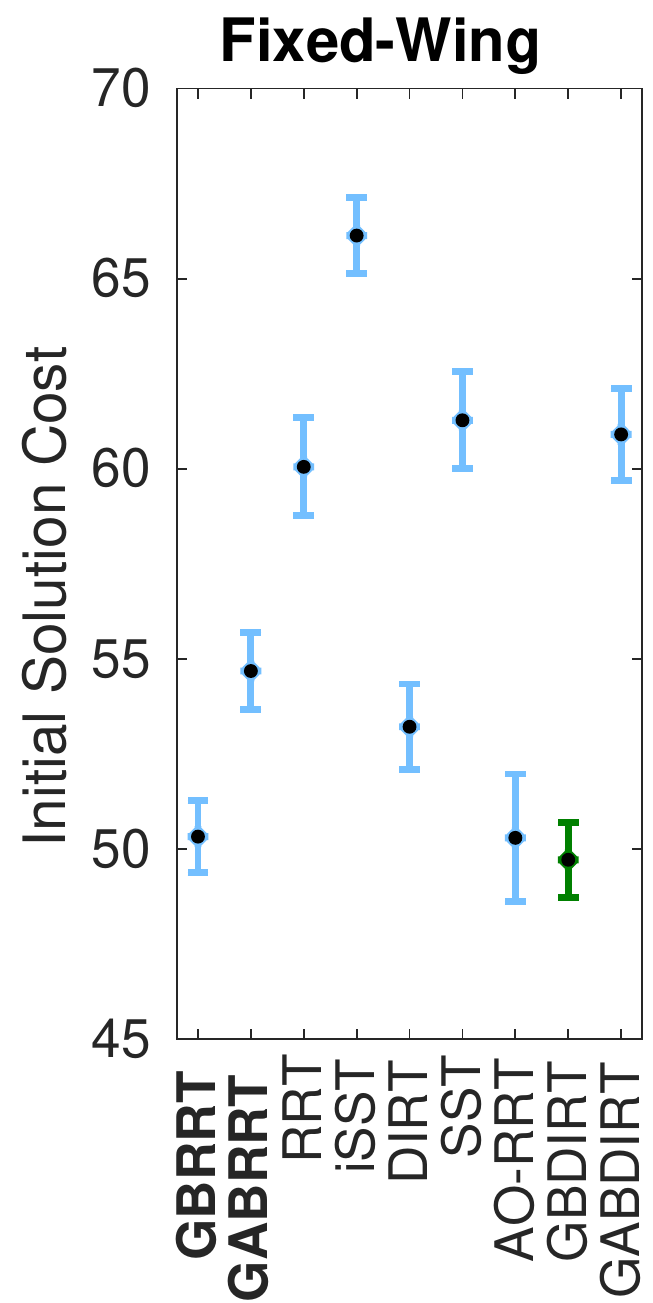}
\end{minipage}
\begin{minipage}[c]{0.15\textwidth}
\includegraphics[scale=\compareResultCostScale]{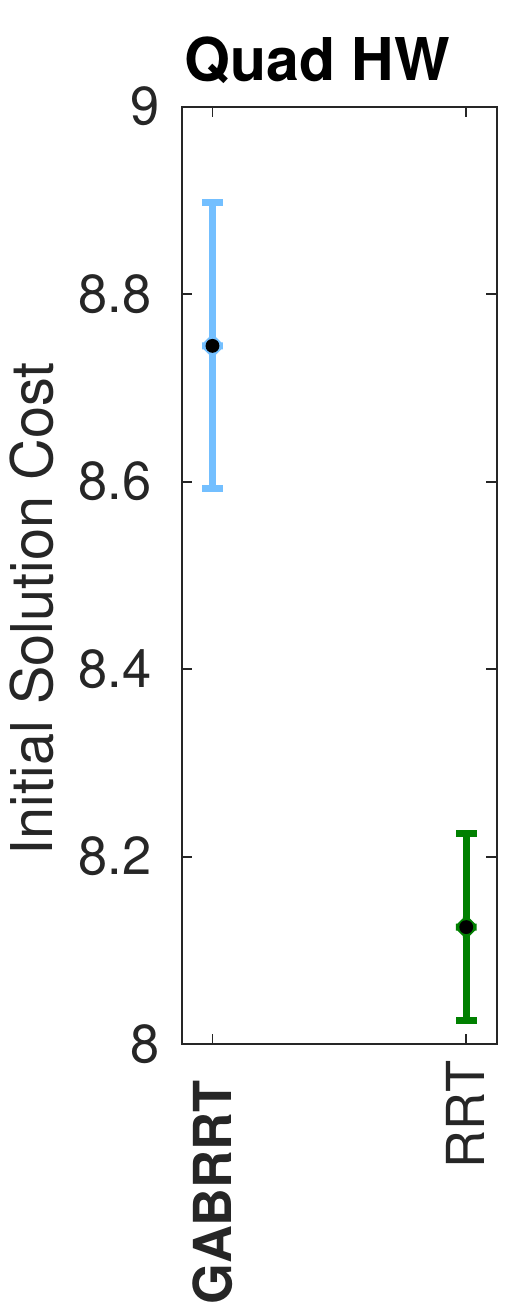}
\end{minipage}
\begin{minipage}[c]{0.6\textwidth}
\caption{Initial solution cost of algorithms for different systems. These graphs are shown only for completeness of analysis as GBRRT does not minimize the solution cost. We use the maximum cost incurred from all algorithms for a specific system as the cost value for unsuccessful trials.}
\label{fig:InitialSolutionCost}
\end{minipage}
\end{minipage}
\vspace{-5mm}
\end{figure}

\subsection{Comparison with variations of GBRRT}
We compare GBRRT with algorithms created from modifications to its operations to show the effect of different operations on the performance metrics (\figref{fig:modificationComparisons}). The algorithms include GBRRT-NF (No fast exploration (Alg. 5), replaced by random exploration),  GBRRT-NU (No priority Q-update, (Alg. 3)), GBRRT-NE (No Exploitation, (Algs. 2, 3, 4)) and GBRRT-RU (Range-based update instead of nearest update in Alg. 3). Fig. 16 shows that baseline GBRRT performs at least as good, on average and in most cases, as these other variations.

\begin{figure}[h!]
\begin{minipage}[c]{0.16\textwidth}
\begin{xy}
\xyimport(100, 100){\includegraphics[width=1.3cm, height=3.1cm]{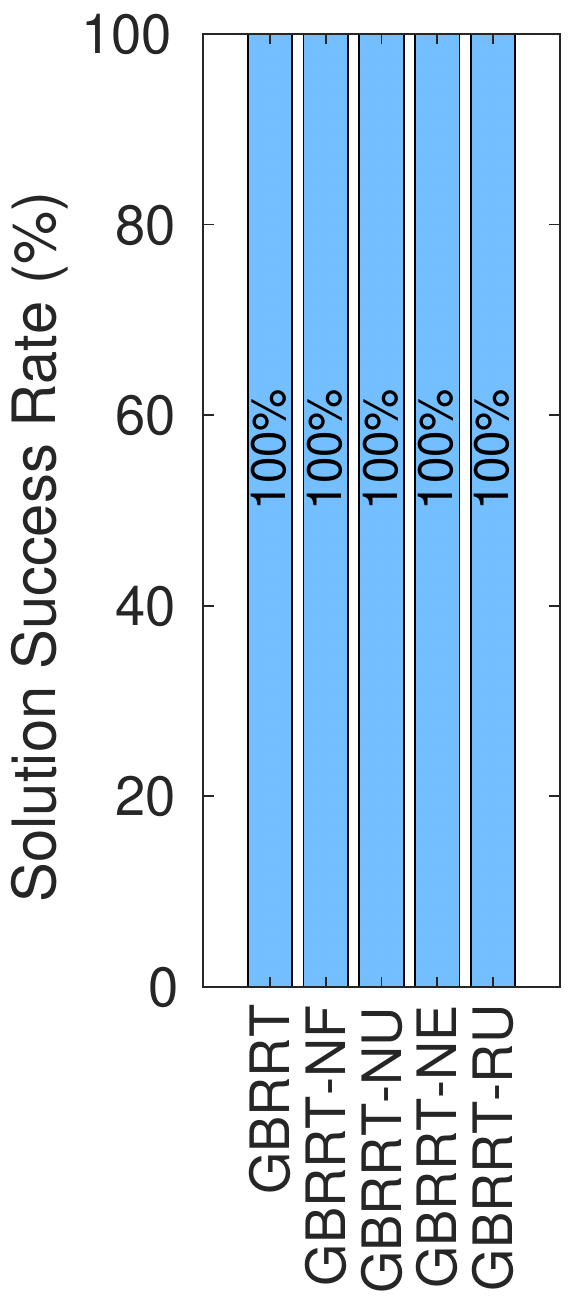}}
,(60, 105)*{\footnotesize \text{Unicycle}}
\end{xy}
\end{minipage}
\begin{minipage}[c]{0.16\textwidth}
\begin{xy}
\xyimport(100, 100){\includegraphics[width=1.0cm, height=3.1cm]{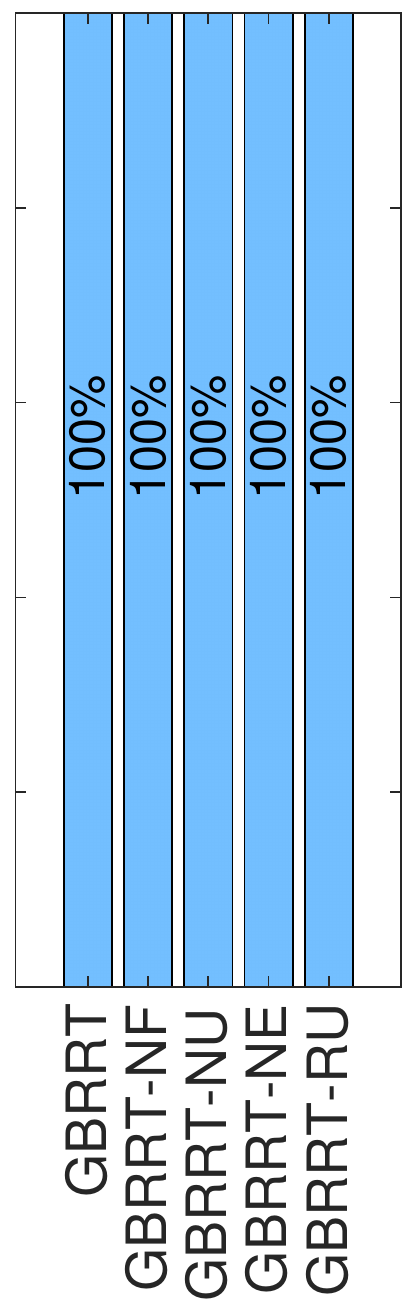}}
,(55, 105)*{\footnotesize \text{Quadrotor}}
\end{xy}
\end{minipage}
\begin{minipage}[c]{0.16\textwidth}
\begin{xy}
\xyimport(100, 100){\includegraphics[width=1.1cm, height=3.1cm]{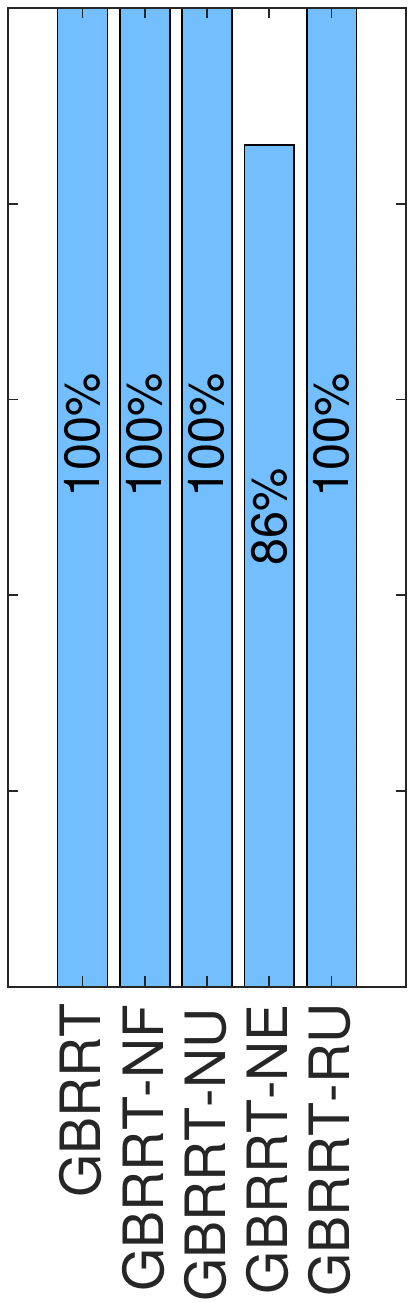}}
,(55, 105)*{\footnotesize \text{Cart-Pole}}
\end{xy}
\end{minipage}
\begin{minipage}[c]{0.16\textwidth}
\begin{xy}
\xyimport(100, 100){\includegraphics[width=1.1cm, height=3.1cm]{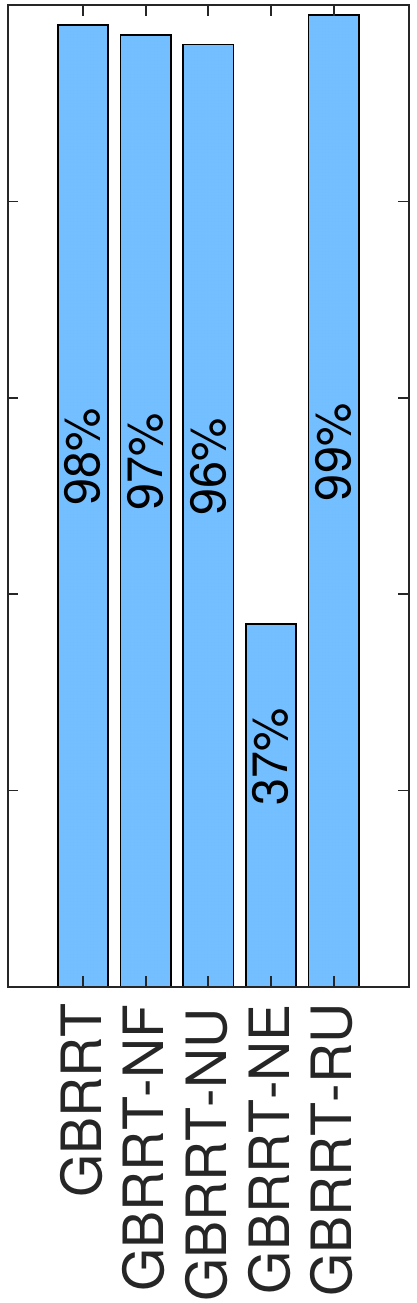}}
,(55, 105)*{\footnotesize \text{Treaded}}
\end{xy}
\end{minipage}
\begin{minipage}[c]{0.16\textwidth}
\begin{xy}
\xyimport(100, 100){\includegraphics[width=1.1cm, height=3.1cm]{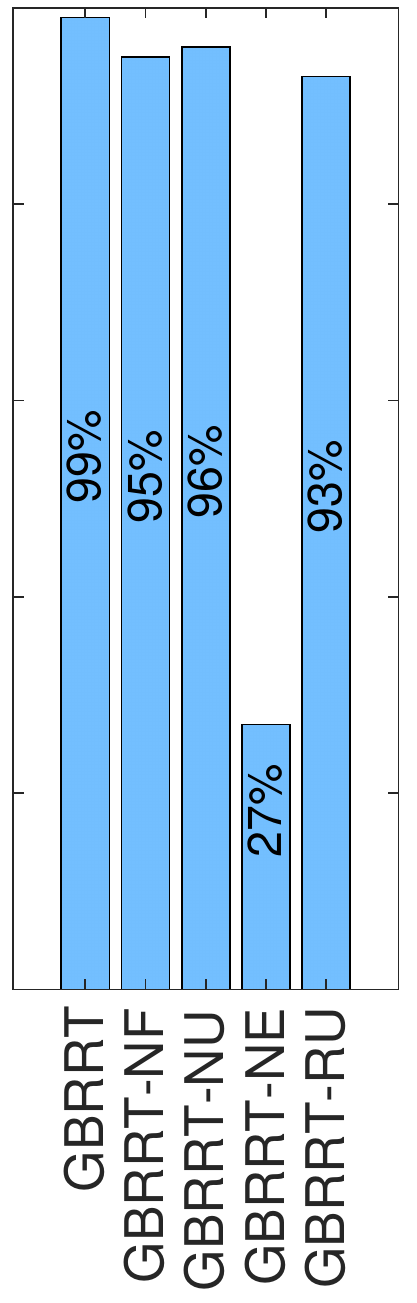}}
,(55, 105)*{\footnotesize \text{Car-Trailer}}
\end{xy}
\end{minipage}
\begin{minipage}[c]{0.16\textwidth}
\begin{xy}
\xyimport(100, 100){\includegraphics[width=1.1cm, height=3.1cm]{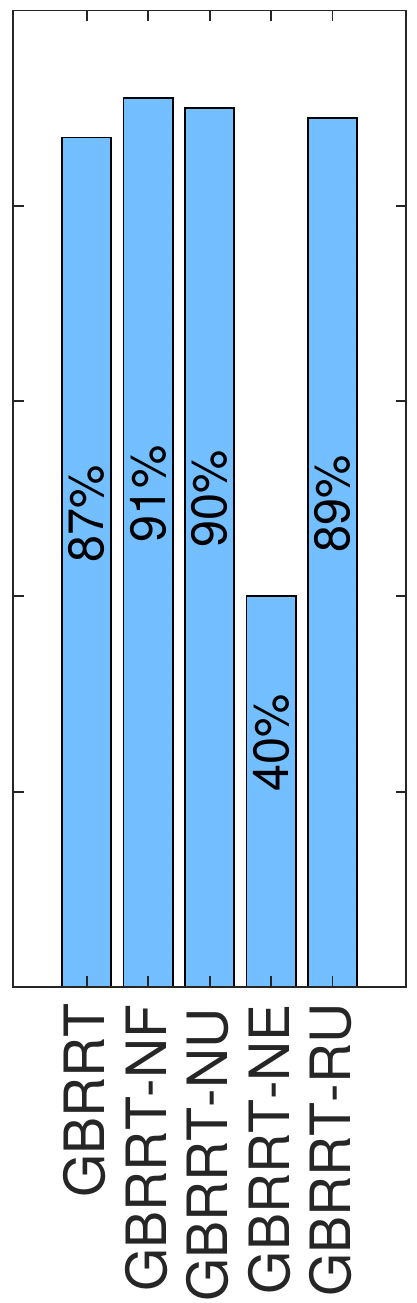}}
,(55, 105)*{\footnotesize \text{Fixed-Wing}}
\end{xy}
\end{minipage}
\vspace{4mm}
\begin{minipage}[c]{0.16\textwidth}
\begin{xy}
\xyimport(100, 100){\includegraphics[width=1.3cm, height=3.2cm]{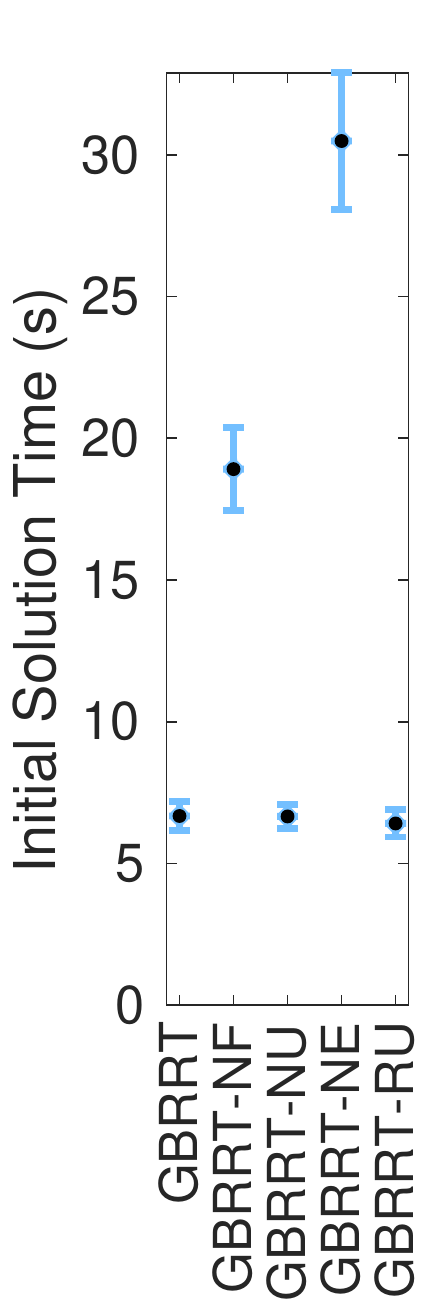}}
,(60, 102)*{\footnotesize \text{Unicycle}}
\end{xy}
\end{minipage}
\begin{minipage}[c]{0.16\textwidth}
\begin{xy}
\xyimport(100, 100){\includegraphics[width=1.35cm, height=3.2cm]{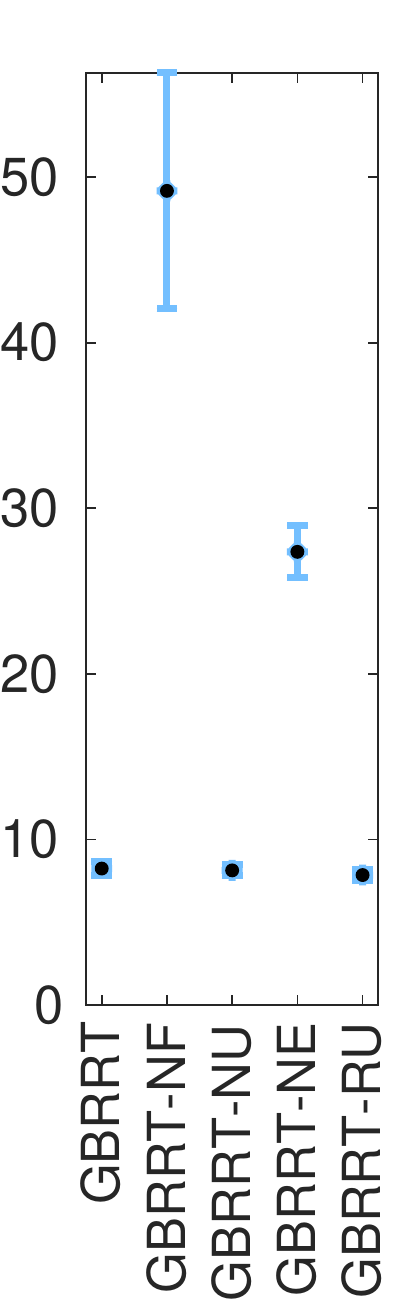}}
,(60, 102)*{\footnotesize \text{Quadrotor}}
\end{xy}
\end{minipage}
\begin{minipage}[c]{0.16\textwidth}
\begin{xy}
\xyimport(100, 100){\includegraphics[width=1.35cm, height=3.2cm]{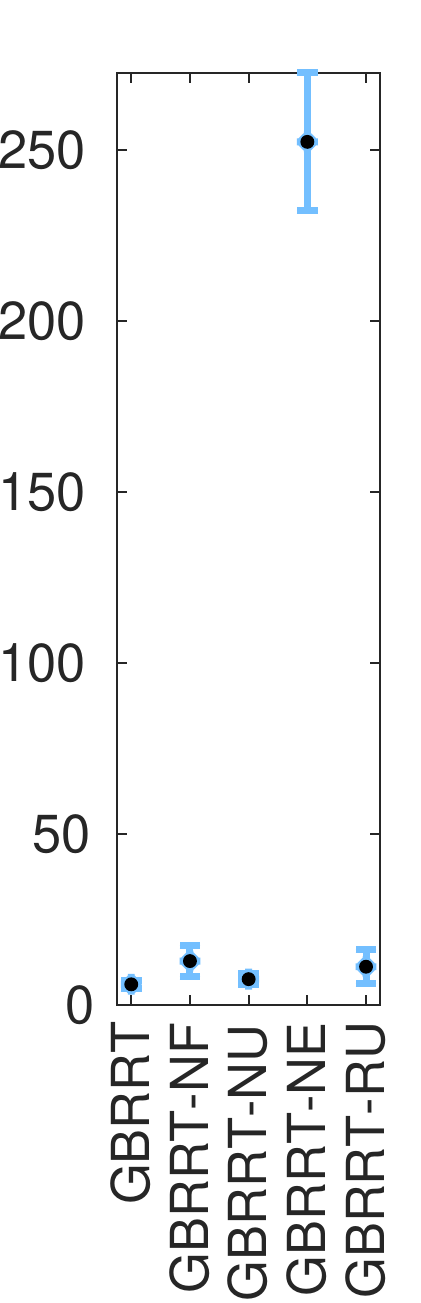}}
,(60, 102)*{\footnotesize \text{Cart-Pole}}
\end{xy}
\end{minipage}
\begin{minipage}[c]{0.16\textwidth}
\begin{xy}
\xyimport(100, 100){\includegraphics[width=1.35cm, height=3.2cm]{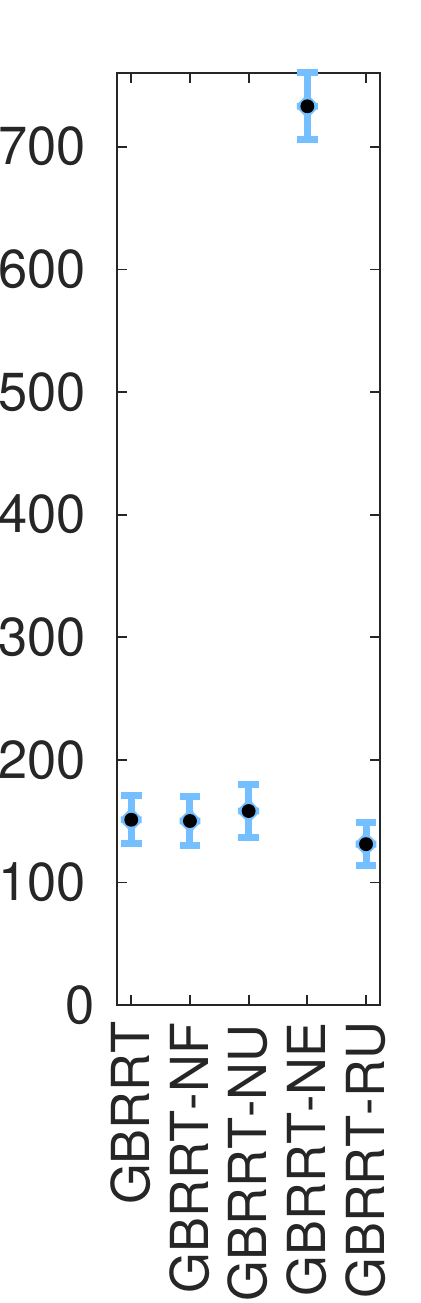}}
,(60, 102)*{\footnotesize \text{Treaded}}
\end{xy}
\end{minipage}
\begin{minipage}[c]{0.16\textwidth}
\begin{xy}
\xyimport(100, 100){\includegraphics[width=1.35cm, height=3.2cm]{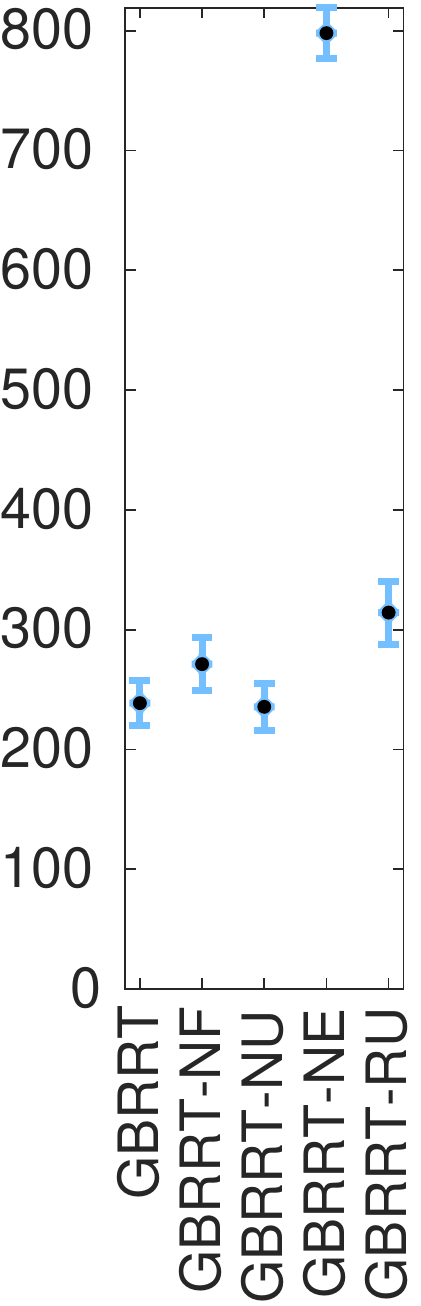}}
,(60, 105)*{\footnotesize \text{Car-Trailer}}
\end{xy}
\end{minipage}
\begin{minipage}[c]{0.16\textwidth}
\begin{xy}
\xyimport(100, 100){\includegraphics[width=1.35cm, height=3.2cm]{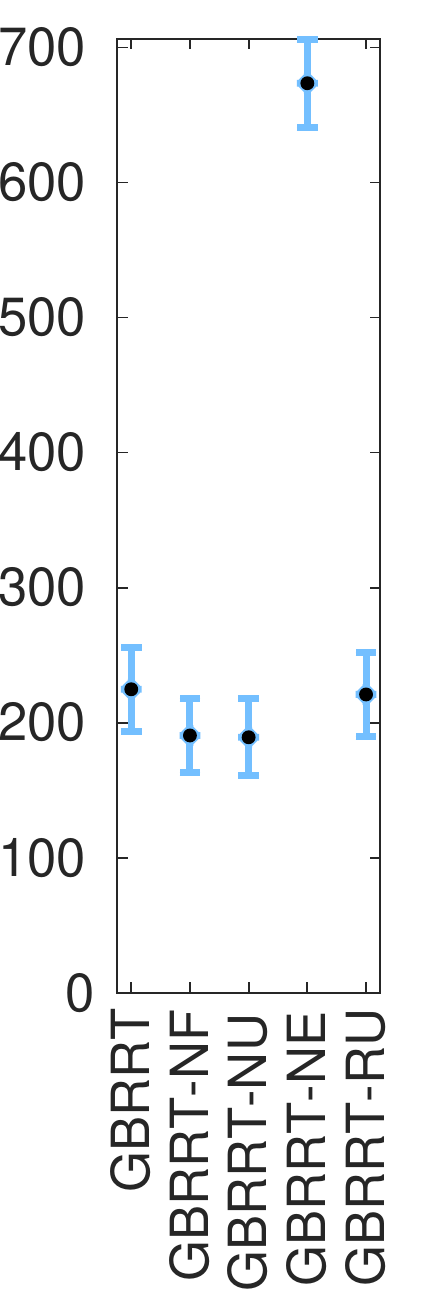}}
,(60, 102)*{\footnotesize \text{Fixed-Wing}}
\end{xy}
\end{minipage}
\caption{Comparison of success rate and solution time metrics for different modifications of GBRRT --- GBRRT-NF (No Fast exploration), GBRRT-NU (No Priority Q-Update), GBRRT-NE (No Exploitation) and GBRRT-RU (Range-Based Update).} \label{fig:modificationComparisons}
\end{figure}

\begingroup
\removelatexerror
\begin{figure*}[h!]
\scalebox{\algorithmScale}{
\centering
\begin{minipage}[c]{7.8cm}
\begin{algorithm}[H]
  \tcp{Rev. search Best-input prop.}
  \caption{$\mathtt{BestInputProp}(x_{near}, x_{rand}, \mathcal{U}, T_{max})$} \label{alg:BestInputPropRev}
  
    $\mathbf{E}_{initial} \gets \{\} $ \\
    \For{$k \gets 1$ to $N_{B}$}
    {
        $ \mathbf{E}_{initial} \gets \mathbf{E}_{initial} \cup \{\mathtt{MonteCarloProp}(x_{near}, \mathcal{U}, T_{max})\} \!\!\!\!\!\!$ \\ \label{alg:BestInputPropRev:MonteCarloProp}
    }
    
    $\displaystyle{\mathcal{E}_{new} \gets \argmin_{\mathcal{E} \in \mathbf{E}_{initial}}  ({d_{X}(x_{rand},\,\, \mathcal{E}.initialNode())}}$ \\ \label{alg:BestInputPropRev:bestNew}
  
    \Return {$\mathcal{E}_{new}$}
\end{algorithm}
\begin{algorithm}[H]
  \caption{$\reverseAlg{\mathtt{UpdatePQueueND}}(\mathcal{G}_{for}, \mathbf{Q}, x_{rev}, r_{k})$} \label{alg:UpdatePQueueND}
  
  $x_{closest} \gets \reverseAlg{\mathcal{G}_{for}.\mathtt{NearestNeighbor}(x_{rev})}$\\
  
  \label{alg:UpdatePQueueND:NearestVertices}
     
      \If{\revisionTwo{$x_{closest} \leq r_{k}$}}
      {
          $\mathbf{Q}\mathtt{.update}\paren{x_{closest}, \reverseAlg{d^{\mathtt{ND}}_{X}}(x_{closest}, x_{rev}) +  h(x_{rev})}$ \\ \label{alg:UpdatePQueueND:addToQ}
      }
\end{algorithm}
\end{minipage}}
\scalebox{0.7}{
\begin{minipage}{8.2cm}
\revisionAlgorithm
\begin{algorithm}[H]
  \tcp{Monte Carlo propagation for reverse search}
  \caption{$\mathtt{MonteCarloProp}(x_{final}, \mathcal{U}, T_{max})$} \label{alg:MonteCarloPropRev}
    $t_{\mathcal{E}} \gets \mathtt{Random}(0, T_{max})$ \\
     $u \gets \mathtt{Random}(\mathcal{U})$\\
     \revision{$\mathcal{E}_{new} = -\int_{t_{{\mathcal{E}}}}^{0} f(x, u) \,dt$ where $\mathcal{E}_{new}(t_{{\mathcal{E}}})=x_{final}$\\}
    \Return {$\mathcal{E}_{new}$}
\end{algorithm}
\begin{algorithm}[H]
  \caption{$\reverseAlg{\mathtt{insertToPQueueND}}(\reverseAlg{\mathcal{G}^{\mathtt{ND}}_{rev}}, \mathbf{Q}, x_{for}, r_{k})$} \label{alg:insertToPQueueND}
  
  $x_{closest} \gets\reverseAlg{\mathcal{G}^{\mathtt{ND}}_{rev}.\mathtt{NearestNeighbor}(x_{for})}$ \\
  \label{alg:insertToPQueueND:NearestVertices}

  \If{\revisionTwo{$x_{closest} \leq r_{k}$}}
  {
      $\mathbf{Q}\mathtt{.insert}\paren{x_{for}, \reverseAlg{d^{\mathtt{ND}}_{X}}(x_{for}, x_{closest}) +  h(x_{closest})}$ \\ \label{alg:insertToPQueueND:addToQ}
  }
      
\end{algorithm}
\end{minipage}}
\scalebox{0.72}{
\begin{minipage}{8.5cm}
\begin{algorithm}[H]
    \caption{$\reverseAlg{\mathtt{RevSrchFastExploreND}}(\reverseAlg{\mathcal{G}^{\mathtt{ND}}_{rev}}, \mathcal{U}, T_{max})$} \label{alg:RevSrchND}
    $x_{rand} \gets \mathtt{RandomState} \paren{}$ \\ \label{alg:getExtendEdge:RandomState}
    $x_{near} \gets \mathtt{NearestNeighbor} \paren{x_{rand}, \reverseAlg{\mathcal{G}^{\mathtt{ND}}_{rev}}}$ \\ \label{alg:getExtendEdge:NearestNeighbor}
    $\reverseAlg{\mathcal{E}^{\mathtt{ND}}_{new} \gets \mathtt{Extend} \paren{x_{near}, x_{rand}}}$ \\ \label{alg:getExtendEdge:BestInputProp}
    \Return {$\mathcal{E}^{\mathtt{ND}}_{new}$}
\end{algorithm}
\begin{algorithm}[H]
  \tcp{Fwd. search for Best-input prop.}
  \caption{$\reverseAlg{\mathtt{BestInputPropND}}(x_{near}, x_{rand}, \mathcal{U}, T_{max})$} \label{alg:BestInputPropND}
  
    $\mathbf{E}_{fin} \gets \{\} $ \\
    \For{$k \gets 1$ to $N_{B}$}
    {
        $ \mathbf{E}_{fin} \gets \mathbf{E}_{fin} \cup \{\mathtt{MonteCarloProp}(x_{near}, \mathcal{U}, T_{max})\} \!\!\!\!\!\!$ \\ \label{alg:BestInputPropND:MonteCarloProp}
    }
    
    $\displaystyle{\mathcal{E}_{new} \gets \argmin_{\{\mathcal{E}\in \mathbf{E}_{fin}\}} ({\reverseAlg{d^{\mathtt{ND}}_{X}}(\mathcal{E}.finalNode(),\,\, x_{rand})})}$ \\ \label{alg:BestInputPropND:bestNew}
    
    \Return {$\mathcal{E}_{new}$}
\end{algorithm}
\end{minipage}}
\vspace{-1mm}
\end{figure*}
\endgroup

\vspace{-3mm}
\subsection{\revision{Reverse search algorithms}}
\label{apdx:ReverseSearchAlgorithms}
\revision{Algorithms \ref{alg:BestInputPropRev} and \ref{alg:MonteCarloPropRev} present reverse search subroutines.  }

\begingroup
\removelatexerror
\begin{figure}[h!]
\vspace{-8mm}
\centering
\scalebox{0.64}{
\begin{minipage}{10.3cm}
\begin{algorithm}[H]
  \caption{$\reverseAlg{\mathtt{ForSrchExploitND}}(\reverseAlg{\mathcal{G}^{\mathtt{ND}}_{rev}}, \mathcal{U}, T_{max}, x_{pop}, r_{k}$)} \label{alg:ForSrchExploitND}

$x_{best} \gets \underset{g(x_{pop}) + \reverseAlg{d^{\mathtt{ND}}_{X}}(x_{pop}, x) + h(x)}{\min} \{x\in \mathbf{V}_{rev\_near}\}$ \\ \label{alg:ForSrchExploitND: x_near}

$\mathcal{E}_{new} \gets \reverseAlg{\mathtt{BestInputPropND}} \paren{x_{pop}, x_{best},\mathcal{U}, T_{max}}$ \\ \label{alg:ForSrchExploitND: bestInputProp}

\Return {$\mathcal{E}_{new}$} \\ \label{alg:ForSrchExploitND: returnEdge}
\end{algorithm}

\revisionAlgorithm
\begin{algorithm}[H]
    \caption{$\reverseAlg{\mathtt{GABRRT}}(x_{start}, x_{goal}, \mathcal{X}_{goal}, \mathcal{U}, T_{max}, \delta_{hr}, \mathcal{P})$} \label{alg:GABRRT}
    $\mathcal{G}_{for} \gets \{\mathbf{V}_{for} \gets \{x_{start}\}, \mathbf{E}_{for} \gets \{\}\}$ \\ \label{alg:GABRRT:initGFor}
     $\reverseAlg{\mathcal{G}^{\mathtt{ND}}_{rev}} \gets \{\mathbf{V}_{rev} \gets \{x_{goal}\}, \reverseAlg{\mathbf{E}^{\mathtt{ND}}_{rev}} \gets \{\}\}$ \\ \label{alg:GABRRT:initGrev}
     $\mathbf{Q} \gets \{\}$ \\ \label{alg:GABRRT:initQ}
    
    \For{$k \gets 1$ to $M_{iter}$}
    {
        \tcp{Reverse Tree Expansion (lines 12-30)}
        $r_{k} \gets \min(\gamma {\paren{\frac{\log{\abs{\mathbf{V_{rev}}}}}{\abs{\mathbf{V_{rev}}}}}}^{\frac{1}{d + 1}}, \, \delta_{hr})$

         \reverseAlg{$\mathcal{E}^{\mathtt{ND}}_{rev} \gets \mathtt{RevSrchFastExploreND}(\reverseAlg{\mathcal{G}^{\mathtt{ND}}_{rev}}, \mathcal{U}, T_{max})$} \\ \label{alg:GABRRT:ReverseExpansion}
         
         \If {$\mathtt{not \, CollisionCheck}(\reverseAlg{\mathcal{E}^{\mathtt{ND}}_{rev}})$ \label{alg:GABRRT:collisioncheckReverse}}
         {
             
              $x_{rev} \gets \reverseAlg{\mathcal{E}^{\mathtt{ND}}_{rev}}.\mathtt{initialNode()}$ \\
              \label{alg:GABRRT:AssignXRev}
              \reverseAlg{$\mathbf{E}^{\mathtt{ND}}_{rev} \gets \mathbf{E}^{\mathtt{ND}}_{rev} \cup \{\mathcal{E}^{\mathtt{ND}}_{rev}\}$} \\ \label{alg:GABRRT:updateErev}
             $\mathbf{V}_{rev} \gets \mathbf{V}_{rev} \cup \{x_{rev} \}$ \\ \label{alg:GABRRT:updateVrev}
             
              \reverseAlg{$\mathtt{UpdatePQueueND} \paren{\mathcal{G}_{for}, \mathbf{Q},  x_{rev}, r_{k}}$} \\ \label{alg:GABRRT:updatePriorityQueue}
         }
    
        \tcp{Forward Tree Expansion (lines 10-28)}
         $ q \gets \mathcal{P}(k)$ \\ \label{alg:GABRRT:updateExploitRatio}
         
         $c_{rand} \sim \mathit{U}\paren{[0, 1]}$ \\ \label{alg:GABRRT:initp}

        \If {$c_{rand} < q$ \label{alg:GABRRT:exploitRatio}}
        {
        $\mathcal{E}_{for} = \mathtt{NULL}$ \\ \label{alg:GABRRT:Enull}
        
        \tcp{Pop node from $\mathbf{Q}$}
        $x_{pop} \gets \mathtt{Pop}\paren{\mathbf{Q}}$ \\ \label{alg:GABRRT:PopQ}
        
        \If {$x_{pop} \neq \mathtt{NULL}$}
        {
                \reverseAlg{ $\mathcal{E}_{for} \gets 
                \mathtt{ForSrchExploitND}\paren{\reverseAlg{\mathcal{G}^{\mathtt{ND}}_{rev}}, \mathcal{U}, T_{max}, x_{pop}, r_{k}}$} \\ \label{alg:GABRRT:ForwardSearchExploit}
        }
        
        \If {$\mathcal{E}_{for} = \mathtt{NULL}$}
            {
        $\mathcal{E}_{for} \gets \mathtt{ForSrchFastExplore}\paren{\mathcal{G}_{for}, \mathcal{U}, T_{max}}$ \\ \label{alg:GABRRT:ForwardSearchFastExplore}
            }
        }
  
       \If {$c_{rand} \geq q \;\; \mathtt{or} \;\; \mathcal{E}_{for} = \mathtt{NULL}$}
       {
        \tcp{For Probabilistic completeness}
        $\mathcal{E}_{for} \gets \mathtt{ForSrchRandomExplore}\paren{\mathcal{G}_{for}, \mathcal{U}, T_{max}}$ \\ \label{alg:GABRRT:ForwardSearchRandomExplore}
       }
        
        \If {$\mathcal{E}_{for} \neq \mathtt{NULL}$}
        {
         
         \If {$\mathtt{not \, CollisionCheck}(\mathcal{E}_{for})$ \label{alg:GABRRT:collisioncheckForward}}
         {
             $x_{for} \gets \mathcal{E}_{for}.\mathtt{finalNode()}$ \\
              \label{alg:GABRRT:AssignXfor}
             $\mathbf{E}_{for} \gets \mathbf{E}_{for} \cup \{\mathcal{E}_{for}\}$ \\ \label{alg:GABRRT:updateEfor}
             $\mathbf{V}_{for} \gets \mathbf{V}_{for} \cup \{x_{for} \}$ \\ \label{alg:GABRRT:updateVfor}
            
            \If {$\mathtt{goalRegionReached}(\mathcal{X}_{goal}, x_{for})$ \label{alg:GABRRT:goalRegionReached}}
            {
                \Return {$\mathtt{Path}(\mathcal{G}_{for}, x_{start}, x_{for})$} \\ \label{alg:GABRRT:Path}
            }
            
            \ \reverseAlg{ $\mathtt{insertToPQueueND} \paren{\reverseAlg{\mathcal{G}^{\mathtt{ND}}_{rev}}, \mathbf{Q},  x_{for}, r_{k}}$} \\ \label{alg:GABRRT:insertToPriorityQueue}
         }
        }
    }
 \Return {$\mathtt{NULL}$}
\end{algorithm}
\end{minipage}}
\vspace{-7mm}
\end{figure}
\endgroup

\subsection{\revision{Generalized Asymmetric Bidirectional RRT (GABRRT)}}
\label{apdx:GABRRTAlgorithms}

\noindent \revision{
GABRRT appears in \algoref{alg:GABRRT} and its subroutines in Algorithms 11, 13-16. Differences between GBRRT and GABRRT appear brown. $\mathcal{G}^{\mathtt{ND}}_{rev}$, $\mathcal{E}^{\mathtt{ND}}_{rev}$, $\mathbf{E}^{\mathtt{ND}}_{rev}$, $d^{\mathtt{ND}}_{X}$ denote the non-dynamical reverse tree, reverse trajectory, edge set, and distance functions, respectively.}
$\mathtt{Extend}$ returns a straight line trajectory $\mathcal{E}^{\mathtt{ND}}_{new}$ starting at $x_{near}$ and ending at a point on the line between $x_{near}$ and $x_{rand}$. The length  ${\|\mathcal{E}^{\mathtt{ND}}_{new}\| \leq \epsilon}$, where $\epsilon$ is a user-defined value. $\mathtt{ForSrchFastExplore}$ (\algoref{alg:ForSrchFastExplore}) and $\mathtt{ForSrchRandomExplore}$ (\algoref{alg:ForSrchRandomExplore}) appear in \secref{sec:specificImplementation}.

\subsection{\revisionTwo{Computational Complexity and Shrinking $d$-Ball Rate}} \label{sub:RuntimeComplexity}

\newcommand{\forwardTreeCount}{|\mathbf{V}^{i}_{\mathrm{for}}|}
\newcommand{\reverseTreeCount}{|\mathbf{V}^{i}_{\mathrm{rev}}|}

\revisionTwo{%
In this section, we derive a worst case runtime bound and then discuss how the function used to shrink the neighborhood $d$-Ball affects the expected runtime of iteration $i$.   
GBRRT and GABRRT require two data structures. A dynamic space partitioning tree is used to perform dynamic insertion (DI), nearest neighbor (NN), and range queries (RS), while a priority heap stores nodes based on priority key values and supports Push, Pop, and Update operations \cite{samet1990design, cormen2009introduction}. Our runtime analysis assume these subroutines have the following runtime bounds, where $N_{T}$ is the number of points in the tree.}

\begin{assumption} \label{assumption:A}
\revisionTwo{DI takes expected time $\mathcal{O}\paren{\log{}N_{T}}$.}
\label{assum:insertion}
\end{assumption}
\begin{assumption} \label{assumption:B}
\revisionTwo{NN takes expected time $\mathcal{O}\paren{\log{}N_{T}}$.}
\label{assum:NN}
\end{assumption}

\revisionTwo{These bounds are met by many algorithms \cite{cormen2009introduction}. GBRRT does not require exact NN and RS queries and, similar to other sampling based motion planning algorithms \cite{karaman2011sampling}, can also work with `approximate' solutions to these operations \cite{arya1998optimal,arya2000approximate,Eppstein.eatal.SCG05}. For example, the Skip Quadtree  \cite{Eppstein.eatal.SCG05} supports a special kind of approximate RS query in which all points within range $r_{k}$ are returned plus {\it possibly some additional points} that are further than $r_{k}$ but within $r_{k}(1 +\varepsilon)$ for a user defined $\varepsilon$. The Skip Quadtree is dynamic, allowing for both new point insertions and look-ups in expected time $\log N_{T}$. The approximate range query returns $\nu$ points in expected time $\mathcal{O}(\varepsilon^{1-d} \log(N_{T}) + \nu)$. The Skip Quadtree also meets Assumptions 1 and 2.}

\begin{assumption}  \label{assumption:C}
\revisionTwo{RS query (or approximate RS query) in dynamic space partitioning tree takes $\mathcal{O}\paren{c_d \log{}N_{T} + \nu}$ expected runtime and returns $\nu$ neighbors in expectation, where $c_d$ is a (possibly user chosen) parameter that may depend on $d$ but is otherwise constant.}
\label{assum:RS}
\end{assumption}

\revisionTwo{While various implementations of  priority heaps exist \cite{cormen2009introduction}, many perform  push and pop operations in $\mathcal{O}\paren{\log{}N_{Q}}$ runtime, where $N_{Q}$ is the number of points in the queue \cite{cormen2009introduction}.
}
\begin{assumption} \label{assumption:D}
\revisionTwo{Push, Pop, and Update operations of the priority queue each take time  $\mathcal{O}\paren{\log{}N_{Q}}$.}
\label{assum:PushPop}
\end{assumption}

\revisionTwo{A trivial worst case runtime for GBRRT and GABRRT is given below in Theorem~\ref{theorem:gbrrtWorstCase} for the case when Assumptions \ref{assumption:A}-\ref{assumption:D} hold.  We denote $\forwardTreeCount$ and  $\reverseTreeCount$ as the number of forward and reverse tree nodes at the start of iteration $i$, respectively. $n$ is the total number of nodes stored in both trees at the start of iteration $i$. Note that $n = \forwardTreeCount + \reverseTreeCount \leq 2i$}.

\begin{theorem}\label{theorem:gbrrtWorstCase}
\revisionTwo{The worst case runtime of GBRRT and GABRRT is $\mathcal{O}\paren{n^2}$.}
\end{theorem}
\begin{proof}
\revisionTwo{In a (worse than) worst case $\delta_{hr} = \gamma = \infty$ such that $r_{k} = \infty$ and every query for forward (resp.\ reverse) tree neighbors returns the entire forward (resp.\ reverse) tree, which contains $\forwardTreeCount$ (resp.\ $\reverseTreeCount$) nodes at iteration $i$. An upper bound on the time required per iteration is obtained by observing $\nu \leq \forwardTreeCount + \reverseTreeCount = n$. Summing over iterations $i = 1, \ldots, n$ leads to a worst case bound of $\mathcal{O}\paren{n^2}$.}
\end{proof}

\begin{remark}
\revisionTwo{
The {\it worst case} $\mathcal{O}\paren{n^2}$ time is identical to that of many other sampling based motion planning algorithms. Such a worst case is typically only realized in degenerative cases, and average runtimes tend to be much lower in practice.}
\end{remark}

\revisionTwo{Before deriving {\it expected} per iteration runtime bounds, it is useful to discuss what we do {\it not} know about the distribution of nodes stored within the forward and backward trees.}

\revisionTwo{Expected runtime analysis of sampling based motion planning algorithms often focuses on the geometric case (without system dynamics) where the set of potential neighbors already in the tree can be assumed to be generated from a uniform random distribution over the free space, at least in the limiting case as ${n \rightarrow \infty}$. In contrast, the nodes stored in the forward and reverse trees of GBRRT and the forward tree of GABRRT are selected using a combination of randomness, forward or backward integration of system dynamics (or, alternatively, maneuver libraries). Moreover, a portion of all nodes' position in the forward tree are influenced by the reverse tree focusing heuristic. Whether or not the nodes within either tree resemble a set of points drawn uniformly at random and i.i.d.\ is, {\it at best}, system/scenario dependent. It remains an open question whether or not the sampling distributions realized by GBRRT and GABRRT converge toward the uniform distribution (even for the special case of geometric planning).}

\revisionTwo{
When calculating expected runtime, {\it if} points can be assumed to be drawn i.i.d from a uniform distribution, {\it then} many existing proofs \cite{karaman2011sampling,otte2016rrtx} take advantage of the fact that the expected number of nodes $\mathbb{E}(| \mathrm{V}_{\mathrm{near}} |)$ in a $d$-ball is proportional to the ball's volume $V_{d}(r)$. Given $N_T$ nodes drawn uniformly random and i.i.d.\ from $X_{free}$ then 
$\mathbb{E}(| \mathrm{V}_{\mathrm{near}, r} |) \sim N_T \frac{V_{d}(r)}{||X_{free}||}
$, where  $||X_{free}||$ is the volume of the free space, and where the free space is assumed to have ``nice'' boundaries, e.g., locally Lipschitz continuous boundaries for some positive finite Lipschitz constant.
In a static environment, $||X_{free}||$ is constant. This discussion leads to the following Proposition~\ref{prop:nodesInBallUniform}.}

\begin{proposition} \label{prop:nodesInBallUniform}
\revisionTwo{If points are sampled uniformly at random and i.i.d. from $X_{free}$, then $\mathbb{E}(| \mathrm{V}_{\mathrm{near}, r} |) = \Theta(N_T V_{d}(r))
$.}
\end{proposition}

\begin{lemma} \label{remark:approxQueryBound}
\revisionTwo{Given range $r$ and $N_T$ points sampled uniformly at random i.i.d., and assuming $\varepsilon$ and $d$ are both constant, the expected value of $\nu$ for the approximate RS query described above is $\mathbb{E}(\nu(r,N_T)) = \Theta(N_T V_{d}(r))$.}
\end{lemma}
\begin{proof}
\revisionTwo{It is a useful fact that the ratio between the volume of a \mbox{$d$-ball} of radius $r(1 + \varepsilon)$ and the volume of a $d$-ball of radius $r$ is:
$
\left(\frac{\pi^{d/2}}{\Gamma({1 + d/2})}(r_{k}(1 + \varepsilon))^{d}\right)  /  \left(\frac{\pi^{d/2}}{\Gamma({1 + d/2})}r_{k}^{d}\right) = (1 + \varepsilon)^{d}
$, where $\Gamma(\cdot)$ is the gamma function. 
The volume ratio $(1 + \varepsilon)^{d}$ does not depend on $r_{k}$ and, once $d$ is chosen, is a constant.}
\end{proof}

\revisionTwo{
If the point sampling process {\it cannot} be modeled as uniformly at random and i.i.d., then it is possible that $\mathbb{E}(| \mathrm{V}_{\mathrm{near}, r} |) = \Omega(N_T V_{d}(r))$. For example, if new points are more likely to be sampled in the vicinity of existing points such that, in expectation, greater than half of all points have greater than  $N_T \frac{V_{d}(r)}{||X_{free}||}$ neighbors. Since the value of $\mathbb{E}(| \mathrm{V}_{\mathrm{near}, r} |)$ achieved in GBRRT/GABRRT is an open question, we shall derive lower bounds on runtime for the $\mathbb{E}(| \mathrm{V}_{\mathrm{near}, r} |) = \Omega(N_T V_{d}(r))$ case.
How $\mathbb{E}(| \mathrm{V}_{\mathrm{near}, r} |) = \Omega(N_T V_{d}(r))$ may affect $\nu$ is important for our analysis.
}

\begin{lemma} \label{lemma:nodesInBall}
\revisionTwo{If $\mathbb{E}(| \mathrm{V}_{\mathrm{near}, r} |) = \Omega(N_T V_{d}(r))$ then $\mathbb{E}(\nu(r,N_T)) = \Omega(N_T V_{d}(r))$.} \hfill
\end{lemma}
\begin{proof}
\revisionTwo{By definition, using Bachmann–Landau notation, ${{\mathbb{E}(| \mathrm{V}_{\mathrm{near}, r} |) = \Omega(N_T V_{d}(r))} \iff {{\displaystyle{\liminf_{N_T \rightarrow \infty}}} \frac{\mathbb{E}(| \mathrm{V}_{\mathrm{near}, r} |) }{N_T V_{d}(r)}  > 0}}$. 
By construction of the approximate range search (RS) query ${\nu(r,N_T) \geq | \mathrm{V}_{\mathrm{near}, r} |}$ and so ${\mathbb{E}(\nu(r,N_T)) \geq \mathbb{E}(| \mathrm{V}_{\mathrm{near}, r} |)}$. Thus, 
${{\displaystyle{\liminf_{N_T \rightarrow \infty}}} \frac{\mathbb{E}(\nu(r,N_T)) }{N_T V_{d}(r)}  > 0}$ and so by definition, using Bachmann–Landau notation,  $\mathbb{E}(\nu(r,N_T)) = \Omega(N_T V_{d}(r))$.}
\end{proof}

\revisionTwo{We now characterize the expected runtime required for the forward and reverse search portions of GBRRT and GABRRT in iteration $i$.
Note $\abs{\mathbf{Q}} \leq \forwardTreeCount$ at iteration $i$, where $\abs{\mathbf{Q}}$ is the numbers of nodes in the priority queue $\mathbf{Q}$.}

\begin{lemma}\label{lemma:gbrrtForward}
Given Assumptions \ref{assumption:A}-\ref{assumption:D}, if $\mathbb{E}(\nu(r,\reverseTreeCount)) = \mathcal{O}\parenShort{\log{}\reverseTreeCount}$, then the expected runtime of forward tree operations during iteration $i$ is $\mathcal{O}\parenShort{\log{}\forwardTreeCount + \log{}\reverseTreeCount}$; and alternatively, if $\mathbb{E}(\nu(r,\reverseTreeCount)) =  \Omega \parenShort{\log{}\reverseTreeCount}$, then the expected runtime of forward tree operations during iteration $i$  is $\Theta(\nu(r, \reverseTreeCount))$.
\end{lemma}
\begin{proof}
\revisionTwo{Forward search involves a NN query and a DI operation on the forward tree and push/pop of a forward tree node from $\mathbf{Q}$; all take expected time $\mathcal{O}\parenShort{\log{}\forwardTreeCount}$. It also involves a NN query and RS query on the reverse tree. 
For constant $d$, these operations take expected time $\mathcal{O}\parenShort{\log{}\reverseTreeCount}$ and $\mathcal{O}\parenShort{\log{}\reverseTreeCount + \nu(r,\reverseTreeCount)}$, respectively. The expected time involved in processing the $\mathbb{E}( \nu(r, \reverseTreeCount))$ nodes returned by the RS query is $\Theta(\nu(r, \reverseTreeCount))$. The proof is finished by considering the case where  $\mathbb{E}(\nu(r, \reverseTreeCount)) = \mathcal{O}\parenShort{\log{}\reverseTreeCount}$; and alternatively, the case where $\mathbb{E}(\nu(r, \reverseTreeCount)) =  \Omega \parenShort{\log{}\reverseTreeCount}$.}
\end{proof}

\begin{lemma}
\revisionTwo{Given Assumptions \ref{assumption:A}-\ref{assumption:D}, the expected runtime of reverse search operations during iteration $i$ is $\mathcal{O}\parenShort{\log{}\forwardTreeCount + \log{}\reverseTreeCount}$.
\label{lemma:gbrrtReverse}}
\end{lemma}
\begin{proof}
\revisionTwo{The computationally dominant operations performed during reverse search are NN query and insertion in reverse tree which both take expected time $\mathcal{O}\parenShort{\log{}\reverseTreeCount}$ and NN query in forward tree and update operation in $\mathbf{Q}$ which both take expected time $\mathcal{O}\parenShort{\log{}\forwardTreeCount}$.
}
\end{proof}

\begin{theorem} \label{th:perIteration}
\revisionTwo{Given Assumptions \ref{assumption:A}-\ref{assumption:D}, if $\mathbb{E}(\nu(r,n)) = \mathcal{O}\parenShort{\log{} n }$, then the expected runtime of iteration $i$ is $\mathcal{O}\parenShort{\log{} n}$, alternatively, if $\mathbb{E}(\nu(r,n)) =  \Omega \parenShort{\log{} n }$,  then the expected runtime of iteration $i$ is $\Theta(\nu(r,n))$.}
\end{theorem}
\begin{proof}
\revisionTwo{Combine Lemma~\ref{lemma:gbrrtForward} and Lemma~\ref{lemma:gbrrtReverse} with the fact that ${2 \leq \forwardTreeCount  + \reverseTreeCount = n \leq 2i}$, and we alternate between attempting to grow the forward and reverse trees, so we expect both trees to have a nonzero proportion of the $n$ nodes in the limit, i.e., $\lim_{i \rightarrow \infty} \frac{\forwardTreeCount}{\reverseTreeCount + \forwardTreeCount} = c$ such that ${0 < c < 1}$. Because neither tree's cardinality is expected to overwhelm the other, the most conservative runtime associated with either tree determines the runtime of the overall algorithm.} 
\end{proof}

\revisionTwo{Theorem \ref{th:perIteration} shows that $\mathbb{E}(\nu(r,n))$ plays a key role in determining the expected runtime of iteration $i$ in GBRRT and GABRRT. $\mathbb{E}(\nu(r,n))$ is affected by the radius $r$ of the $d$-ball as well as the sampling distribution of the $n$ points in $X_{free}$. As already mentioned, the sampling distribution that results when running GBRRT or GABRRT is an open question. 
We shall now discuss the radius of the ball, looking at three cases of interest --- using a constant ball, and shrinking at the ball according to two functions that have appeared in the literature. We consider the uniform i.i.d. case followed by the case where ${\mathbb{E}(\nu(r,N_T)) = \Omega(N_T V_{d}(r))}$.}

\begin{theorem}\label{lem:iidPerIterationRuntime}
\revisionTwo{Given Assumptions \ref{assumption:A}-\ref{assumption:D}, and assuming $d$ is constant, nodes are drawn uniformly at random and i.i.d., and ${r_k = c}$ for a positive constant ${c < \infty}$, the expected runtime of GBRRT/GABRRT iteration $i$ is $\Theta(n)$.}
\end{theorem}
\begin{proof}
\revisionTwo{${\mathbb{E}(\nu(r,n)) = \Theta(n V_{d}(r))}$ by Proposition~\ref{prop:nodesInBallUniform} and Lemma~\ref{remark:approxQueryBound}. By construction, $V_{d}(r)$ is constant for constant $r$ and so $\mathbb{E}(\nu(r,n)) = \Theta(n)$. 
By definition ${n = \Omega \parenShort{\log{} n }}$.  Thus, by Theorem~\ref{th:perIteration}, the expected runtime of iteration $i$ is $\Theta(\nu(r,n))$. Because $\mathbb{E}(\nu(r,n)) = \Theta(n)$ when $r_k$ is constant, the expected runtime of iteration $i$ is $\Theta(n)$.
}
\end{proof}

\begin{corollary} \label{cor:noniidruntime}
\revisionTwo{Given Assumptions \ref{assumption:A}-\ref{assumption:D} and constant $d$, and assuming nodes are sampled such that $\mathbb{E}(| \mathrm{V}_{\mathrm{near}, r} |) = \Omega(N_T V_{d}(r))$, and assuming ${r_k = c}$ for some positive constant ${c < \infty}$, then the {expected} runtime of GBRRT/GABRRT  iteration $i$ is $\Omega(n)$.}
\end{corollary}

\noindent \revisionTwo{Corollary~\ref{cor:noniidruntime} results from combining Theorem~\ref{lem:iidPerIterationRuntime}
with Lemma~\ref{lemma:nodesInBall}.}

\revisionTwo{Many sampling based motion planning algorithms shrink the neighborhood $d$-ball  to reduce expected runtime. For example, RRT$^*$ \cite{karaman2011sampling} shrinks the neighborhood ball according to ${r_{k} = \gamma \big(\frac{\log n}{n} \big)^{1/d}}$ to achieve an expected per iteration runtime of $\mathcal{O}(\log n)$.
We analyze the effects of using this strategy for GBRRT and GABRRT. While ${r_{k} = \min (\gamma (\frac{\log n}{n})^{1/d},  \delta_{hr})}$, the use of $\delta_{hr}$ ceases after a finite number of iterations because ${(\frac{\log n}{n})^{1/d} = \mathcal{O}(\delta_{hr})}$. So, our asymptotic analysis focuses on the  ${r_{k} = \gamma \big(\frac{\log n}{n} \big)^{1/d}}$ case.}

\begin{theorem}\label{lem:iidPerIterationRuntimeShrinking}
\revisionTwo{Given Assumptions \ref{assumption:A}-\ref{assumption:D} and constant $d$, and assuming nodes are drawn uniformly at random and i.i.d.\, and assuming $r_{k} = \gamma \big(\frac{\log n}{n} \big)^{1/d}$ then the {expected} runtime of GBRRT/GABRRT iteration $i$ is $\Theta(\log n)$.}
\end{theorem}
\begin{proof}
\revisionTwo{Combining Proposition~\ref{prop:nodesInBallUniform} with Lemma~\ref{remark:approxQueryBound}  ${\mathbb{E}(\nu(r,n)) = \Theta(n V_{d}(r))}$. Given constant $d$, by construction 
$n V_{d}(r) = n \big(\frac{\pi^{d/2}}{\Gamma({1 + d/2})} \big(\gamma \big(\frac{\log n}{n} \big)^{1/d} \big)^{d}\big) = \Theta(\log n)$, where $\gamma$ is constant. Theorem~\ref{th:perIteration} completes the proof.}
\end{proof}

\begin{corollary} \label{cor:noniidRuntimeShrinking}
\revisionTwo{Given Assumptions \ref{assumption:A}-\ref{assumption:D} and constant $d$, and assuming nodes are sampled such that $\mathbb{E}(| \mathrm{V}_{\mathrm{near}, r} |) = \Omega(N_T V_{d}(r))$, and assuming $r_{k} = \gamma \big(\frac{\log n}{n} \big)^{1/d}$ then the {expected} runtime of GBRRT/GABRRT iteration $i$ is $\Omega( \log n)$.}
\end{corollary}

\noindent \revisionTwo{Corollary~\ref{cor:noniidRuntimeShrinking} results from combining Theorem~\ref{lem:iidPerIterationRuntimeShrinking} 
with Lemma~\ref{lemma:nodesInBall}.}

\revisionTwo{A recent result from Solovey et al.\ \cite{solovey2020revisiting} has highlighted a small error in Karaman and Frazolli's proof of almost sure {\it asymptotic optimally} for RRT$^*$ \cite{karaman2011sampling}. (Note: the proof of {\it probabilistic completeness} in \cite{karaman2011sampling} is still valid.) The original  proof from \cite{karaman2011sampling} has motivated the use of ${r_{k} = \gamma \big(\frac{\log n}{n} \big)^{1/d}}$ in a number of algorithms. Solovey et al. show that using a larger ${r_{k} = \gamma \big(\frac{\log n}{n} \big)^{1/(d+1)}}$ {\it is sufficient} to achieve almost sure asymptotic optimality. Whether or not using ${r_{k} =  \gamma \big(\frac{\log n}{n} \big)^{1/d}}$ is also sufficient is now an open problem \cite{solovey2020revisiting}.
The main difference between \cite{solovey2020revisiting} and \cite{karaman2011sampling} is a consideration of how sampling order affects local rewiring direction.  
While GBRRT and GABRRT are feasible (and not asymptotically optimal) planning algorithms, the reverse tree focusing heuristic is concerned with local wiring direction. 
While further study of the relationship between asymptotic optimality and focusing heuristics is outside the scope of this paper, we both (i) use ${r_{k} = \gamma \big(\frac{\log n}{n} \big)^{1/(d+1)}}$ in our experiments and (ii) acknowledge that future results may determine that ${r_{k} \leq  \gamma \big(\frac{\log n}{n} \big)^{1/d}}$ is sufficient for focusing heuristic purposes in GBRRT and/or GABRRT.  We now derive results for ${r_{k} = \gamma \left(\frac{\log n}{n} \right)^{1/(d+1)}}$.
}

\begin{theorem}\label{lem:iidPerIterationRuntimeShrinkingSlow}
\revisionTwo{Given Assumptions \ref{assumption:A}-\ref{assumption:D} and constant $d$, and assuming nodes are drawn uniformly at random and i.i.d., and assuming $r_{k} = \gamma \big(\frac{\log n}{n} \big)^{1/(d+1)}$, then the expected runtime of GBRRT/GABRRT iteration $i$ is $\Theta \parenShort{\left(\log n \right)^{d/(d+1)} \sqrt[d+1]{n}}$.}
\end{theorem}
\begin{proof}
\revisionTwo{Combining Proposition~\ref{prop:nodesInBallUniform} with Lemma~\ref{remark:approxQueryBound}  ${\mathbb{E}(\nu(r,n)) = \Theta(n V_{d}(r))}$. Given constant $d$, by construction 
$n V_{d}(r) = n \big(\frac{\pi^{d/2}}{\Gamma({1 + d/2})} \big(\gamma \big(\frac{\log n}{n} \big)^{1/(d+1)} \big)^{d}\big) = \Theta\parenShort{\big(\log n \big)^{d/(d+1)} \sqrt[d+1]{n}}$, where $\gamma$ is a constant. Applying Theorem~\ref{th:perIteration} completes the proof.} 
\end{proof}

\begin{corollary} \label{cor:noniidPerIterationRuntimeShrinkingSlow}
\revisionTwo{Given Assumptions \ref{assumption:A}-\ref{assumption:D} and constant $d$, and assuming nodes are sampled such that ${\mathbb{E}(| \mathrm{V}_{\mathrm{near}, r} |) = \Omega(N_T V_{d}(r))}$, and assuming \newline $r_{k} = \gamma \big(\frac{\log n}{n} \big)^{1/(d + 1)}$, then the {expected} runtime of GBRRT/GABRRT iteration $i$ is $ \Omega\parenShort{\left(\log n \right)^{d/(d+1)} \sqrt[d+1]{n}}$.}
\end{corollary}
\noindent \revisionTwo{Corollary~\ref{cor:noniidPerIterationRuntimeShrinkingSlow} results from combining Lemma~\ref{lem:iidPerIterationRuntimeShrinking}
with Lemma~\ref{lemma:nodesInBall}.}

\bibliography{citation}

\vspace{-1.3cm}

\begin{IEEEbiography}[{\includegraphics[width=1in,height=1.25in,clip,keepaspectratio]{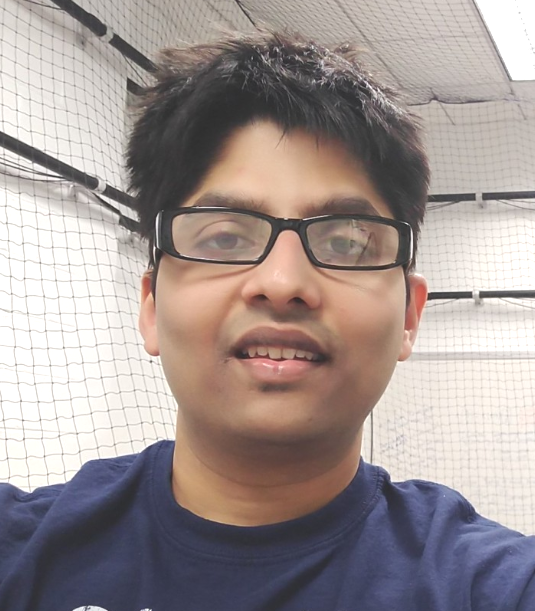}}]{Sharan Nayak}
received his Bachelor of Engineering (B.E.) degree in electronics and communications engineering from Visveswaraya Technological University, India, in 2009; his M.S. degree in electrical and computer engineering from Georgia Institute of Technology, Atlanta, GA, USA, in 2011 and M.S. degree in aerospace engineering from University of Maryland (UMD), College Park, MD in 2020. He is currently pursuing a Ph.D. degree in the department of Aerospace Engineering at UMD under the supervision of Dr. Michael Otte. His research interests are in motion planning of single and multi-agent autonomous systems. He is a student member of IEEE and a recipient of the Clark Doctoral Fellowship.
\end{IEEEbiography}

\vspace{-1.2cm}

\begin{IEEEbiography}[{\includegraphics[width=1in,height=1.25in,clip,keepaspectratio]{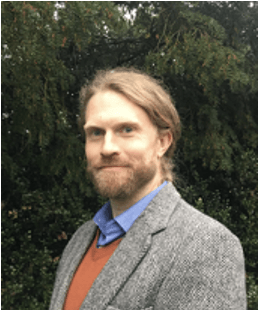}}]{Michael Otte}
(M’07) received the B.S. degrees in aeronautical
engineering and computer science from Clarkson University, Potsdam,
New York, USA, in 2005, and the M.S. and Ph.D. degrees in computer
science from the University of Colorado Boulder, Boulder, CO, USA, in
2007 and 2011, respectively.
From 2011 to 2014, he was a Postdoctoral Associate at the Massachusetts
Institute of Technology. From 2014 to 2015, he was a Visiting Scholar at
the U.S. Air Force Research Lab. From 2016 to 2018, he was a National
Research Council RAP Postdoctoral Associate at the U.S. Naval Research
Lab. He has been with the Department of Aerospace Engineering, at the
University of Maryland, College Park, MD, USA, since 2018. He is the
author of over 30 articles, and his research interests include autonomous robotics, motion planning, and multi-agent systems.
\end{IEEEbiography}
\nopagebreak

\end{document}